\documentclass{article}

\usepackage{microtype}
\usepackage{graphicx}
\usepackage{subcaption}
\usepackage{booktabs} 
\usepackage{smile}
\usepackage{hyperref}

\usepackage[accepted]{icml2024}

\usepackage{amsmath}
\usepackage{thmtools, thm-restate}
\usepackage{amssymb}
\usepackage{mathtools}
\usepackage{amsthm}
\usepackage{framed}

\usepackage[capitalize,noabbrev]{cleveref}

\theoremstyle{plain}
\newtheorem{theorem}{Theorem}[section]

\newtheorem{lemma}[theorem]{Lemma}
\newtheorem{corollary}[theorem]{Corollary}
\theoremstyle{definition}

\newtheorem{assumption}[theorem]{Assumption}
\theoremstyle{remark}
\newtheorem{remark}[theorem]{Remark}

\usepackage[textsize=tiny]{todonotes}

\icmltitlerunning{Learning the Target Network in Function Space}

\begin{document}

\twocolumn[

\icmltitle{Learning the Target Network in Function Space}



\icmlsetsymbol{equal}{*}

\begin{icmlauthorlist}
\icmlauthor{Kavosh Asadi}{equal,amzn}
\icmlauthor{Yao Liu}{equal,amzn}
\icmlauthor{Shoham Sabach}{equal,amzn,tech}
\icmlauthor{Ming Yin}{equal,prince}
\icmlauthor{Rasool Fakoor}{amzn}
\end{icmlauthorlist}

\icmlaffiliation{amzn}{Amazon}
\icmlaffiliation{tech}{Technion}
\icmlaffiliation{prince}{Princeton University}

\icmlcorrespondingauthor{Kavosh Asadi}{kavosh@alumni.brown.edu}

\icmlcorrespondingauthor{Ming Yin}{my0049@princeton.edu}

\icmlkeywords{Machine Learning, ICML}

\vskip 0.3in
]



\printAffiliationsAndNotice{\icmlEqualContribution} 

\begin{abstract}
We focus on the task of learning the value function in the reinforcement learning (RL) setting. This task is often solved by updating a pair of online and target networks while ensuring that the parameters of these two networks are equivalent. We propose Lookahead-Replicate (LR), a new value-function approximation algorithm that is agnostic to this parameter-space equivalence. Instead, the LR algorithm is designed to maintain an equivalence between the two networks in the function space. This value-based equivalence is obtained by employing a new target-network update. We show that LR leads to a convergent behavior in learning the value function. We also present empirical results demonstrating that LR-based target-network updates significantly improve deep RL on the Atari benchmark.
\end{abstract}

\section{Introduction}
Learning an accurate value function is a core competency of reinforcement learning (RL) agents. The value function is the main ingredient in numerous RL approaches such as TD~\cite{TD}, Q-learning~\cite{Qlearning}, and DQN~\cite{DQN}. It is also a key ingredient in some of the policy-based approaches to RL, such as actor-critic~\cite{barto1983neuronlike}, and even in some model-based RL approaches, for example the Dyna Architecture~\cite{sutton1991dyna}. Therefore, obtaining a deeper understanding of the value-function approximation setting allows us to design better learning algorithms, and ultimately improve the effectiveness of RL agents in practice.

A key property of the value function $v$ is that it corresponds to the fixed point of the Bellman operator $\mathcal{T}$ (defined in Section~\ref{sec:pre}), meaning that $v = \mathcal{T}v$. In RL problems with large state spaces we leverage function approximation to learn the value function. More specifically, in this approximate setting we select a certain hypothesis class (such as neural networks) to represent an approximate value function $v_\theta$ parameterized by $\theta$~\cite{RLbook}. The value-function approximation problem is then formulated as the problem of finding a parameter $\theta$ whose corresponding value function $v_{\theta}$ solves the Bellman equation $v_{\theta} = \mathcal{T} v_{\theta}$.


When solving this equation it is typical to maintain two different parameters, denoted by $(\theta,w)$, akin to the use of target and online networks in DQN and its successors~\cite{DQN}. A deeper examination reveals that these algorithms are designed to find a pair $(\theta,w)$ that jointly satisfy two constraints. First, there is a function-space constraint, namely that the value function solves the Bellman equation $v_{w} = \mathcal{T} v_{\theta}$. Additionally there is a second constraint, this time in the parameter space, that the value functions on two sides of the Bellman equation are parameterized using exactly the same parameters $\theta = w$. By satisfying this constraint, we guarantee that we ultimately learn a single value function. But, is there a more direct formulation to ensure that we solve the Bellman equation and we also learn a single value function?

We answer this question affirmatively by reformulating the value function approximation problem. Specifically, we substitute the parameter-space constraint in the original formulation with a function-space constraint. In addition to solving the Bellman equation, $v_{w} = \mathcal{T} v_{\theta}$, we impose a second function-space constraint $v_{\theta} = v_{w}$. As a consequence of this reformulation, the value function could be parameterized using a pair of parameters $(\theta, w)$ where $\theta$ is not necessarily equal to $w$. Suppose that we have found two parameters $\theta \neq w$ whose corresponding value functions are equivalent $v_{\theta} = v_{w}$. Finding such a pair of parameters is doable especially in the overparameterized setting. Moreover, suppose that this unique value function is the fixed-point of the Bellman Operator $v_{w} = \mathcal{T} v_{\theta}$. We argue that $(\theta , w)$ together constitute a perfectly valid solution to the original problem, but this pair is excluded by existing algorithms only because $\theta \neq w$.

Moving into our algorithmic contribution, we develop a novel and practical RL algorithm that solves the reformulated problem. The new algorithm, referred to as Lookahead-Replicate (LR), is comprised of two alternating operations. The Lookahead step is designed to handle the first constraint, $v_{w} = \mathcal{T} v_{\theta}$. It employs the bootstrapping technique which is best embodied by Temporal-Difference (TD) learning~\cite{TD}. Meanwhile, the Replicate step handles the second constraint $v_\theta = v_w$ by minimizing the mean-squared error between the two value functions $v_\theta$ and $v_w$. Our Replicate step stands in contrast to the prevalent target-synchronization step in deep RL, which proceeds by copying (duplicating rather than replicating) the latest online network by using either hard frequency-based~\cite{DQN} or soft Polyak-based~\cite{LillicrapDDPG2016} updates. We present theoretical results demonstrating that the algorithm converges to a pair of parameters $(\theta,w)$ that satisfies the two constraint jointly.

Inspired by the LR algorithm, we finally augment the target-update in Rainbow~\cite{rainbow}. Equipped with this LR-based update, the resultant deep RL agent significantly outperforms the original Rainbow agent on the standard Atari benchmark~\cite{bellemare_atari}. This result provides empirical evidence that solving the reformulated problem is doable, and that doing so can ultimately lead into better performance compared to the existing formulation.

\section{Background}\label{sec:pre}
Reinforcement learning (RL) is the study of agents that use their environmental interaction to find a desirable behavior~\citep{RLbook}. The Markov Decision Process~\citep{PutermanMDP1994}, or MDP, is used to mathematically define the RL problem. An MDP is specified by a tuple $\langle \mathcal{S,A, R,P,\gamma}\rangle$, where $\mathcal{S}$ is the set of states and $\mathcal{A}$ is the set of actions. At a particular timestep $t$, the agent is in a state $S_t$ and takes an action $\mathcal{A}_t$. It then receives a reward signal $R_t$ according to $\mathcal{R:S\times A}\rightarrow\! \mathbb R$ and then transitions into a next state $S_{t+1}$ based on the transition probabilities $\mathcal{P:S \times A \times S\rightarrow}\ [0,1]$. Finally the scalar $\gamma$ discounts rewards that are received in the future.

The agent interacts with the environment by sampling an action from a policy $\pi: \mathcal{S}\rightarrow \mathcal{P(A)}$, a mapping from states to a probability distribution over actions. The agent aims to learn a policy that achieves high rewards in the long term. More formally, the aim is to maximize the expected discounted sum of future rewards, which is referred to as the value function, and is defined as follows: ${v^{\pi}(s)= \EE\big[\sum_{t=0}^{\infty}\gamma^{t} R_t \mid S_0 = s, \pi \big]}$. The value function can be written recursively by leveraging the Bellman operator $\mathcal{T}^{\pi}$, which is defined as follows:
$$\big[\mathcal{T}^{\pi}v\big](s)=\sum_{a\in\mathcal{A}}\pi(a\mid s)\big(\mathcal{R}(s,a)+\gamma\!\sum_{s'\in \mathcal{S}}\! \mathcal{P}(s,a, s') v(s')\big) \ .$$
In problems where $|\mathcal{S}|$ is small a table could be leveraged to represent the value function $v$. In large-scale problems, however, it is undesirable to learn a separate number per state, so we resort to function approximation. In this case, we approximate the value function $v_{\theta}\approx v^{\pi}$ with the parameter $\theta\in\Theta$. In learning $v_{\theta}$, we leverage the Bellman operator which we discuss next.

\section{Revisiting Value-Function Approximation}\label{sec:revistvalue}
In this section we provide an alternative formulation for the value-function approximation task in RL. This task is often formulated as finding a parameter $\theta\in\Theta$ whose corresponding value function $v_{\theta}$ is the fixed-point of the Bellman operator. We specify the set of such desirable parameters as follows:
$$\mathcal{F}_{single} =  \{\theta \in \Theta \mid  v_{\theta} = \mathcal{T}^{\pi} v_{\theta}\}\ .$$
We use the subscript \textit{single} in the set above to accentuate that this task is solved by algorithms that operate in the space $\Theta$. This means that only a single parameter is learned across learning. The two most notable algorithms operating in this space are the original TD algorithm without a delayed target network~\cite{TD}, as well as the residual-gradient algorithm~\cite{residual}.

In contrast, there exists a second class of algorithms in which learning is operated in a lifted space, namely $\Theta \times \Theta$. Chief among these algorithms is TD with a delayed target network which lies at the core of DQN \cite{DQN}, Rainbow \cite{rainbow}, and related deep-RL algorithms. The value-function approximation component of these algorithms could be viewed as finding a pair of parameters in the following set:
$$\mathcal{F}_{pair} =  \{\theta \in \Theta,\ w \in \Theta \mid  v_{w} = \mathcal{T}^{\pi} v_{\theta}\ \textrm{and}\ \theta = w \}\ .$$
The bootstrapping step in these algorithms, which is performed by using Bellman lookahead~\cite{RLbook}, is designed to handle the first constraint $v_{w} = \mathcal{T}^{\pi} v_{\theta}$. Meanwhile, these algorithms also perform a parameter duplication step, $\theta\leftarrow w$, which ensures the second constraint. We connect the two sets $\mathcal{F}_{single}$ and $\mathcal{F}_{pair}$ using our first claim:

\textbf{Claim 1:} $\theta \in \mathcal{F}_{single}$ if and only if $(\theta,\theta)\in \mathcal{F}_{pair}\ $.

Prior work established the benefits of operating in $\mathcal{F}_{pair}$ by maintaining a pair of parameters when learning the value function~\cite{DQN}. However, does the set $\mathcal{F}_{pair}$ capture all valid solutions to the original value-function approximation problem?

To answer this question, we need to understand the solutions that are excluded based on each individual constraint forming $\mathcal{F}_{pair}$. Clearly, $(\theta, w)$ must satisfy $v_{w} = \mathcal{T}^{\pi} v_{\theta}$. However, it is not clear why $(\theta, w)$ must also satisfy the second constraint $\theta = w$. 

Notice that eliminating the constraint $\theta = w$ can violate the notion of learning a single value function. However, forcing the two parameters to be equivalent is an overkill provided that our true goal is to just obtain a single value function. Instead, we propose to achieve this goal by substituting the parameter-based constraint with a direct function-space constraint, namely $v_{\theta} = v_{w}$. This leads us to a new characterization of the solution set:
$$\mathcal{F}_{value} =  \{\theta \in \Theta,\ w \in \Theta \mid  v_{w} =  \mathcal{T}^{\pi} v_{\theta}\ \textrm{and}\ v_\theta = v_w \}\ .$$
Comparing the two sets $\mathcal{F}_{pair}$ and $\mathcal{F}_{value}$, notice that $\mathcal{F}_{value}$ does not necessarily require the two parameters to be equivalent ($\theta = w$), merely that the two parameters should provide a single value function ($v_\theta = v_w$). To further highlight the difference between the two sets, suppose that we have found a pair of parameters $\theta \neq w$ that jointly satisfies $v_{w} =  \mathcal{T}^{\pi} v_{\theta}$ and $v_\theta = v_w$. Then, we arguably have found a perfectly valid solution to the original task of value-function approximation, but such a pair is excluded from $\mathcal{F}_{pair}$ just because $\theta \neq w$. Conversely, any point in $\mathcal{F}_{pair}$ is actually included in $\mathcal{F}_{value}$ as we highlight below. The set $\mathcal{F}_{value}$ is therefore constituting a more inclusive solution characterization for the original value-function approximation task.

\textbf{Claim 2:} $\mathcal{F}_{pair} \subset \mathcal{F}_{value}$.

Combining Claims 1 and 2, we can conclude that if $\theta\in F_{single}$, then $(\theta , \theta) \in F_{value}$. 
Moreover, the set $\mathcal{F}_{pair}$ can exclude a large number of perfectly valid solutions to the original value-function approximation task. We quantify this gap below.

\textbf{Claim 3:} $|\mathcal{F}_{pair}| \leq |\mathcal{F}_{value}| \leq |\mathcal{F}_{pair}|^2$.

The lower bound follows immediately from Claim 2. To prove the upper bound, suppose that only a single value function can satisfy the Bellman Equation. Then, choose (with repetition) any two parameters $\theta$ and $w$ from $\mathcal{F}_{single}$. We know that $v_{\theta} = v_{w}$ because there is only one fixed-point. We also know that $v_{w} = \mathcal{T}^{\pi}v_{\theta}$ since otherwise $\theta$ and $w$ cannot be in $\mathcal{F}_{single}$. Therefore, $(\theta,w)\in\mathcal{F}_{value}$. Now, since we chose with repetition, we have $|\mathcal{F}_{single}|^2$ such pairs of parameters. We conclude the upper bound in light of the fact that $|\mathcal{F}_{pair}|=|\mathcal{F}_{single}|$ due to Claim 1.

\section{Lookahead-Replicate}\label{sec:fvalue}
Having explained the advantage of $\mathcal{F}_{value}$, we now desire to develop a practical algorithm that finds a solution in $\mathcal{F}_{value}$. Our first step is to convert each individual constraint in $\mathcal{F}_{value}$ into a loss function.
A natural choice is to employ the least-square difference between the two sides of the equations, leading us to the pair of loss functions 
$$H(\theta, w) = \norm{v_{w} - \mathcal{T}^{\pi} v_{\theta}}_{D}^{2} \,\, \textrm{and} \,\, G(\theta, w) = \norm{v_{\theta} - v_{w}}_{D}^{2}\ .$$ 
Here, $\norm{x}_{D} = \sqrt{x^{\top}Dx}$ and $D$ is a diagonal matrix with entries $d^{\pi}(s_1),...,d^{\pi}(s_n)$, and $n=|\mathcal{S}|$. An ideal pair $(\theta,w)$ is one that fully optimizes the two losses $H$ and $G$. We now present our algorithm, which we refer to as Lookahead-Replicate, that is designed to exactly achieve this goal.
\begin{algorithm}
    \caption{Lookahead-Replicate (LR)}
    \label{alg:main}
    \begin{algorithmic}[1]
	\STATE {\bfseries Input:}  $\theta^{0} , w^{0} , T, K_{L},K_{R},\alpha,\beta $
	\FOR{$t=0,1,\ldots,T-1$}
	    \STATE $w^{t+1} \leftarrow \textrm{Lookahead}(\theta^{t} , w^{t} , \alpha, K_{L})$
            \STATE $\theta^{t+1}\ \leftarrow \textrm{Replicate}(w^{t+1} , \theta^{t} , \beta,K_{R})$
        \ENDFOR
	\STATE {\bfseries Return} $(\theta^{T}, w^{T})$
    \end{algorithmic}
\end{algorithm}

As the name indicates, Lookahead-Replicate is comprised of two individual operations per each iteration. Starting from the Lookeahead operation, we use the Bellman operator and the target network $\theta$ to lookahead. We then employ multiple ($K_L$) steps of gradient descent to minimize the discrepancy between the resultant target $\mathcal{T}^{\pi}v_{\theta}$ and $v_w$. The Lookahead operation is at the core of numerous existing RL algorithms, such as TD~\cite{TD} and Fitted Value Iteration~\cite{gordon1995stable, ernst2005tree}.
\begin{algorithm}
    \caption*{$\ \textrm{Lookahead}(\theta, w,\alpha, K)$}
    \begin{algorithmic}[1]
        \STATE $w^0 \leftarrow w$
  	\FOR{$k=0,1,\ldots,K-1$}
            \STATE \textrm{compute} $\nabla_{w} H(\theta, w^{k})\! =\! \nabla_{w} \norm{v_{w^{k}}-\mathcal{T}^{\pi}v_{\theta}}_{D}^2$ 
            \STATE $w^{k+1}\leftarrow w^{k} - \alpha\cdot\nabla_w H$
        \ENDFOR
        \STATE {\bfseries return} $w^{K}$
    \end{algorithmic}
\end{algorithm}

\begin{algorithm}
    \caption*{$\ \textrm{Replicate}(w, \theta,\beta, K)$}
    \begin{algorithmic}[1]
        \STATE $\theta^{0} \leftarrow \theta$
        \FOR{$k=0,1,\ldots,K-1$}
        \STATE \textrm{compute} $\nabla_{\theta} G(\theta^k, w)\! =\! \nabla_\theta \norm{ v_{\theta^{k}}-v_{w}}_{D}^2$ 
            \STATE $\theta^{k+1}\leftarrow \theta^{k} - \beta\cdot\nabla_{\theta} F$
        \ENDFOR
        \STATE {\bfseries return} $\theta^{K}$
    \end{algorithmic}
    \label{alg:replicate}
\end{algorithm}

However, the Replicate operation is where we deviate from common RL algorithms. Specifically, the traditional way to update the target network is to simply copy the parameters of the online network into the target network, using either a frequency-based ($\theta\leftarrow w$ every couple of steps) or Polyak-based ($\theta\leftarrow (1-\tau)\theta + \tau w$) updates. Recall that these parameter-based updates are performed to achieve the second constraint in the set $\mathcal{F}_{pair}$, namely $\theta = w$. However, our new solution characterization $\mathcal{F}_{value}$ is agnostic to this parameter equivalence. Our Replicate step is free to find any pair of parameters $(\theta, w)$ so long as $v_{\theta} = v_{w}$. To this end, we use gradient descent to minimize the discrepancy between the value functions provided by the target and the online networks directly in the function space. This could be viewed as replicating, rather than duplicating or copying, the online network.

Note that it is straightforward to extend the LR algorithm to the setting where the agent learns the value function from environmental interactions. In this case, we just estimate the gradients of the two loss functions, $\nabla_{w} H$ and $\nabla_{\theta} G$, from environmental interactions (line 3 in both Lookahead and Replicate).  Similarly, LR can easily be extended to the full control setting by using the Bellman optimality operator. Moreover, we may not necessarily use least-squares loss when performing either the Lookeahed or the Replicate steps. Rather we can employ, for instance, a distributional loss. Indeed in our experiments we present an extension of LR to the online RL setting where we use a distributional loss akin to the C51 algorithm~\cite{bellemare17ac51}. In this sense, similar to TD, LR could be viewed as a fundamental algorithm that can naturally facilitate the integration of many of the existing extensions and techniques that are popular in the RL literature~\cite{rainbow}.

\subsection{An Illustrative Example}\label{sec:toy_exp}
In this section, we provide an illustrative example to visualize the sequence of parameters and value functions found by the LR algorithm during learning. The example serves as a demonstration that the LR algorithm is indeed capable of achieving the value equivalence $v_\theta=v_w$, as desired, but it is agnostic to parameter-equivalence $\theta=w$. More concretely, in this example LR converges to a pair $(\theta, w)\in\mathcal{F}_{value}$ that does not belong to $\mathcal{F}_{pair}$ since $\theta\neq w$.

\begin{figure*}[t!]
  \centering
  \begin{subfigure}[b]{0.45\linewidth}
    \includegraphics[width=1\linewidth]{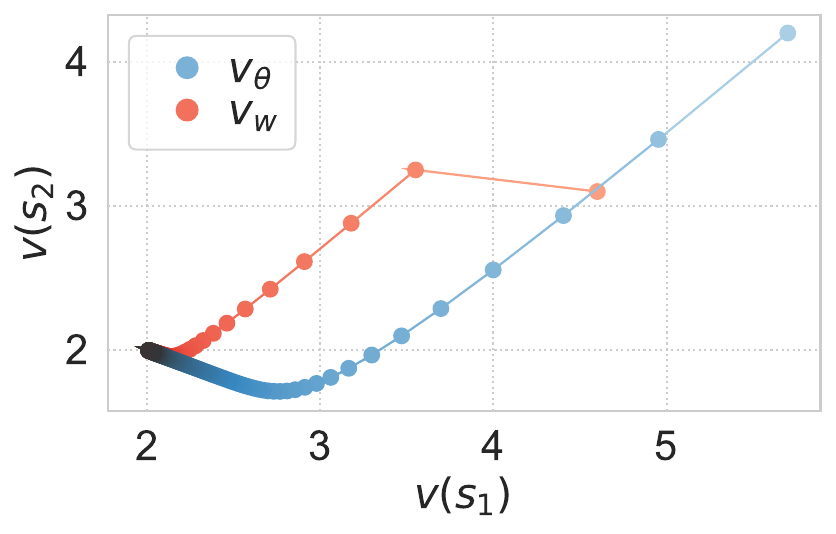}
  \end{subfigure}
  \hspace{1cm}
  \begin{subfigure}[b]{0.33\linewidth}
    \includegraphics[width=1\linewidth]{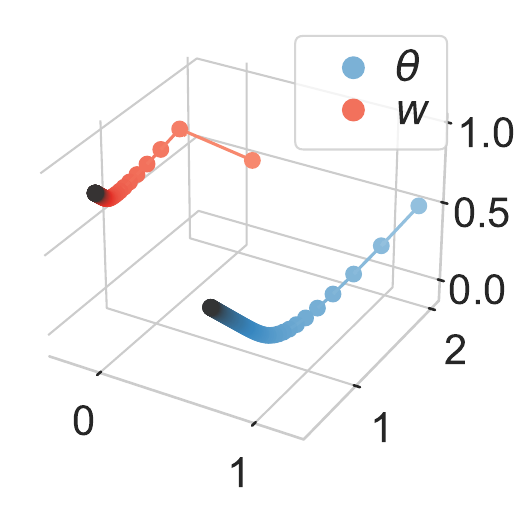}
  \end{subfigure}
  \caption{A sample trial of the LR algorithm on the Markov chain.
  The iterations of $v_\theta$ and $v_w$ in the value space (left), and the iterations of the two parameters $\theta$ and $w$ in the parameter space (right). Notice that LR converges to a pair of points where the value functions are equivalent $v_{\theta} = v_{w}$ despite the fact that $\theta \neq w$.
  }
\label{fig:mini_example}
\end{figure*}

In this simple prediction example (no actions), we just have two states $\mathcal{S}=\{s_1,s_2\}$ and the transition matrix $P=\begin{bmatrix} 0.6&0.4\\0.2& 0.8 \end{bmatrix}$, $\gamma=1/2$, and reward is always 1. The true value is 2 in both states. We construct state feature vectors $\phi(s_1) = [1, 2, 1]^\top, \phi(s_2) = [1, 1, 2]^\top$ and use linear value function approximation, i.e. $v_\theta(s)=\phi(s)^{\top}  \theta$ and $v_w(s)=\phi(s)^{\top}  w$.

In Figure~\ref{fig:mini_example}, we show the actual iterations of $\theta$, $w$ in the parameter space and their corresponding value functions $v_\theta$ and $v_w$ under the LR algorithm. While the two parameters $\theta$ and $w$ converge to different points, their resultant value functions nevertheless converge to the true values: $v(s_1) = v(s_2) = 2$. See Appendix~\ref{sec:appendix-example-ones} for more detail. Also notice that the value functions can even be parameterized differently, that is using two completely separate hypothesis spaces. This stands in contrast to TD-like algorithms where we must use the same hypothesis space for the two networks in light of the parameter duplication step. See Appendix~\ref{sec:appendix-example-twos} for such an example.

\subsection{Convergence Analysis}

In this subsection, we formally prove that the Lookahead-Replicate algorithm converges to a pair of parameters $(\theta,w)\in\mathcal{F}_{value}$. We first state our assumptions, which we later make use of when proving the result.
\begin{assumption}
For all $\theta \in \Theta$, $v_\theta$ is differentiable and Lipschitz-continuous. Formally, given $\theta_1,\theta_2\in\Theta$, we have $\norm{v_{\theta_1}-v_{\theta_2}}\leq \kappa_1 \norm{\theta_1-\theta_2}$ for some $\kappa_1>0$.
\end{assumption}
Following recent work~\citep{asadi2023td}, we also assume that the loss function in the Lookahead step, $H(\theta, w) = \norm{v_{w} - \mathcal{T}^{\pi} v_{\theta}}_{D}^{2}$, satisfies the following three assumptions:
\begin{assumption}\quad \label{assum1}
\begin{enumerate}
    \item For all $\theta_1 , \theta_2$, there exists $F_{\theta} > 0$ such that:     
        \begin{equation*}
            \|\nabla_{w} H(\theta_{1} , w) - \nabla_{w} H(\theta_{2} , w)\| \leq F_{\theta}\|\theta_{1} - \theta_{2}\|\ .   
        \end{equation*}
    \item There exists an $L > 0$ such that:
        \begin{equation*}
            \|\nabla_{w} H(\theta , w_{1}) - \nabla_{w} H(\theta , w_{2})\| \leq L\|w_{1} - w_{2}\|\ .   
        \end{equation*}
    \item The function $H(\theta , w)$ is $F_{w}$-strongly convex in $w$, i.e., for all $w_1 , w_2$ we have
        \begin{align*}
             \big(\nabla_{w} H(\theta, w_{1}) - \nabla_{w} H(\theta, w_{2})\big)^{\top}(w_{1} - w_{2}) \\
             & \hspace{-0.6in} \geq F_{w}\|w_{1} - w_{2}\|^{2}\ .   
        \end{align*}
\end{enumerate}
\end{assumption}
Note that these assumptions are satisfied in the linear case~\cite{lee2019target} and even in some cases beyond the linear setting~\cite{asadi2023td}.

We are now ready to state the main theoretical result of our paper:
\begin{theorem}
\label{thm:interm3_strong-M-temp}
Let $\left\{ (\theta^{t} , w^{t}) \right\}_{t \in \mathbb{N}}$ be a sequence of parameters generated by the Lookahead-Replicate algorithm. Assume $F_w>\max\{F_\theta, 7\kappa_1^2,\frac{4\kappa_1^2}{1-\zeta}\}$. Given appropriate settings of step-sizes $(\alpha,\beta)$, where $\alpha,\beta,\zeta$ explained in the Appendix:
\begin{align*}
    \norm{(\theta^{t+1} , w^{t+1}) - (\theta^\star , w^{\star})} \leq \sigma\norm{(\theta^{t} , w^{t}) - (\theta^\star , w^{\star})},
\end{align*}
for some $\sigma < 1$. In particular, the pair $(\theta^{\star} , w^{\star})\in \mathcal{F}_{value}$.
\end{theorem}

Appendix \ref{sec:appx-th-proof} includes a more detailed statement of the theorem as well as the individual lemmas and steps used to prove the theorem.

\begin{corollary}\label{cor:main}
Under the condition of Theorem~\ref{thm:interm3_strong-M-temp}, as $t\rightarrow\infty$,
\[
\norm{v_{\theta^t}-v_{w^t}} \leq \sqrt{2}\kappa_1\sigma^{t-1}\norm{(\theta^{0} , w^{0}) - (\theta^\star , w^{\star})}\rightarrow 0.
\]
In addition, we also have
\[
\norm{v_{w^t}-\mathcal{T}v_{\theta^t}} \leq \sqrt{2}\kappa_1\sigma^{t-1}\norm{(\theta^{0} , w^{0}) - (\theta^\star , w^{\star})}\rightarrow 0.
\]
\end{corollary}
Corollary~\ref{cor:main}, at the value level, shows that $(\theta^t,w^t)$ converges to a point in $\mathcal{F}_{value}$.

\section{Experiments}
\label{sec:exp}
To evaluate LR in a large-scale setting, we now test it on the Atari benchmark~\cite{bellemare_atari}. Our baseline is Rainbow~\cite{rainbow}, which is viewed as a combination of important techniques in value-function approximation. We used the implementation of Rainbow from the Dopamine~\citep{dopamine}, and followed Dopamine's experimental protocol to carry out our experiments.

\begin{figure*}[t!]
  \centering
  \begin{subfigure}[b]{0.33\linewidth}
    \includegraphics[width=1\linewidth]{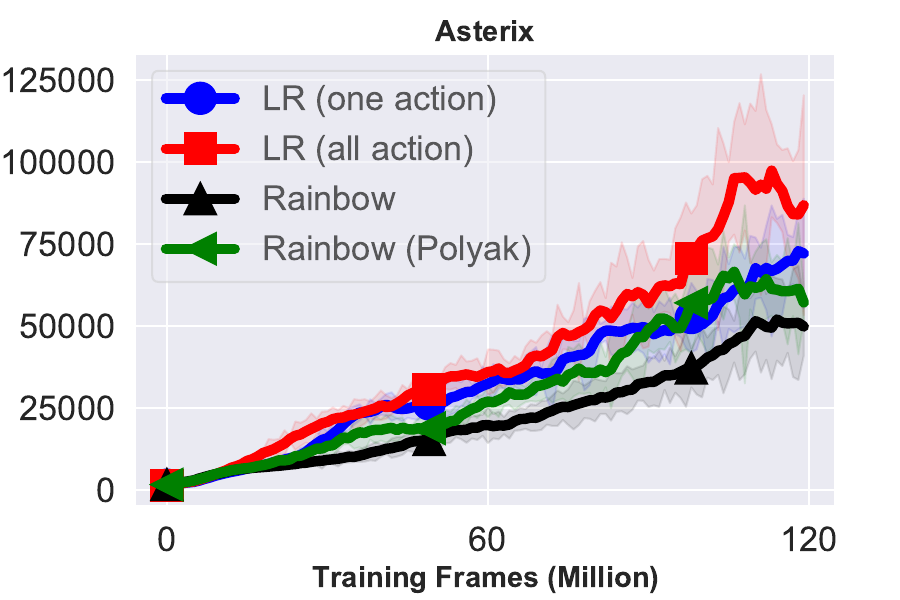}
  \end{subfigure}
  \hfill 
  \begin{subfigure}[b]{0.33\linewidth}
    \includegraphics[width=1\linewidth]{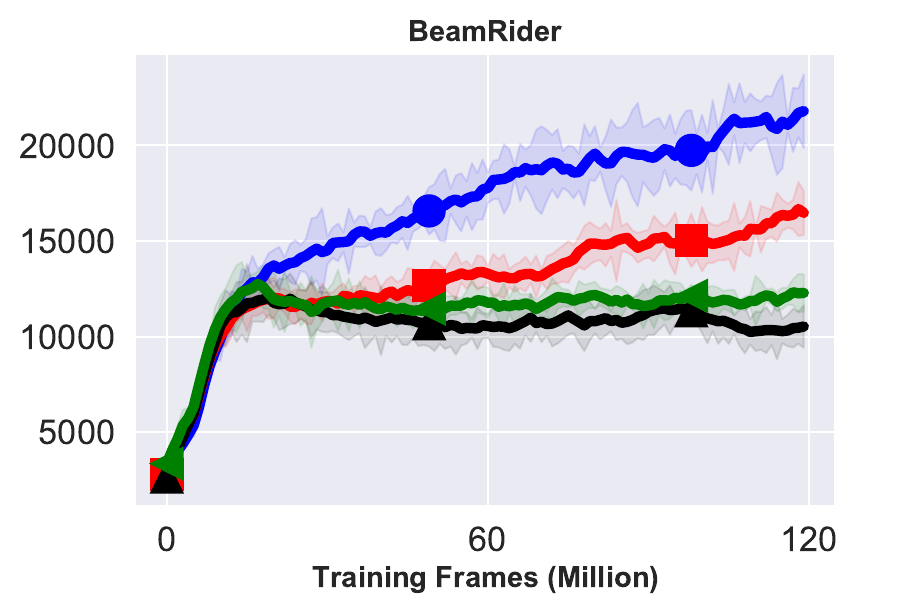}
  \end{subfigure}
  \hfill 
  \begin{subfigure}[b]{0.33\linewidth}
    \includegraphics[width=\linewidth]{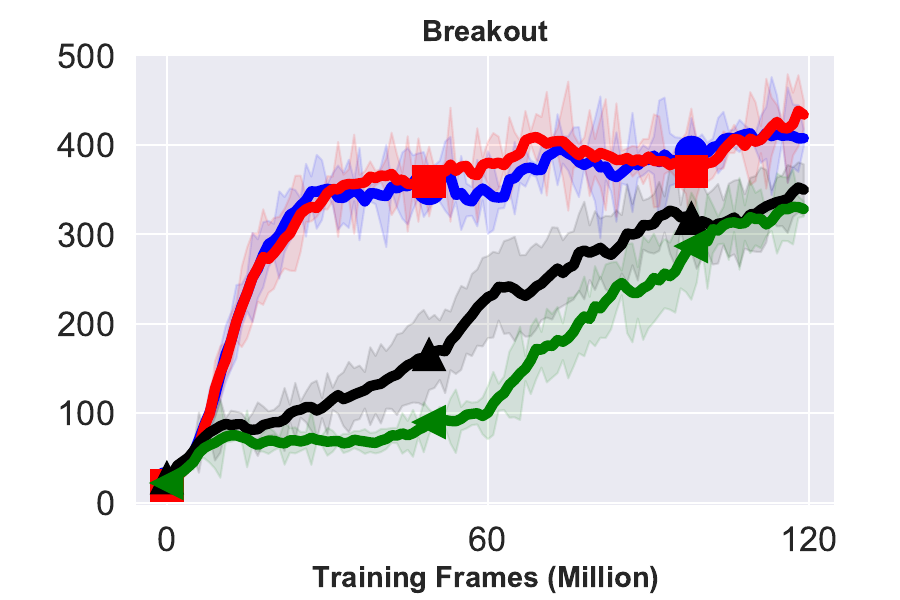}
  \end{subfigure}
  \begin{subfigure}[b]{0.33\linewidth}
    \includegraphics[width=1\linewidth]{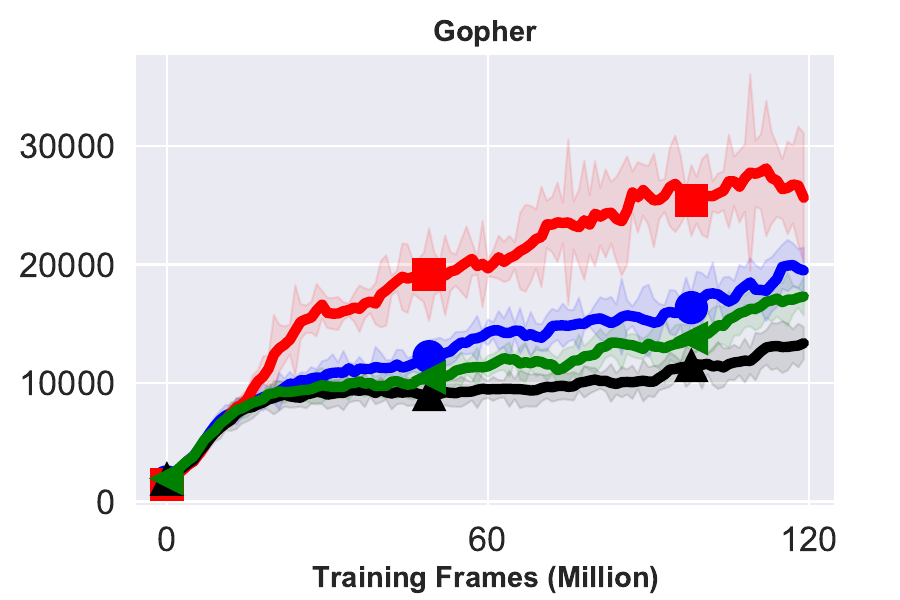}
  \end{subfigure}
  \hfill 
  \begin{subfigure}[b]{0.33\linewidth}
    \includegraphics[width=1\linewidth]{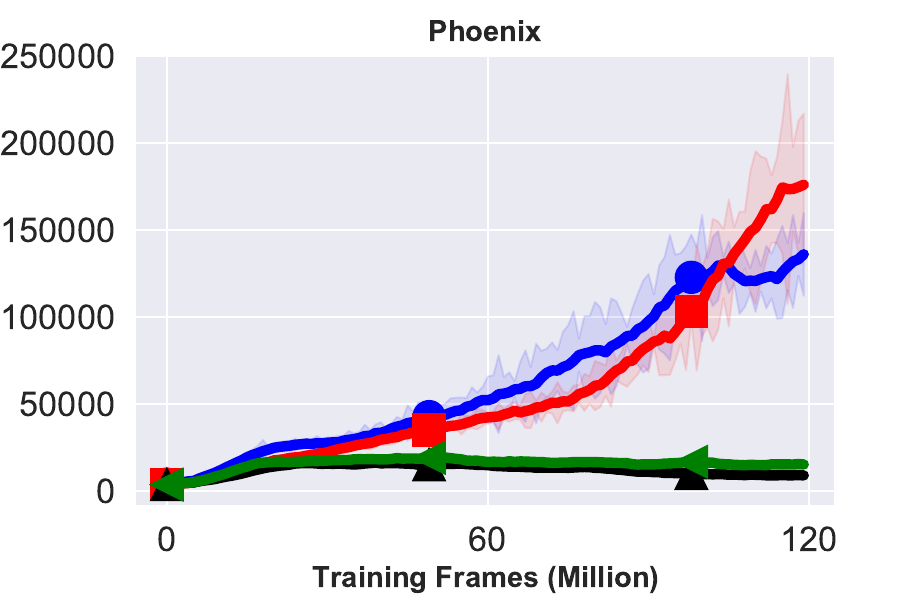}
  \end{subfigure}
  \hfill 
  \begin{subfigure}[b]{0.33\linewidth}
    \includegraphics[width=\linewidth]{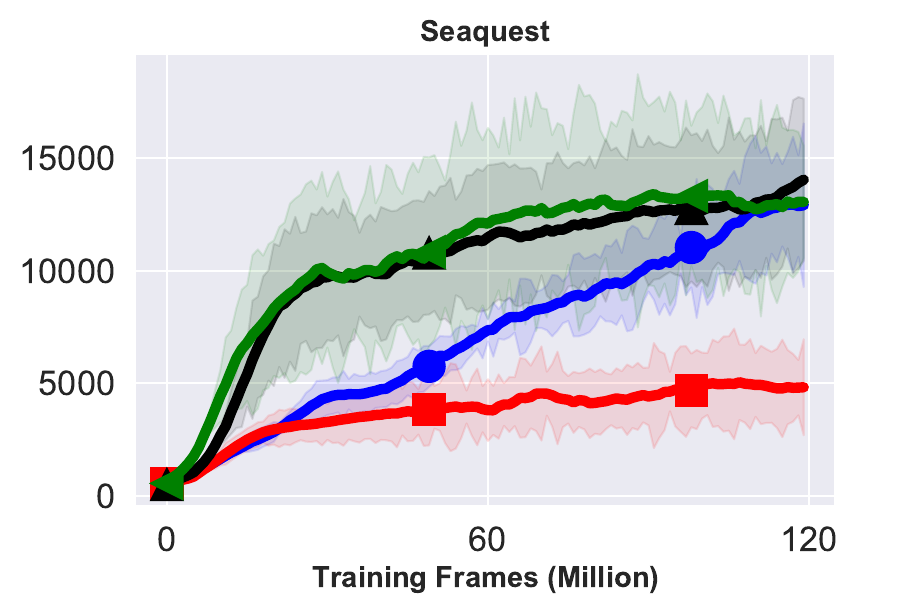}
  \end{subfigure}
  \caption{A comparison between two variations of LR (blue and red) with Rainbow under the default frequency-based (black) and Polyak-based (green) updates. The first variation of LR (blue) performs the Replicate step by sampling states and actions from the replay buffer and minimizing the value difference between the target and online network. The second variation of LR (red) only samples states from the replay buffer and minimizes the value difference for all actions in each sampled state. Results are averaged over 5 random seeds.}
  \label{fig:K800} 
\end{figure*}

We ran Rainbow with two common target-network updates, namely the hard frequency-based update, which is the default choice in Rainbow, as well as the soft Polyak update. We used the default hyperparameters from Dopamine. We reasonably tuned the hyper-parameters associated with each update. We noticed that tuning the frequency parameter has minimal effect on Rainbow's performance. This is consistent with the recent findings of~\citet{asadi2023resetting} who noticed that the performance of Rainbow is flat with respect to the frequency parameter. In the case of Rainbow with Polyak update ($\theta \leftarrow \tau w + (1-\tau)\theta$), we tuned the $\tau$ parameter and found that $\tau=0.005$ performs best.

\begin{figure*}[t!]
  \centering
  \begin{subfigure}[b]{0.33\linewidth}
    \includegraphics[width=1\linewidth]{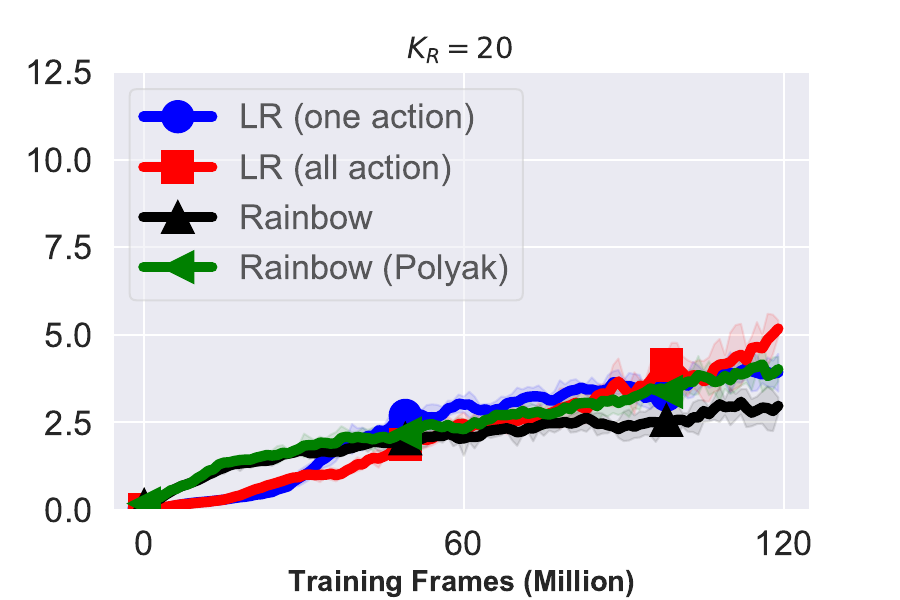}
  \end{subfigure}
  \hfill 
  \begin{subfigure}[b]{0.33\linewidth}
    \includegraphics[width=1\linewidth]{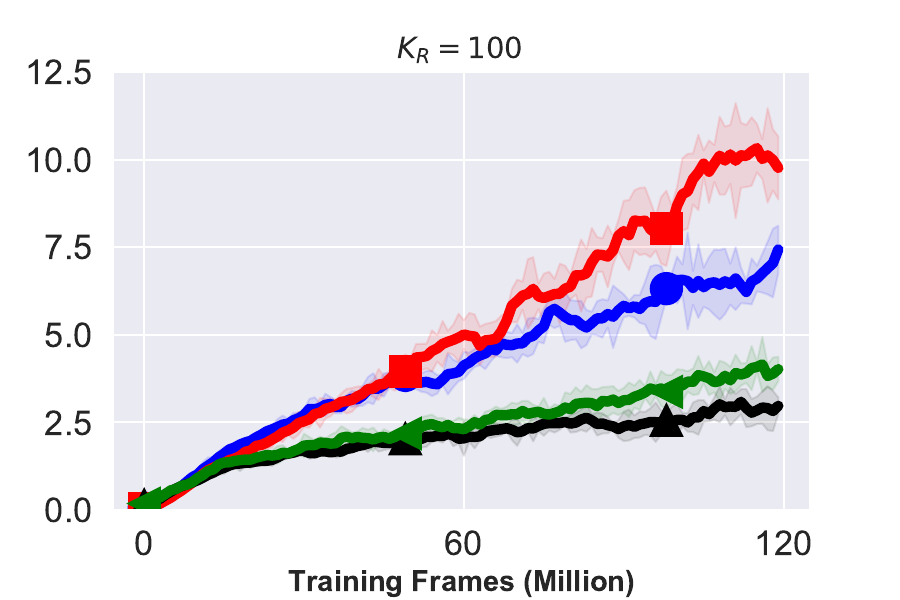}
  \end{subfigure}
  \hfill 
  \begin{subfigure}[b]{0.33\linewidth}
    \includegraphics[width=\linewidth]{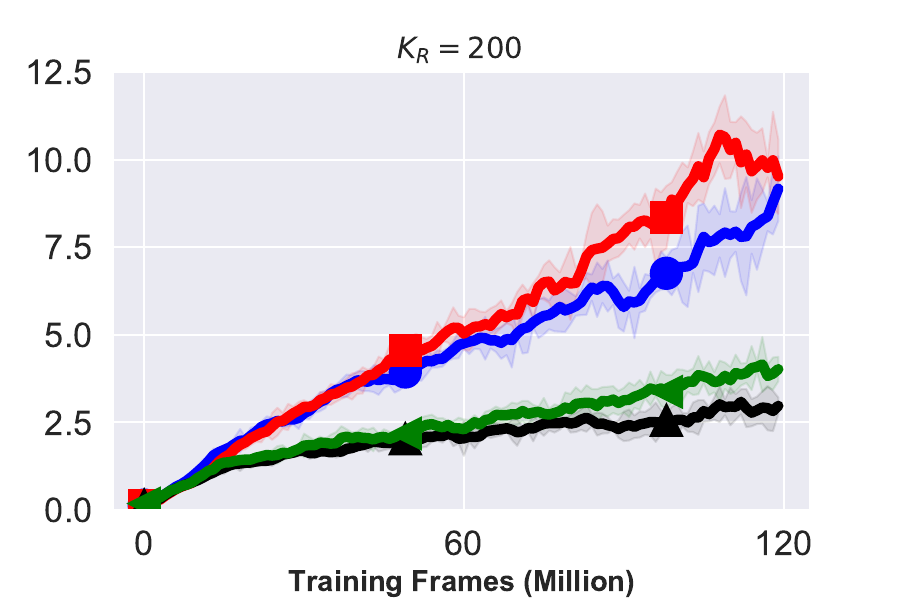}
  \end{subfigure}
  \begin{subfigure}[b]{0.33\linewidth}
    \includegraphics[width=1\linewidth]{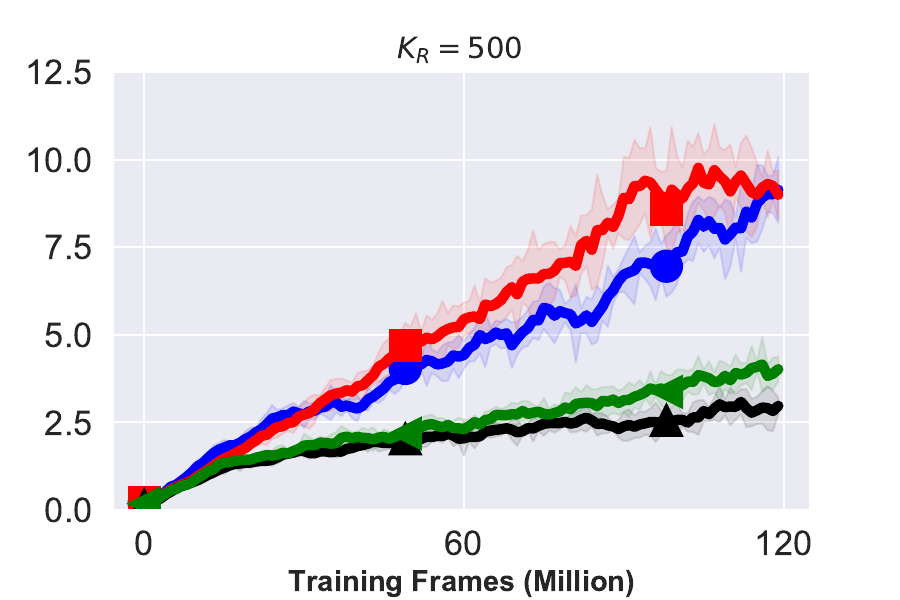}
  \end{subfigure}
  \begin{subfigure}[b]{0.33\linewidth}
    \includegraphics[width=1\linewidth]{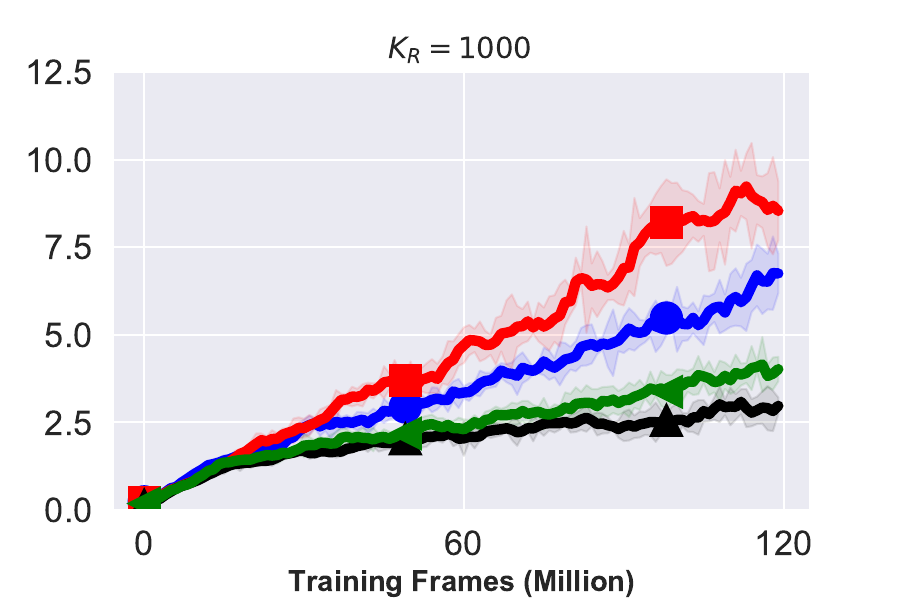}
  \end{subfigure}
  \hfill 
  \caption{A Comparison between Rainbow and LR  with different values of $K_R$ which is the number of gradient updates to the target network before updating the online network. The y-axis is the median of human-normalized performance across the 6 games. We are using 5 random seeds to aggregate the results. Higher is better. }
  \label{fig:many_K} 
\end{figure*}

Moving to LR, we used $K_L=2000$ in the Lookahead operation as the choice of frequency hyper-parameter in Rainbow. Notice that the Replicate step introduces two new hyper-parameters, specifically the number of optimizer updates to the target network $K_R$ and the learning rate of the optimizer itself. In this case, we used the same learning rate as we did for optimizing the online network in Rainbow namely $6.25\times 10^{-5}$ (again, the default choice in the Dopamine). We did not tune this parameter. We also chose the Adam optimizer to update the target network, which is again the default optimizer used to train the online network in the Rainbow implementation from Dopamine. This allowed us to only focus on tuning the $K_R$ parameter. We did not modify other hyper-parameters.

Following the distributional perspective presented by~\citet{bellemare17ac51}, Rainbow approximates the distribution of returns rather than the expected return~\cite{rainbow}. This means that the neural network outputs a distribution of the return instead of a single value. As described by~\citet{bellemare17ac51}, define a set of atoms $z_i$ for $1 \leq i \leq N$, then the state-action value function is represented as follows $q_{\theta}(s,a) = \sum_{i = 1}^{N} z_i \cdot p_i(s,a;\theta)$. Thus, to adapt our approach to this setting, the Replicate operation employs gradient-descent steps to learn a parameter $\theta$ that minimizes the cross-entropy loss between the distributions $p(s,a;\theta)$ and $p(s,a;w)$, i.e., $\min_{\theta} \text{CE}(p(s,a;\theta), p(s,a;w))$. In practice, this minimization is performed on a batch of states and actions sampled from the replay buffer. 

Moreover, two natural choices exist when selecting which state-action pairs to use in our updates. In the first case, we can update $\theta$ by minimizing the loss \text{CE} on state-actions $\langle s,a \rangle$ sampled from the replay buffer. In the second case, we only sample states from the replay buffer $\langle s\rangle$ and then perform the minimization on all actions across the sampled states. These two variations are referred to as \emph{\textbf{one action}} and \emph{\textbf{all action}} variants of LR in our experiments. These results are shown in Figure~\ref{fig:K800}. We present these results on 6 games to keep the number of experiments manageable and later present comprehensive results on all 55 games. We also used 5 random seeds and present confidence intervals.

We can clearly see in Figure~\ref{fig:K800} that the two LR variants are outperforming their Rainbow counterparts on all but one game. Interestingly, we also see that the Polyak update is roughly as effective as the frequency-based update. Moreover, the all-action variant of LR is able to outperform the one-action counterpart. Thus, we will be using this all-action version for the rest of our experiments.


We are also interested in understanding the effect of changing $K_{R}$ on the overall performance of the LR algorithm. To this end, we repeated the previous experiment, exploring a range of different values for this parameter. These results are shown in Figure~\ref{fig:many_K}. To aggregate all games, we first compute the human-normalized score, defined as $
\frac{\textrm{Score}_{\textrm{Agent}} - \textrm{Score}_{\textrm{Random}} }{ \textrm{Score}_{\textrm{Human}}- \textrm{Score}_{\textrm{Random}} }
$, and then compute the median across 6 games akin to previous work~\cite{wang2016dueling}.

To further distill these results, we present the area under the curve for the two variants of Rainbow, and for LR as a function of $K_{R}$. Notice from Figure \ref{fig:aug} that an inverted-U shape manifests itself, meaning that LR with an intermediate value of $K_{R}$ performs best. To explain this phenomenon, notice that performing a tiny number of updates (small $K_{R}$) would mean not changing the target network too much, and therefore not appropriately handling the constraint $v_{\theta} = v_{w}$. On the other extreme, we can fully handle the constraint $v_{\theta} = v_{w}$ by using very large $K_{R}$, but then doing so can have the external effect of significantly violating the other constraint namely $v_{w} = \mathcal{T} v_{\theta}$. Therefore, a trade-off exists, so an intermediate value of $K_{R}$ performs best.
\begin{figure}[h]
\centering\captionsetup[subfigure]{justification=centering,skip=0pt}
\begin{subfigure}[t]{0.4\textwidth} 
\centering 
\includegraphics[width=\textwidth]{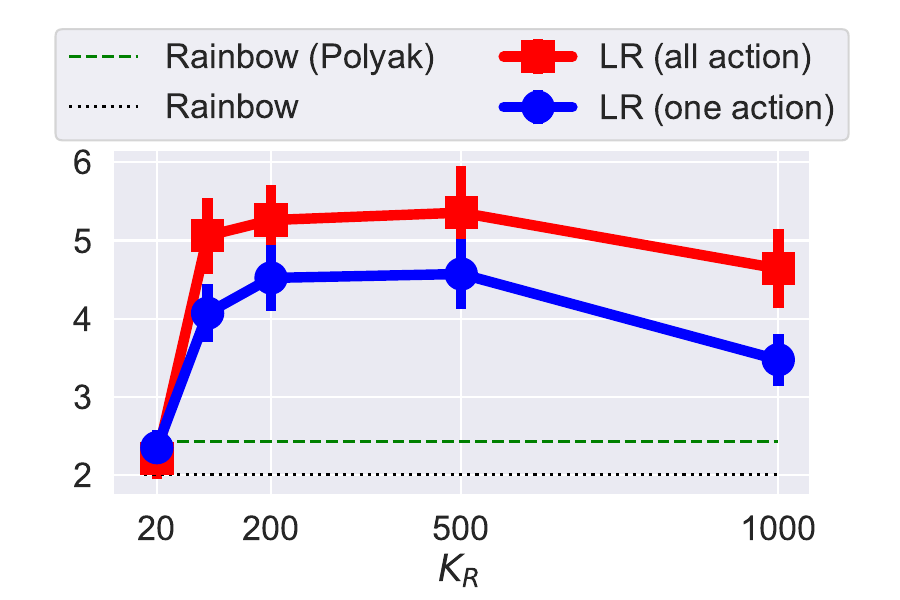} 
\end{subfigure}
\caption{Performance of the LR agent as a function of $K_R$ the number of updates to the target network in the Replicate step. Higher is better. Notice that an intermediate value of $K_R$ performs best.}
\label{fig:aug} 
\end{figure}

The above ablations clearly display the benefit of LR over standard target network updates. We hypothesize that by updating the target network via a cross-entropy loss and gradient-based optimizer, it also benefits from the implicit regularization effect of stochastic gradient descent. To verify this hypothesis, we examine the norm of target and online Q networks's parameters in LR and Rainbow. As Figure \ref{fig:wnorm_mean} shows, LR reduced both the online and target network's norm, which may serve as an implicit regularization on the Q networks. Notice the difference in the magnitude of the norm of the online and the target network in LR, which indicates that LR typically finds a solution where $w\neq\theta$.
\begin{figure}[h]
\centering\captionsetup[subfigure]{justification=centering,skip=0pt}
\begin{subfigure}[t]{0.4\textwidth} 
\centering 
\includegraphics[width=\textwidth]{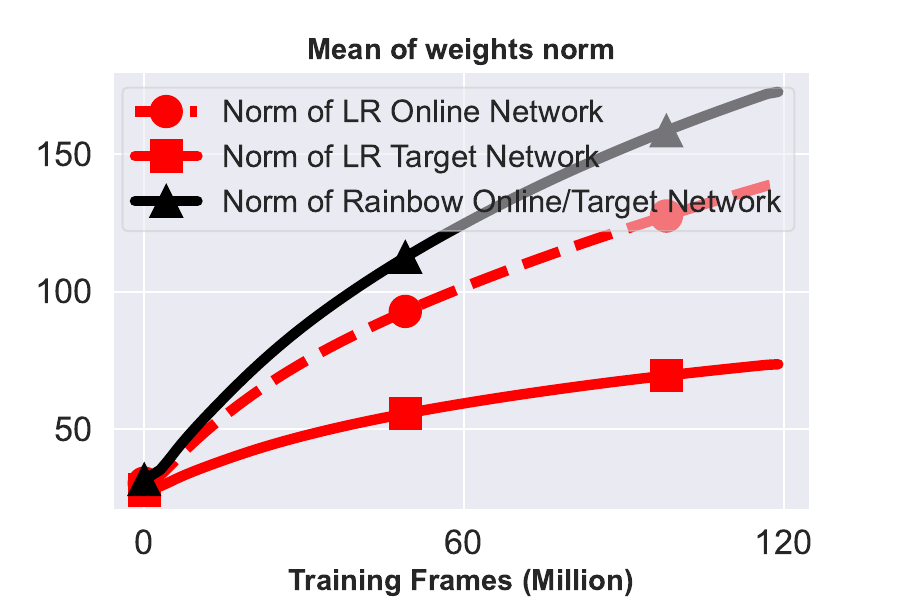} 
\end{subfigure}
\caption{Parameter norm of target and online Q network in different algorithms, averaged over 6 games in Figure \ref{fig:K800}.}
\label{fig:wnorm_mean} 
\end{figure}

We finally would like to evaluate LR beyond these 6 games. To this end, we fix $K_R=800 = 0.1\cdot K_L$ and also chose the all action implementation of LR. We then ran LR and Rainbow on all 55 Atari games. The aggregate learning curve and the final asymptotic comparison between the two agents are presented in Figures~\ref{fig:55_median} and~\ref{fig:55_assymptote}, respectively. Overall, LR can outperform Rainbow by both measures.

\begin{figure}[h]
\centering\captionsetup[subfigure]{justification=centering,skip=0pt}
\begin{subfigure}[b]{0.4\textwidth} 
\centering 
\includegraphics[width=\textwidth]{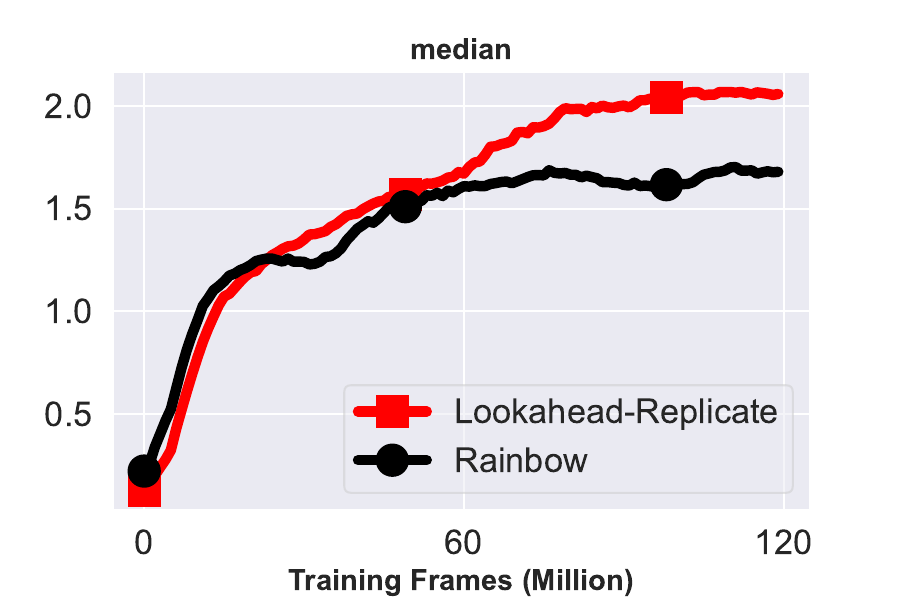}
\end{subfigure}
\begin{subfigure}[b]{0.4\textwidth} 
\centering 
\includegraphics[width=\textwidth]{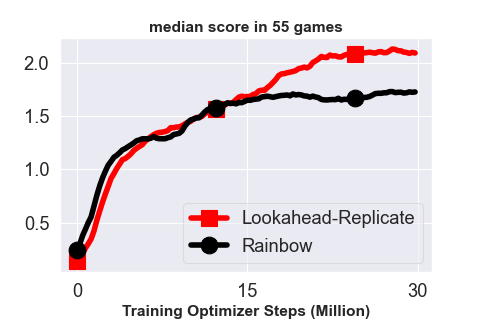}
\end{subfigure}
\caption{A comparison in terms of human-normalized median across 55 games and 5 seeds between LR and Rainbow based on the number of frames on the x-axis (Top). Here the two agents use the same amount of data per each point on the x-axis but note that Rainbow is taking slightly lower number of gradient steps (9\% less) per a fixed number of training frames due to the additional steps in updating the target network by LR. To ensure a fair comparison in terms of computation, we also report the same result but with the total number of optimizer steps on the x-axis (Bottom). Here Rainbow is using slightly more data than LR for any given point on the x-axis. LR is outperforming Rainbow in both cases.}
\label{fig:55_median}
\end{figure}

Note that in updating the target network, LR takes additional gradient steps, unlike Rainbow and given the same amount of data, the total number of gradient steps taken by LR is slightly higher than the number of gradient steps taken by Rainbow. To account for this discrepancy, we provide an alternative presentation of Figure \ref{fig:55_median} (Bottom) where on the x-axis we report the number of gradient steps taken by each agent, and on the y-axis we present the performance of the corresponding agent having performed the number of gradient steps on x-axis. This would ensure that we are not giving any computational advantage to LR. 

Observe that under this comparison, LR is still outperforming Rainbow. We also would like to highlight that in this comparison, given a point on the x-axis, LR has experienced a 9\% lower amount of data relative to Rainbow. To understand where this number comes from, recall that we always perform one step of online network update per $4$ frames. We repeat this for some time, and then perform $200$ updates to the target network. Therefore, after $T$ number of overall updates, we have performed roughly $(0.91)*T$ updates to the online network (due to the ratio $2000/(2000+200)$), and therefore we have seen only 91\% of the amount of interaction experienced by Rainbow. It therefore means that under this comparison we are fair in terms of compute, but we have now given more data to Rainbow over LR. That said we still observe that LR is the more superior agent overall.

\begin{figure}[h]
\centering\captionsetup[subfigure]{justification=centering,skip=0pt}
\includegraphics[width=.48\textwidth]{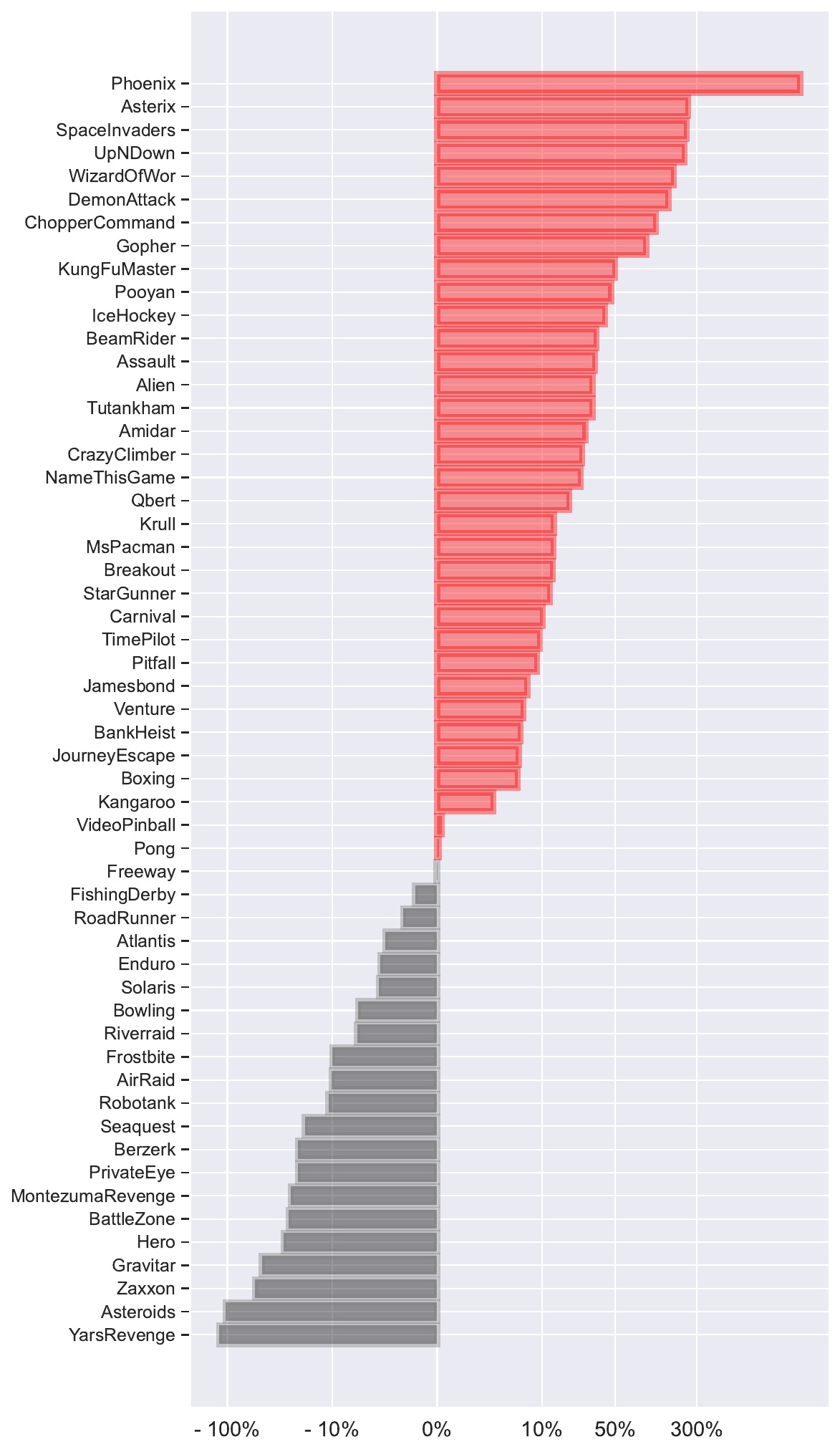} 
\caption{A comparison between the asymptotic performance of LR and Rainbow on 55 Atari games.} 
\label{fig:55_assymptote} 
\end{figure}

\section{Related Work}
In this work we emphasized that it is necessary for online and target networks to be equivalent only in the function space, not in the parameter space. To the best of our knowledge this is a significant deviation from all of the existing work. Nonetheless, there are a few examples in the literature where learning the value function is operated based on some information related to the function space. For instance~\citet{dabney2014natural} and \citet{knight2018DQNnatural} use natural gradient to perform value-function optimization by considering the geometry of the parameter space. Still, using a different parameterization for online and target networks is absent in these efforts. A similar trend exists when learning the policy in RL. In this case, the natural gradient can be employed to perform distribution-aware updates~\citep{kakade2001natural, PetersNaturalGradient08}. Using variations of natural-gradient policy optimization along with trust regions has recently become quite popular leading to a number of competitive baselines in the deep RL setting~\citep{trposchulman15,schulman2017proximal, abdolmaleki2018maximum,fakoor20a}.

Our primary algorithmic contribution was to present the Lookahead-Replicate (LR) algorithm, which updates the target network in a manner that is fundamentally different than standard updates. In deep RL, it is common to update the target network using one of two strategies: either hard frequency-based (copying or duplicating $\theta$ into $w$ every couple of steps) or Polyak ($\theta = \tau w + (1-\tau)\theta$ at every step, also known as soft updates). The frequency-based update is commonly applied in value-function-only approaches to RL such as DQN~\cite{DQN}, Double Q-learning~\cite{ van2016deep}, C51~\cite{bellemare17ac51}, Rainbow~\cite{rainbow}, DQN Pro~\cite{dqnproAsadi2022}, and is particularly effective in MDPs with discrete actions. On the other hand, Polyak is predominantly used in actor-critic algorithms such as DDPG~\cite{LillicrapDDPG2016}, SAC~\cite{haarnoja2018soft}, TD3~\cite{fujimoto2018addressing}, D4PG~\cite{maron2018distributional}, etc., and is mainly effective with continuous-action MDPs. Both updates could be viewed as achieving the goal of learning a single value function by forcing $\theta=w$. Enforcing this parameter-based equivalence is an overkill so long as we have the alternative value-space equivalence proposed in this paper.

Numerous works have studied ways to better utilize the target network during learning. This line of work includes using a proximal term between target and online network parameters to improve the performance~\cite{dqnproAsadi2022}, normalizing the target network to improve efficiency and stability~\cite{Hasselt2016LPropQ}, and applying regularization~\cite{gracshao22a}. Conversely, some works aimed to remove the target network from while not affecting~\cite{kim2019deepmellow} or even improving performance~\cite{bhatt2020crossnorm,gracshao22a}. 


On the more theoretical side, classical temporal difference (TD) learning analysis~\citep{tsitsiklis1996analysis, melo2008analysis, bhandari2018finite} usually focuses on TD in absence of a target parameter. Recently \citet{lee2019target} studied TD with delayed target updates under the linear setting. \citet{asadi2023td} provided novel convergence guarantees that went beyond the linear setting and worked with any setting of the frequency of updates for the target parameter. Convergence of TD was also studied in presence of regularization \cite{liu2012regularized, DeadlyTriadzhang21y}. There are also recent studies on TD in the overparameterized setting~\cite{cai2019neural, xiao2021understanding, lyle2021effect} where the focus of the studies is on the original TD algorithm, which unlike Lookahead-Replicate (LR), enforces the online-target parameter equivalence. Our Lookahead-Replicate algorithm is specifically designed to leverage overparameterization by finding two parameters whose corresponding value functions are equivalent. Last but not least, the pioneer works \cite{maei2010toward,maei2010gq} considers the (Projected) square Bellman error objective with the ``two time scale'' update for the parameters under the linear structure. Our framework essentially reformulates the fixed-point Bellman equation into a bilevel formulation (with the desired solution of the bilevel formulation being $F_{value}$) that allows flexibility for parameters $\theta$ and $w$.

\section{Conclusion}
We presented an alternative formulation for value-function approximation, a problem which lies at the core of numerous approaches to RL. Value-function approximation algorithms commonly maintain two parameters during training. In this context, we showed that it is not necessary, and arguably undesirable, to enforce an equivalence between the parameters. This equivalence is usually enforced to ensure that a single value function is learned. But, we demonstrated a more direct approach to achieving this goal, namely to update the two parameters while ensuring that their provided value-functions are equivalent. Algorithmically, we proposed the new Lookahead-Replicate (LR) algorithm that updates the target parameter in a fashion that is fundamentally different than existing work. In particular, rather than copying the the online parameter, we update the target parameter using gradient-descent steps to minimize the value-difference between the target and the online networks. This new style of update, while simple to understand and implement, led to improved performance in the context of deep RL and on the Atari benchmark. These results demonstrated the benefits of our reformulation as well as our proposed LR algorithm that is designed to solve the reformulated problem.

\section{Future Work}
An important area for future work is to understand the behavior of LR and TD when we scale up the capacity of the neural network. An interesting observation is that the gap between the sizes of the two sets $\mathcal{F}_{pair}$ and $\mathcal{F}_{value}$ grows as we increase the expressiveness of the function approximator. Therefore, it would be interesting to see if the LR algorithm is more conducive to scaling than existing algorithms such as TD. Notice that we designed LR to search for a solution in the larger set $\mathcal{F}_{value}$, whereas TD and similar algorithms were designed to search for a solution in the set $\mathcal{F}_{pair}$. Therefore, our current conjecture is that LR may better harness the power of scaling. It would be very interesting to test the validity of this conjecture in future work.

Notice that the constraint $v_{\theta}= v_{w}$ is agnostic to the specific function class chosen for the value-function approximation task. It would be interesting to define two separate function classes to represent the online and the target networks. In this paper, we mainly focused on the setting where we operate using the same parameter space for both networks ($\Theta \times \Theta$), but in general these two spaces could be different. Indeed, we show an example of running LR in this setting in Appendix B. Notice that TD-like algorithms are inapplicable in this setting because the parameter-space equivalence is meaningless to enforce when we have different hypothesis spaces. A question for future work is to identify scenarios in which using different parameter spaces is fruitful. One such case, inspired by supervised learning, is the area of network distillation~\citep{hinton2015distilling,rusu2015policy} where the goal is to imitate a first teacher network by using an often smaller student network. This direction is beyond the scope of this paper, but is really interesting to explore in future.

It is well-known that TD can exhibit misbehavior when used in conjunction with bootstrapping, arbitrary function approximators, and off-policy updates, the so-called deadly triad~\cite{RLbook}. An important question for future work is to compare the LR algorithm with TD as it pertains to convergence guarantees on simple counter examples as well as the more general cases. The LR algorithm updates the target network in a different fashion than TD, therefore, it would be interesting to understand the impact of this new style of target update on existing convergence guarantees of TD and related algorithms. More generally, numerous questions about the TD algorithm have been studied in the RL literature. Some of these questions are well-understood while others still remain open. In light of our new LR algorithm, these questions can be revisited and addressed in a broader scope. In this paper, we just laid the foundation and set the stage for such questions to be studied in future.

\section*{Impact Statement}
This paper presents a technical work whose goal is to advance the field of reinforcement learning.
While we understand this work to hold significant potential in terms of technical advancements, the scope of its societal impact remains limited at this point.






\bibliography{refs}
\bibliographystyle{icml2024}

\newpage
\appendix
\onecolumn

\section{Proof of Theorem \ref{thm:interm3_strong-M-temp}}\label{sec:appx-th-proof}
Here, we will prove the main theorem. The analysis to be developed now is inspired by the proof technique of the recent paper \cite{asadi2023td}. However they focus on standard TD learning settings while we study a new value function learning algorithm.

We first recall the pair of loss functions $H(\theta, w) = \norm{v_{w} - \mathcal{T} v_{\theta}}_{D}^{2}$ and $R(\theta, w) = \norm{v_{\theta} - v_{w}}_{D}^{2}$ that are optimized in the Lookahead-Replicate algorithm. We will provide analysis for a general value functions $v_{\theta}$ and $v_{w}$ as long as the pair of loss function satisfy Assumption \ref{assum1}. Therefore, when we differentiate the value function $v_{\theta}$, we get the Jacobian of $v_{\theta}$. Therefore, for the simplicity of the developments to come, we compute the relevant gradients of the loss functions:
\begin{equation} \label{GradLoss}
    \nabla_{w} H(\theta, w) = 2\nabla v_{w}D\left(v_{w} - \mathcal{T} v_{\theta}\right) \qquad \text {and} \qquad \nabla_{\theta} G(\theta, w) = 2\nabla v_{\theta}D\left(v_{\theta} - v_{w}\right).
\end{equation}
It should also be noted that whenever $(\theta^{\star} , w^{\star}) \in \mathcal{F}_{value}$, we have that (using \eqref{GradLoss})
\begin{equation} \label{GradCond}
    \nabla_{w} H\left(\theta^{\star} , w^{\star}\right) = 2\nabla v_{w^{\star}}D\left(v_{w^{\star}} - \mathcal{T} v_{\theta^{\star}}\right) = {\bf 0},
\end{equation}
where the last equality follows from the constraint that $v_{w^{\star}} = \mathcal{T} v_{\theta^{\star}}$.

Now we begin our proof with the following useful lemmas.
\begin{lemma} \label{lem:interm1_strong}
Let $\left\{ (\theta^{t} , w^{t}) \right\}_{t \in \mathbb{N}}$ be a sequence generated by the Lookahead Replicate algorithm. Then, for all $t \in \mathbb{N}$ and $0 \leq k \leq \KL-1$, we have 
\begin{equation*}
    H\left(\theta^{t} , w^{\star}\right) - H\left(\theta^{t} , w^{t , k}\right) \leq \frac{F_{\theta}^{2}}{2 F_{w}}\left\|\theta^{t} - \theta^{\star}\right\|^{2} - \left(\frac{1}{\alpha} - \frac{L}{2}\right)\left\|w^{t , k+1} - w^{t , k}\right\|^2.
\end{equation*}
\end{lemma}
\begin{proof}
Let $t \in \mathbb{N}$. From the $F_w$-strong convexity of $w \mapsto H(\theta^t, w)$, we obtain that
\begin{equation*}
H(\theta^t, w^{t,k+1}) \geq H(\theta^t, w^\star) + \langle \nabla_w H(\theta^t, w^\star), w^{t,k+1} - w^\star \rangle + \frac{F_w}{2} \|w_{t,k+1} - w^\star\|^2,
\end{equation*}
which implies
\begin{equation} \label{eqn:7}
H(\theta^t, w^\star) - H(\theta^t, w^{t,k+1}) \leq \langle \nabla_w H(\theta^t, w^\star), w^\star - w^{t,k+1} \rangle - \frac{F_w}{2} \|w^{t,k+1} - w^\star\|^2. 
\end{equation}
Moreover, following \eqref{GradCond} we can use the condition that $\nabla_w H(\theta^\star, w^\star) = 0$, to obtain
\begin{align} \label{eqn:8}
\langle \nabla_w H(\theta^t, w^\star), w^\star - w^{t,k+1} \rangle & = \langle \nabla_w H(\theta^t, w^\star) - \nabla_w H(\theta^\star, w^\star), w^\star - w^{t,k+1} \rangle \nonumber \\
& \leq \frac{1}{2F_w} \|\nabla_w H(\theta^t, w^\star) - \nabla_w H(\theta^\star, w^\star)\|^2 + \frac{F_w}{2} \|w^{t,k+1} - w^\star\|^2 \nonumber \\
& \leq \frac{F^2_{\theta}}{2F_w}  \|\theta^t - \theta^\star\|^2 + \frac{F_w}{2} \|w^{t,k+1} - w^\star\|^2.
\end{align}
Here, the first inequality follows from the fact that for any two vectors $a$ and $b$ we have $\langle a , b \rangle \leq (1/2d)\|a\|^{2} + (d/2)\|b\|^{2}$ for any $d > 0$. In this case, we chose $d = F_{w}$. Also the last inequality follows from the $F_{\theta}$-Lipschitz property of $\nabla_{w}H$, which is our Assumption \ref{assum1}. Now, by combining \eqref{eqn:7} with \eqref{eqn:8}, we obtain that
\begin{equation} \label{eqn:4}
H(\theta^t, w^\star) - H(\theta^t, w^{t,k+1}) \leq \frac{F^2_{\theta}}{2F_w}  \|\theta^t - \theta^\star\|^2 
\end{equation}
From the Lipschitz assumption we can write, due to the Descent Lemma~\citep{beck2017first} applied to the function $w \rightarrow H(\theta^{t} , w)$, that
\begin{equation*}
H(\theta^t, w^{t,k+1}) - H(\theta^t, w^{t,k}) \leq \langle \nabla_w H(\theta^t, w^{t,k}), w^{t,k+1} - w^{t,k} \rangle + \frac{L}{2} \|w^{t,k+1} - w^{t,k}\|^2.
\end{equation*}
Now, notice that using the main step of the Lookahead operation (see step 3) we have $w^{t,k+1} = w^{t,k} - \alpha \nabla_w H(\theta^t, w^{t,k})$, and so we can write:
\begin{align}
H(\theta^t, w^{t,k+1}) - H(\theta^t, w^{t,k}) & \leq \frac{1}{\alpha} \langle w^{t,k} - w^{t,k+1}, w^{t,k+1} - w^{t,k} \rangle + \frac{L}{2} \|w^{t,k+1} - w^{t,k}\|^2 \nonumber \\
 & = -\left( \frac{1}{\alpha} - \frac{L}{2} \right) \|w^{t,k+1} - w^{t,k}\|^2. \label{eqn:5}    
\end{align}
Adding both sides of \eqref{eqn:4} with \eqref{eqn:5} yields
\begin{equation*}
H(\theta^t, w^\star) - H(\theta^t, w^{t,k}) \leq \frac{F^2_{\theta}}{2F_w} \|\theta^t - \theta^\star\|^2 - \left( \frac{1}{\alpha} - \frac{L}{2} \right) \|w^{t,k+1} - w^{t,k}\|^2,
\end{equation*}
which proves the desired result. 
\end{proof}
Based on this result, we can now state and prove the following result.
\begin{lemma}\label{lem:interm2_strong}
Let $\left\{ (\theta^{t} , w^{t}) \right\}_{t \in \mathbb{N}}$ be a sequence generated by the Lookahead Replicate algorithm. Choose $\alpha=1/L$. Then, for all $t \in \mathbb{N}$ and $0 \leq k \leq \KL-1$, we have 
\begin{equation*}
\left\|w^{t, k+1}-w^{\star}\right\|^2 \leq\left(1-\frac{F_w}{L}\right)\left\|w^{t, k}-w^{\star}\right\|^2+\frac{F_\theta^2}{LF_w}\left\|\theta^t-\theta^{\star}\right\|^2.
\end{equation*}
\end{lemma}
\begin{proof}
Let $t \in \mathbb{N}$. From the main step of the Lookahead operation (see step 3), that is, $w^{t,k+1} = w^{t,k} - \alpha \nabla_w H(\theta^t, w^{t,k})$, we obtain for any $0 \leq k \leq \KL-1$ that
\begin{align} \label{eqn:inter-5.2}
\left\|w^{t, k+1}-w^{\star}\right\|^2 & =\left\|\left(w^{t, k}-w^\star\right)-\alpha \nabla_w H\left(\theta^t, w^{t, k}\right)\right\|^2 \nonumber \\
& =\left\|w^{t, k}-w^{\star}\right\|^2+2 \alpha\left\langle\nabla_w H\left(\theta^t, w^{t, k}\right), w^{\star}-w^{t, k}\right\rangle+\left\|\alpha\nabla_w H\left(\theta^t, w^{t, k}\right)\right\|^2 \nonumber \\
& =\left\|w^{t, k}-w^{\star}\right\|^2+2 \alpha\left\langle\nabla_w H\left(\theta^t, w^{t, k}\right), w^{\star}-w^{t, k}\right\rangle+\left\|w^{t, k+1}-w^{t, k}\right\|^2. 
\end{align}
Using the $F_{w}$-strong convexity of $w \rightarrow H(\theta^{t} , w)$, we have
\begin{equation}\label{eqn:inter-5.3}
H(\theta^t, w^\star) \geq H(\theta^t, w^{t,k}) + \langle \nabla_w H(\theta^t, w^{t,k}), w^\star - w^{t,k} \rangle + \frac{F_w}{2} \|w^{t,k} - w^\star\|^2.   
\end{equation}
Combining \eqref{eqn:inter-5.2} and \eqref{eqn:inter-5.3}, we get
\begin{align*}
\left\|w^{t, k+1}-w^{\star}\right\|^2 & \leq\left\|w^{t, k}-w^{\star}\right\|^2+2 \alpha\left(H\left(\theta^t, w^{\star}\right)-H\left(\theta^t, w^{t, k}\right)-\frac{F_w}{2}\left\|w^{t, k}-w^{\star}\right\|^2\right)  +\left\|w^{t, k+1}-w^{t, k}\right\|^2 \\
& =\left(1-\alpha F_w\right)\left\|w^{t, k}-w^{\star}\right\|^2+2 \alpha\left(H\left(\theta^t, w^{\star}\right)-H\left(\theta^t, w^{t, k}\right)\right)+\left\|w^{t, k+1}-w^{t, k}\right\|^2 .
\end{align*}
Hence, from Lemma \ref{lem:interm1_strong}, we obtain
\begin{align*}
\left\|w^{t, k+1}-w^{\star}\right\|^2 & \leq\left(1-\alpha F_w\right)\left\|w^{t, k}-w^{\star}\right\|^2+2 \alpha\left(\frac{F_\theta^2}{2 F_w}\left\|\theta^t-\theta^{\star}\right\|^2-\left(\frac{1}{\alpha}-\frac{L}{2}\right)\left\|w^{t, k+1}-w^{t, k}\right\|^2\right) \\
& +\left\|w^{t, k+1}-w^{t, k}\right\|^2 \\
& =\left(1-\alpha F_w\right)\left\|w^{t, k}-w^{\star}\right\|^2+\frac{\alpha F_\theta^2}{F_w}\left\|\theta^t-\theta^{\star}\right\|^2-(1-\alpha L)\left\|w^{t, k+1}-w^{t, k}\right\|^2.
\end{align*}
Moreover, by choosing the step-size as $\alpha = 1/L$ we obtain that
\begin{align*}
\|w^{t,k+1} - w^\star\|^2 \leq &\left( 1 - \frac{F_w}{L} \right) \|w^{t,k} - w^\star\|^2 + \frac{F^2_{\theta}}{LF_w} \|\theta^t - \theta^\star\|^2.
\end{align*}
This concludes the proof of this result.
\end{proof}

\begin{lemma}\label{lem:interm3_strong}
Let $\left\{ (\theta^{t} , w^{t}) \right\}_{t \in \mathbb{N}}$ be a sequence generated by the Lookahead-Replicate algorithm. Choose $\alpha=1/L$. Then, for all $t \in \mathbb{N}$, we have 
\begin{equation} \label{Le1}
\left\|w^{t+1}-w^{\star}\right\|^2\leq a\left\|w^{t}-w^{\star}\right\|^2+\eta^2 \left(1-a\right)\left\|\theta^t-\theta^{\star}\right\|^2,
\end{equation}
where $a :=(\max(1-\kappa, 0))^{K_L}$ with $\kappa := {L}^{-1}F_w$ and $\eta := F_w^{-1} F_\theta$.
\end{lemma}
\begin{proof}
Recall that from Lemma \ref{lem:interm2_strong} we have that (recall that $\alpha = 1/L$), 
\begin{equation}
    \begin{aligned}
        \left\|w^{t, k+1}-w^{\star}\right\|^2 & \le \left(1-\frac{F_w}{L}\right)\left\| w^{t, k}-w^{\star} \right\|^2 + \frac{F_\theta^2}{L F_w} \left\| \theta^{t}-\theta^{\star} \right\|^2 \\
        & = \left(1-\kappa\right)\left\| w^{t, k}-w^{\star} \right\|^2 + \kappa\eta^2 \left\| \theta^{t}-\theta^{\star} \right\|^2 \\
        & \le \left(1-\kappa\right)_{+}\left\| w^{t, k}-w^{\star} \right\|^2 + \kappa\eta^2 \left\| \theta^{t}-\theta^{\star} \right\|^2,
    \end{aligned}
\end{equation}
where $(\cdot)_{+}$ means $\max(0, \cdot)$, i.e. lower clipping the value by zero.

Then for any $t \in \mathbb{N}$, considering the fact that $w^{t + 1} = w^{t , K_L}$, we have 
\begin{equation*}
\begin{aligned}
 \left\|w^{t+1}-w^{\star}\right\|^2 & =\left\|w^{t, K_L}-w^{\star}\right\|^2 \\
 & \leq(1-\kappa)_{+}\left\|w^{t, K_L-1}-w^{\star}\right\|^2+\eta^2 \kappa\left\|\theta^t-\theta^{\star}\right\|^2 \\
& \leq(1-\kappa)_{+}\left[(1-\kappa)\left\|w^{t, K_L-2}-w^{\star}\right\|^2+\eta^2 \kappa\left\|\theta^t-\theta^{\star}\right\|^2\right]+\eta^2 \kappa\left\|\theta^t-\theta^{\star}\right\|^2 \\
& =(1-\kappa)^2_{+}\left\|w^{t, K_L-2}-w^{\star}\right\|^2+\eta^2 \kappa(1+(1-\kappa)_{+})\left\|\theta^t-\theta^{\star}\right\|^2 \\
& \leq \cdots \\
& \leq(1-\kappa)^{K_L}_{+}\left\|w^{t, 0}-w^{\star}\right\|^2+\eta^2 \kappa \sum_{k=0}^{K_L-1}(1-\kappa)^k_{+}\left\|\theta^t-\theta^{\star}\right\|^2 \\
& =(1-\kappa)^{K_L}_{+}\left\|w^{t}-w^{\star}\right\|^2+\eta^2 \kappa \sum_{k=0}^{K_L-1}(1-\kappa)^k_{+}\left\|\theta^t-\theta^{\star}\right\|^2\\
&=(1-\kappa)^{K_L}_{+}\left\|w^{t}-w^{\star}\right\|^2+\eta^2 \left(1-(1-\kappa)^{K_L}_{+}\right)\left\|\theta^t-\theta^{\star}\right\|^2.
\end{aligned}
\end{equation*}
The last step above comes from the classical geometric series formula, as we can write:
\begin{equation*}
    \eta^2 \kappa \sum_{k = 0}^{\KL - 1} \left(1 - \kappa\right)_{+}^{k} = \eta^2 \kappa \frac{1 - \left(1 - \kappa\right)_{+}^{\KL}}{1 - \left(1 - \kappa\right)_{+}} = \eta^{2}\left(1 - \left(1 - \kappa\right)_{+}^{\KL}\right),
\end{equation*}
which completes the desired result.\end{proof}

\begin{lemma}\label{lem:interm3_strong-M}
Let $\left\{ (\theta^{t} , w^{t}) \right\}_{t \in \mathbb{N}}$ be a sequence of parameters generated by the Lookahead-Replicate algorithm with $K_R=1$. Set learning rates $\alpha = \frac{1}{L}$ and $\beta_{0:K_R-1} = \frac{1}{\kappa_1^2}\cdot\frac{A-1}{B + \sqrt{B^2 - 8A^2 }}$, where $A = \eta^2 (1-a)$, $B = (\kappa_1^2)^{-1}F_{w} - A $, $\eta = F_w^{-1}F_\theta$ and $a = (\max(1-L^{-1}F_w, 0))^{K_L}$. Assume $F_w>\max\{F_\theta, 7\kappa_1^2\}$. Then,
\begin{align*}
    \norm{(\theta^{t+1} , w^{t+1}) - (\theta^\star , w^{\star})} \leq \sigma\norm{(\theta^{t} , w^{t}) - (\theta^\star , w^{\star})},
\end{align*}
for some constant $\sigma < 1$. In particular, the pair $(\theta^{\star} , w^{\star})\in \mathcal{F}_{value}$.
\end{lemma}
\begin{proof}
According to Algorithm Replicate($w,\theta,\beta_{0:K_R-1},K_R$) and the assumption that $K_R=1$, we have that $\theta^{t+1} = \theta^{t} - 2\beta\cdot \nabla v_{\theta^{t}} D (v_{\theta^{t}}-v_{w^{t+1}})$. Using this relation yields the following
\begin{align} \label{eqn:strong_main}
\norm{\theta^{t+1}-\theta^\star}^2 & = \norm{\theta^{t}-\theta^\star - 2\beta\cdot \nabla v_{\theta^{t}} D (v_{\theta^{t}}-v_{w^{t+1}})}^2 \nonumber\\
& =\norm{\theta^{t}-\theta^\star}^2-4\beta (\theta^{t}-\theta^\star)^\top \nabla v_{\theta^{t}} D (v_{\theta^{t}}-v_{w^{t+1}})+4\beta^2\norm{\nabla v_{\theta^{t}} D (v_{\theta^{t}}-v_{w^{t+1}})}^2 \nonumber \\
& =\norm{\theta^{t}-\theta^\star}^2-4\beta (\theta^{t}-\theta^\star)^\top\nabla v_{\theta^{t}} D (v_{\theta^{t}}-v_{\theta^{\star}})-4\beta (\theta^{t}-\theta^\star)^\top\nabla v_{\theta^{t}} D (v_{w^\star}-v_{w^{t+1}}) \nonumber \\
& +4\beta^2\norm{\nabla v_{\theta^{t}} D (v_{\theta^{t}}-v_{w^{t+1}})}^2,
\end{align}
where the last equality follows from the fact that $v_{\theta^{\star}} = v_{w^{\star}}$ since $(\theta^{\star} , w^{\star}) \in \mathcal{F}_{value}$. Next, we first bound the three terms on the right hand side separately, and plug them back later.

Now, using the $F_{w}$-strong convexity of the function $w \rightarrow H(\theta^{\star} , w)$, yields
\begin{align} \label{eqn:s_proof1}
 F_w\left\|\theta^{t}-\theta^\star\right\|^2 & \leq \left(\nabla_w H\left(\theta^\star, \theta^{t}\right)-\nabla_w H\left(\theta^\star, \theta^\star\right)\right)^{\top}\left(\theta^{t}-\theta^\star\right) \nonumber\\
 & = \nabla_w H\left(\theta^\star, \theta^{t}\right)^{\top}\left(\theta^{t}-\theta^\star\right) \nonumber \\
 & = 2(\theta^{t}-\theta^\star)^\top\nabla v_{\theta^{t}} D (v_{\theta^{t}}-v_{\theta^\star}),
\end{align}
where the first equality follows from the fact that $\nabla_w H\left(\theta^\star, \theta^\star\right)=2\nabla v_{\theta^{\star}} D (v_{\theta^{\star}}-\mathcal{T}v_{\theta^\star})=2\nabla v_{\theta^{\star}} D (v_{w^{\star}}-\mathcal{T}v_{\theta^\star})={\bf 0}$, since $(\theta^{\star} , w^{\star}) \in \mathcal{F}_{value}$ and $v_{w^{\star}} = \mathcal{T}v_{\theta^{\star}}$. The second equal sign comes from
$
\nabla_w H\left(\theta^\star, \theta^{t}\right)=2\nabla v_{\theta^{t}}\cdot D(v_{\theta^{t}}-\mathcal{T}v_{\theta^\star}).
$
Next, recall that from Assumption 4.1 we have $\norm{v_{\theta_1}-v_{\theta_2}}\leq \kappa_1 \norm{\theta_1-\theta_2}$ for all $\theta_1,\theta_2\in\Theta$, and the diagonal of $D$ is a probability distribution, hence
now \begin{equation}\label{eqn:s_proof2}
 4\beta(\theta^{t}-\theta^\star)^\top\nabla v_{\theta^{t}}\cdot D(v_{w^\star}-v_{w^{t+1}}) \leq 4\beta\norm{\theta^{t}-\theta^\star}\cdot \kappa_1^2 \norm{w^{t+1}-w^\star} \leq  2\beta\kappa_1^2(\norm{\theta^{t}-\theta^\star}^2+  \norm{w^{t+1}-w^\star}^2),
\end{equation}
where the second inequality uses the simple fact that $2ab\leq a^2+b^2$ for any $a,b \in \mathbb{R}$. Furthermore, 
\begin{align} \label{eqn:s_proof3}
4\beta^2\norm{\nabla v_{\theta^{t}}D(v_{\theta^{t}}-v_{w^{t+1}})}^2 & \leq 8\beta^2(\norm{\nabla v_{\theta^{t}}D(v_{\theta^{t}}-v_{\theta^\star})}^2+\norm{\nabla v_{\theta^{t}}D(v_{w^\star}-v_{w^{t+1}})}^2) \nonumber \\
& \leq 8\beta^2\kappa_1^4(\norm{\theta^{t}-\theta^{\star}}^2+\norm{w^{t+1}-w^\star}^2),
\end{align}
where the first inequality uses $(a+b)^2\leq 2a^2+2b^2$. Therefore, by integrating the bounds found in \eqref{eqn:s_proof1}, \eqref{eqn:s_proof2}, and \eqref{eqn:s_proof3} back into \eqref{eqn:strong_main}, we obtain 
\begin{align}
\norm{\theta^{t+1}-\theta^\star}^2 & = \norm{\theta^{t}-\theta^\star}^2-4\beta (\theta^{t}-\theta^\star)^\top\nabla v_{\theta^{t}} D (v_{\theta^{t}}-v_{\theta^{\star}})-4\beta (\theta^{t}-\theta^\star)^\top\nabla v_{\theta^{t}}D(v_{w^\star}-v_{w^{t+1}}) \nonumber \\
& + 4\beta^2\norm{\nabla v_{\theta^{t}}D(v_{\theta^{t}}-v_{w^{t+1}})}^2 \nonumber \\
& \leq \norm{\theta^{t}-\theta^\star}^2-2\beta F_w \left\|\theta^{t}-\theta^\star\right\|^2+ (2\beta\kappa_1^2 + 8\beta^2\kappa_1^4)\left(\norm{\theta^{t}-\theta^\star}^2+ \norm{w^\star-w^{t+1}}^2\right) \nonumber \\
& =\left(1-2\beta F_w+2\beta\kappa_1^2+8\beta^2\kappa_1^4\right)\left\|\theta^t-\theta^{\star}\right\|^2+\left(8\kappa_1^4\beta^2+2\kappa_1^2\beta\right)\left\|w^{t+1}-w^{\star}\right\|^2 \nonumber \\
& \leq \left(1-2\beta F_w+2\beta\kappa_1^2+8\beta^2\kappa_1^4\right)\left\|\theta^t-\theta^{\star}\right\|^2  +\left(8\kappa_1^4\beta^2+2\kappa_1^2\beta\right)\left(a\left\|w^{t}-w^{\star}\right\|^2+\eta^2 \left(1-a\right)\left\|\theta^t-\theta^{\star}\right\|^2\right) \nonumber \\
& =\left(1-2\beta F_w+2\beta\kappa_1^2(1 +4\beta\kappa_1^2)[1+\eta^2 (1-a)]\right)\left\|\theta^t-\theta^{\star}\right\|^2  +2a\kappa_1^2\beta\left(1 + 4\beta\kappa_1^2\right)\left\|w^{t}-w^{\star}\right\|^2,  \label{eqn:key1}
\end{align}
where the second inequality follows from Lemma \ref{lem:interm3_strong}. Combining \eqref{Le1} with \eqref{eqn:key1}, we obtain that
\begin{equation}
    \norm{\theta^{t+1}-\theta^\star}^2+\norm{w^{t+1}-w^\star}^2\leq E \norm{\theta^{t}-\theta^\star}^2+ G \norm{w^{t}-w^\star}^2
\end{equation}
where
\begin{equation}\label{eq:defineE}
E:=1-2\beta F_w+2\beta\kappa_1^2[1+\eta^2 (1-a)]+8\beta^2\kappa_1^4[1+\eta^2 (1-a)]+\eta^2(1-a)
\end{equation}
and
\begin{equation}
G:=8a\kappa_1^4\beta^2+2a\kappa_1^2\beta + a
\end{equation}
We will show that, if $F_w \ge \max\{ 6\kappa_1^2, F_\theta, L\} $ and $\beta = \frac{\eta^2 (1-a) }{ F_w-\kappa_1^2-\kappa_1^2\eta^2(1-a) + \sqrt{\left(F_w-\kappa_1^2-\kappa_1^2\eta^2(1-a)\right)^2 - 8\kappa_1^4(1+\eta^2(1-a))^2 }} $, we have that $0 < E, G < 1$ holds. Then we can complete the proof of contraction by letting $\sigma = \max(E, G) $.

Now we first prove $0 < E < 1$. For simplicity, we first define some constant:
\begin{align}
    & x := \beta \kappa_1^2 \label{eq:defx} \\
    & A := 1 + \eta^2(1-a)  \label{eq:defA} \\
    & B := \frac{F_w}{\kappa_1^2} - A = \frac{F_w}{\kappa_1^2} - 1 - \eta^2(1-a) \label{eq:defB}
\end{align}
With the definition of $A$ \eqref{eq:defA} and $B$  \eqref{eq:defB}, we can simplify our chosen value of $\beta$ as:
\begin{align*}
    \beta &= \frac{\eta^2 (1-a) }{ F_w-\kappa_1^2-\kappa_1^2\eta^2(1-a) + \sqrt{\left(F_w-\kappa_1^2-\kappa_1^2\eta^2(1-a)\right)^2 - 8\kappa_1^4(1+\eta^2(1-a))^2 }} \\
    &= \frac{1}{\kappa_1^2}\cdot\frac{A-1}{B + \sqrt{B^2 - 8A^2 }} 
\end{align*}
It is straightforward to check the following properties about $\eta, a, A, B$, which will be useful in the later parts of the proof.
\begin{align}
    & 0 <  \frac{F_\theta}{F_w} = \eta = \frac{F_\theta}{F_w} <  1 \\
    & 0 \le \max(1-\kappa, 0)^{K_L} = a = \max(1-\kappa, 0)^{K_L} < 1 \\
    & 1 < A < 2 \\
    & B \ge 7 - A > 6 > 3A \\
    & x = \beta \kappa_1^2 = \frac{A-1}{B + \sqrt{B^2 - 8A^2 }} < \frac{2 - 1}{7 + A} < \frac{1}{8} \label{eq:defx2}
\end{align}
As we can see $B^2 - 8A^2 > A^2 > 0$, our set value to $\beta$ is real and positive.

By the definition of $A$~\eqref{eq:defA}, $B$ \eqref{eq:defB}, and $x$ \eqref{eq:defx2}, we can rewrite $E$ \eqref{eq:defineE} as:
\begin{align*}
    E & = 8Ax^2 - 2Bx + A = 8A \left( x - \frac{B}{8A} \right)^2 + A - \frac{B^2}{8A} \\
    & = \frac{1}{8A} \left[ \left( B - 8A x \right)^2 - \left( B^2 - 8A^2 \right) \right]
\end{align*}
Notice that $B - 8A x > 0$. In order to upper and lower bound the whole term $E$, we only need to lower and upper bound $8Ax$. Now we first upper bound $8Ax$. Notice that $A < 2$,
\begin{align*}
    8Ax &= \frac{8A(A-1)}{B + \sqrt{B^2 - 8A^2}} \\
    &< \frac{8A^2}{B + \sqrt{B^2 - 8A^2}} \\
    &= \frac{(B + \sqrt{B^2 - 8A^2})(B - \sqrt{B^2 - 8A^2})}{B + \sqrt{B^2 - 8A^2}} \\
    &= B - \sqrt{B^2 - 8A^2}
\end{align*}
The second to last equation comes from $(B - \sqrt{B^2 - 8A^2})(B + \sqrt{B^2 - 8A^2}) = 8A^2$ at the same time. 

Now we plug this back to the expression of $E$. Notice that $B - 8A x > B - (B - \sqrt{B^2 - 8A^2}) > 0$, thus 
\begin{align*}
    \left( B - 8A x \right)^2 - \left( B^2 - 8A^2 \right) &> \left( B - \left(B - \sqrt{B^2 - 8A^2}\right) \right)^2 - \left( B^2 - 8A^2 \right) \\
    & = \left( \sqrt{B^2 - 8A^2} \right)^2 - \left( B^2 - 8A^2 \right) = 0
\end{align*}
Thus $E > 0$.

Now we try to lower bound $8Ax$,
\begin{align*}
    8Ax &= \frac{8A(A-1)}{B + \sqrt{B^2 - 8A^2}} \\ 
    &> \frac{8A(A-1)}{B + \sqrt{B^2 - 8A^2 + 8A}} \\
    &=\frac{(B + \sqrt{B^2 - 8A^2 + 8A})(B - \sqrt{B^2 - 8A^2 + 8A})}{B + \sqrt{B^2 - 8A^2 + 8A}} \\
    &= B - \sqrt{B^2 - 8A^2 + 8A}
\end{align*}
The second to last equation comes from $(B + \sqrt{B^2 - 8A^2 + 8A})(B - \sqrt{B^2 - 8A^2 + 8A}) = 8A^2 - 8A = 8A(A-1)$. Now we plug this back to the expression of $E$. Because $B - (B - \sqrt{B^2 - 8A^2 + 8A}) > B - 8A x > 0$, we have that 
\begin{align*}
    E - 1 &= \frac{1}{8A} \left[ \left( B - 8A x \right)^2 - \left( B^2 - 8A^2 + 8A \right) \right] \\
    &< \frac{1}{8A} \left[ \left( B - (B - \sqrt{B^2 - 8A^2 + 8A}) \right)^2 - \left( B^2 - 8A^2 + 8A \right) \right] \\
    &= \frac{1}{8A} \left[ \left( \sqrt{B^2 - 8A^2 + 8A} \right)^2 - \left( B^2 - 8A^2 + 8A \right) \right] \\
    &= 0
\end{align*}
This finishes the proof of $0 < E < 1$.

Now we are going to prove that $0 < G <1$. By definition, it is straightforward to see $G > 0$. In order to show $G < 1$, we need to prove
\begin{align*}
    8ax^2+2ax < 1-a
\end{align*}
\begin{align*}
    8ax^2+2ax & = \frac{8a (A-1)^2 }{\left( B + \sqrt{B^2-8A^2} \right)^2} + \frac{2a (A-1) }{\left( B + \sqrt{B^2-8A^2} \right)} \\
    & < \frac{8a (A-1)^2 }{\left( 6 + \sqrt{9A^2-8A^2} \right)^2} + \frac{2a (A-1) }{\left( 6 + \sqrt{9A^2-8A^2} \right)} \\
    & = \frac{8a (A-1)^2 }{49} + \frac{2a (A-1) }{7} \\
    &= \frac{8a \eta^4 (1-a)^2 }{49} + \frac{2a \eta^2(1-a) }{7} \\
    &< \frac{8 (1-a) }{49} + \frac{2(1-a)}{7} \\
    &= \frac{22(1-a)}{49} <  1-a
\end{align*}
Thus we finish the proof of $0 < G < 1$. Taking that $\sigma = \max(E, G) < 1$ finishes the proof of theorem statement.

\end{proof}

\begin{theorem}[Restatement of Theorem \ref{thm:interm3_strong-M-temp} in the main paper]
\label{thm:interm3_strong-M}
Let $\left\{ (\theta^{t} , w^{t}) \right\}_{t \in \mathbb{N}}$ be a sequence of parameters generated by the Lookahead-Replicate algorithm. Set learning rates $\alpha = \frac{1}{L}$ and $\beta_0 = \frac{1}{\kappa_1^2}\cdot\frac{A-1}{B + \sqrt{B^2 - 8A^2 }}$, where $A = \eta^2 (1-a)$, $B = (\kappa_1^2)^{-1}F_{w} - A $, $\eta = F_w^{-1}F_\theta$ and $a = (\max(1-L^{-1}F_w, 0))^{K_L}$. For $1\leq k\leq K_R-1$, set $\beta_k=\frac{1}{\kappa_1^2}\frac{3J+\sqrt{J^2-32}}{32}$, where $J=2F_w(1-\zeta)/\kappa_1^2-2$ and $\zeta = \max\{a,\eta^2(1-a)\}<1$. Assume $F_w>\max\{F_\theta, 7\kappa_1^2, \frac{4\kappa_1^2}{1-\zeta}\}$. Then,
\begin{align*}
    \norm{(\theta^{t+1} , w^{t+1}) - (\theta^\star , w^{\star})} \leq \sigma\norm{(\theta^{t} , w^{t}) - (\theta^\star , w^{\star})},
\end{align*}
for some constant $\sigma < 1$. In particular, the pair $(\theta^{\star} , w^{\star})\in \mathcal{F}_{value}$.
\end{theorem}
\begin{proof}
According to Algorithm Replicate($w,\theta,\beta_{0:K-1},K$), we have that $\theta^{t,k+1} = \theta^{t,k} - 2\beta_k\cdot \nabla v_{\theta^{t,k}} D (v_{\theta^{t,k}}-v_{w^{t+1}})$. Using this relation yields the following
\begin{align} 
\norm{\theta^{t,k+1}-\theta^\star}^2 & = \norm{\theta^{t,k}-\theta^\star - 2\beta_k\cdot \nabla v_{\theta^{t,k}} D (v_{\theta^{t,k}}-v_{w^{t+1}})}^2 \nonumber\\
& =\norm{\theta^{t,k}-\theta^\star}^2-4\beta_k (\theta^{t,k}-\theta^\star)^\top \nabla v_{\theta^{t,k}} D (v_{\theta^{t,k}}-v_{w^{t+1}})+4\beta^2_k\norm{\nabla v_{\theta^{t,k}} D (v_{\theta^{t,k}}-v_{w^{t+1}})}^2 \nonumber \\
& =\norm{\theta^{t,k}-\theta^\star}^2-4\beta_k (\theta^{t,k}-\theta^\star)^\top\nabla v_{\theta^{t,k}} D (v_{\theta^{t,k}}-v_{\theta^{\star}})-4\beta_k (\theta^{t,k}-\theta^\star)^\top\nabla v_{\theta^{t,k}} D (v_{w^\star}-v_{w^{t+1}}) \nonumber \\
& +4\beta_k^2\norm{\nabla v_{\theta^{t,k}} D (v_{\theta^{t,k}}-v_{w^{t+1}})}^2,
\end{align}
where the last equality follows from the fact that $v_{\theta^{\star}} = v_{w^{\star}}$ since $(\theta^{\star} , w^{\star}) \in \mathcal{F}_{value}$. Next, we first bound the three terms on the right hand side separately, and plug them back later.

Now, using the $F_{w}$-strong convexity of the function $w \rightarrow H(\theta^{\star} , w)$, yields
\begin{align} 
 F_w\left\|\theta^{t,k}-\theta^\star\right\|^2 & \leq \left(\nabla_w H\left(\theta^\star, \theta^{t,k}\right)-\nabla_w H\left(\theta^\star, \theta^\star\right)\right)^{\top}\left(\theta^{t,k}-\theta^\star\right) \nonumber\\
 & = \nabla_w H\left(\theta^\star, \theta^{t,k}\right)^{\top}\left(\theta^{t,k}-\theta^\star\right) \nonumber \\
 & = 2(\theta^{t,k}-\theta^\star)^\top\nabla v_{\theta^{t,k}} D (v_{\theta^{t,k}}-v_{\theta^\star}),
\end{align}
where the first equality follows from the fact that $\nabla_w H\left(\theta^\star, \theta^\star\right)=2\nabla v_{\theta^{\star}} D (v_{\theta^{\star}}-\mathcal{T}v_{\theta^\star})=2\nabla v_{\theta^{\star}} D (v_{w^{\star}}-\mathcal{T}v_{\theta^\star})={\bf 0}$, since $(\theta^{\star} , w^{\star}) \in \mathcal{F}_{value}$ and $v_{w^{\star}} = \mathcal{T}v_{\theta^{\star}}$. The second equal sign comes from
$
\nabla_w H\left(\theta^\star, \theta^{t,k}\right)=2\nabla v_{\theta^{t,k}}\cdot D(v_{\theta^{t,k}}-\mathcal{T}v_{\theta^\star}).
$
Next, recall that from Assumption 4.1 we have $\norm{v_{\theta_1}-v_{\theta_2}}\leq \kappa_1 \norm{\theta_1-\theta_2}$ for all $\theta_1,\theta_2\in\Theta$, and the diagonal of $D$ is a probability distribution, hence
now \begin{equation}
\begin{aligned}
 & 4\beta_k(\theta^{t,k}-\theta^\star)^\top\nabla v_{\theta^{t,k}}\cdot D(v_{w^\star}-v_{w^{t+1}}) \leq  4\beta_k\norm{\theta^{t,k}-\theta^\star}\cdot \kappa_1^2 \norm{w^{t+1}-w^\star}\\ 
 \leq & 2\beta_k\kappa_1^2(\norm{\theta^{t,k}-\theta^\star}^2+  \norm{w^{t+1}-w^\star}^2)
 \end{aligned}
\end{equation}
where the second inequality uses the simple fact that $2ab\leq a^2+b^2$ for any $a,b \in \mathbb{R}$. Furthermore, 
\begin{align}
4\beta_k^2\norm{\nabla v_{\theta^{t,k}}D(v_{\theta^{t,k}}-v_{w^{t+1}})}^2 & \leq 8\beta_k^2(\norm{\nabla v_{\theta^{t,k}}D(v_{\theta^{t,k}}-v_{\theta^\star})}^2+\norm{\nabla v_{\theta^{t,k}}D(v_{w^\star}-v_{w^{t+1}})}^2) \nonumber \\
& \leq 8\beta_k^2\kappa_1^4(\norm{\theta^{t,k}-\theta^{\star}}^2+\norm{w^{t+1}-w^\star}^2),
\end{align}
where the first inequality uses $(a+b)^2\leq 2a^2+2b^2$. Next, we prove for any $k\in[K_R-1]$, it holds that 
\begin{equation}
    \norm{\theta^{t,k}-\theta^\star}^2+\norm{w^{t+1}-w^\star}^2\leq \sigma_k (\norm{\theta^{t,0}-\theta^\star}^2+ \norm{w^{t}-w^\star}^2)
\end{equation}
for some $\sigma_k<1$.

We prove this by induction. For $k=1$, by Lemma~\ref{lem:interm3_strong-M} and the choice of $\beta_0$ we have the conclusion holds true. Now assume that the above holds true for a iteration number $k$, then for $k+1$, by integrating the bounds found in \eqref{eqn:s_proof1}, \eqref{eqn:s_proof2}, and \eqref{eqn:s_proof3} back into \eqref{eqn:strong_main}, we obtain 
\begin{align}
\norm{\theta^{t,k+1}-\theta^\star}^2 & = \norm{\theta^{t,k}-\theta^\star}^2-4\beta_k (\theta^{t,k}-\theta^\star)^\top\nabla v_{\theta^{t,k}} D (v_{\theta^{t,k}}-v_{\theta^{\star}})-4\beta_k (\theta^{t}-\theta^\star)^\top\nabla v_{\theta^{t}}D(v_{w^\star}-v_{w^{t+1}}) \nonumber \\
& + 4\beta_k^2\norm{\nabla v_{\theta^{t,k}}D(v_{\theta^{t,k}}-v_{w^{t+1}})}^2 \nonumber \\
& \leq \norm{\theta^{t,k}-\theta^\star}^2-2\beta_k F_w \left\|\theta^{t,k}-\theta^\star\right\|^2+ (2\beta_k\kappa_1^2 + 8\beta_k^2\kappa_1^4)\left(\norm{\theta^{t,k}-\theta^\star}^2+ \norm{w^\star-w^{t+1}}^2\right) \nonumber \\
& \leq \norm{\theta^{t,k}-\theta^\star}^2-2\beta_k F_w \left\|\theta^{t,k}-\theta^\star\right\|^2+ (2\beta_k\kappa_1^2 + 8\beta_k^2\kappa_1^4)\left(\norm{\theta^{t,0}-\theta^\star}^2+ \norm{w^\star-w^{t}}^2\right) \nonumber 
\label{eqn:key1}
\end{align}
where the second inequality uses induction hypothesis. The above is equivalent to 
\begin{equation}
\begin{aligned}
 &\norm{\theta^{t,k+1}-\theta^\star}^2 
 \leq \norm{\theta^{t,k}-\theta^\star}^2-2\beta_k F_w \left\|\theta^{t,k}-\theta^\star\right\|^2+ (2\beta_k\kappa_1^2 + 8\beta_k^2\kappa_1^4)\left(\norm{\theta^{t,0}-\theta^\star}^2+ \norm{w^\star-w^{t}}^2\right) \\
 \Leftrightarrow & \norm{\theta^{t,k+1}-\theta^\star}^2+(1-2\beta_kF_w)\norm{w^\star-w^{t+1}}^2 
 \leq (1-2\beta_k F_w) [\left\|\theta^{t,k}-\theta^\star\right\|^2+\norm{w^\star-w^{t+1}}^2 ]\\
 +& (2\beta_k\kappa_1^2 + 8\beta_k^2\kappa_1^4)\left(\norm{\theta^{t,0}-\theta^\star}^2+ \norm{w^\star-w^{t}}^2\right)\\
  \Leftrightarrow & \norm{\theta^{t,k+1}-\theta^\star}^2+(1-2\beta_kF_w)\norm{w^\star-w^{t+1}}^2 
 \leq (1-2\beta_k F_w+2\beta_k\kappa_1^2 + 8\beta_k^2\kappa_1^4) \left(\norm{\theta^{t,0}-\theta^\star}^2+ \norm{w^\star-w^{t}}^2\right)\\
   \Leftrightarrow & \norm{\theta^{t,k+1}-\theta^\star}^2+\norm{w^\star-w^{t+1}}^2 
 \leq (1-2\beta_k F_w+2\beta_k\kappa_1^2 + 8\beta_k^2\kappa_1^4) \left(\norm{\theta^{t,0}-\theta^\star}^2+ \norm{w^\star-w^{t}}^2\right)\\
 +& 2\beta_kF_w \left(a\left\|w^{t}-w^{\star}\right\|^2+\eta^2 \left(1-a\right)\left\|\theta^t-\theta^{\star}\right\|^2\right),
\end{aligned}
\end{equation}
where the last inequality uses Lemma~\ref{lem:interm3_strong}. Now, let $\zeta = \max\{a,\eta^2(1-a)\}<1$, then above implies 
\[
\norm{\theta^{t,k+1}-\theta^\star}^2+\norm{w^\star-w^{t+1}}^2 
 \leq (1-2\beta_k F_w(1-\zeta)+2\beta_k\kappa_1^2 + 8\beta_k^2\kappa_1^4) \left(\norm{\theta^{t,0}-\theta^\star}^2+ \norm{w^\star-w^{t}}^2\right)
\]
Denote $x:=\beta_k\kappa_1^2$ and $J=2F_w(1-\zeta)/\kappa_1^2-2$, then above is equivalent to 
\begin{equation}\label{eqn:main_ineq}
\norm{\theta^{t,k+1}-\theta^\star}^2+\norm{w^\star-w^{t+1}}^2 
 \leq (1-Jx+8x^2) \left(\norm{\theta^{t,0}-\theta^\star}^2+ \norm{w^\star-w^{t}}^2\right)
\end{equation}
It remains to show 
\begin{equation}\label{eqn:main_eq}
0< 1-Jx+8x^2 <1.
\end{equation}
Since $F_w\geq\frac{4\kappa_1^2}{1-\zeta}$, this implies $J\geq 6$ which further implies $J^2-32\geq 0$. Next, since $x>0$, the above is equivalent to 
\begin{equation}\label{eqn:sss}
\frac{J+\sqrt{J^2-32}}{16}<x<\frac{J}{8}.
\end{equation}
By our choice 
\[
\beta_k=\frac{1}{\kappa_1^2}\frac{3J+\sqrt{J^2-32}}{32}\Leftrightarrow x = \frac{1}{2}\left(\frac{J+\sqrt{J^2-32}}{16}+\frac{J}{8}\right),
\]
then \eqref{eqn:sss} is satisfied. Finally, combine \eqref{eqn:main_ineq} and \eqref{eqn:main_eq}, by induction we receive the conclusion that 
\[
\norm{\theta^{t,k}-\theta^\star}^2+\norm{w^\star-w^{t+1}}^2 
 \leq \sigma_k \left(\norm{\theta^{t,0}-\theta^\star}^2+ \norm{w^\star-w^{t}}^2\right)
\]
where $\sigma_1$ is defined in Lemma~\ref{lem:interm3_strong-M} and $\sigma_k:= 1-Jx+8x^2<1$ for all $2\leq k\leq K_R$ with $x=\frac{3J+\sqrt{J^2-32}}{32}$. In particular, choose $k=K_R-1$ and $\sigma^2 = \sigma_{K_R}$, we have
\[
\norm{\theta^{t+1}-\theta^\star}^2+\norm{w^\star-w^{t+1}}^2 
 \leq \sigma^2 \left(\norm{\theta^{t}-\theta^\star}^2+ \norm{w^\star-w^{t}}^2\right)
\]
where we used $\theta^{t,K_R-1}=\theta^{t+1}$ and $\theta^{t,0}=\theta^t$. This is exactly our claim.

\end{proof}

\begin{remark}
We do emphasize the theoretical convergence analysis is for the population update (where we assume the access to $H$ that contains the population operator $\mathcal{T}$), where the in practice the algorithm work in the data-driven fashion. To incorporate this perspective, we could conduct the finite sample analysis that is similar to \cite{yang2018finite, wu2020finite} for actor-critic algorithms and \cite{yin2021towards,yin2022near} for offline RL. We leave these data-driven analysis for future work.  
\end{remark}

\begin{corollary}
As $t\rightarrow\infty$,
\[
\norm{V_{\theta^t}-V_{w^t}} \leq \sqrt{2}\kappa_1\sigma^{t-1}\sqrt{\norm{\theta^{0}-\theta^\star}^2+\norm{w^\star-w^{0}}^2}\rightarrow 0.
\]
In addition, we also have
\[
\norm{V_{w^t}-\mathcal{T}V_{\theta^t}} \leq \sqrt{2}\kappa_1\sigma^{t-1}\sqrt{\norm{\theta^{0}-\theta^\star}^2+\norm{w^\star-w^{0}}^2}\rightarrow 0.
\]
\end{corollary}

\begin{proof}
For the first part of the proof, notice $V_{\theta^\star}=V_{w^\star}$, we have 
\begin{align*}
\norm{V_{\theta^t}-V_{w^t}}^2=&\norm{V_{\theta^t}-V_{\theta^\star}+V_{w^\star}-V_{w^t}}^2\leq 2(\norm{V_{\theta^t}-V_{\theta^\star}}^2+\norm{V_{w^\star}-V_{w^t}}^2)\\
\leq&2\kappa_1^2(\norm{\theta^t-\theta^\star}^2+\norm{w^\star-w^t}^2)\leq 2\kappa_1^2\sigma^2(\norm{\theta^{t-1}-\theta^\star}^2+\norm{w^\star-w^{t-1}}^2)\\
\leq& 2\kappa_1^2(\sigma^2)^{t-1}(\norm{\theta^{0}-\theta^\star}^2+\norm{w^\star-w^{0}}^2)\rightarrow 0.
\end{align*}
where the last two inequalities use Theoerem~\ref{thm:interm3_strong-M-temp}. 
For the second part of the proof,
\begin{align*}
\norm{V_{w^t}-\mathcal{T}V_{\theta^t}}=&\norm{V_{w^t}-V_{w^\star}+\mathcal{T}V_{\theta^\star}-\mathcal{T}V_{\theta^t}}\leq 2(\norm{\mathcal{T}V_{\theta^t}-\mathcal{T}V_{\theta^\star}}^2+\norm{V_{w^\star}-V_{w^t}}^2)\\
\leq&2\kappa_1^2(\gamma^2\norm{\theta^t-\theta^\star}^2+\norm{w^\star-w^t}^2)\leq 2\kappa_1^2(\norm{\theta^t-\theta^\star}^2+\norm{w^\star-w^t}^2)\\
\leq& 2\kappa_1^2\sigma^2(\norm{\theta^{t-1}-\theta^\star}^2+\norm{w^\star-w^{t-1}}^2)\\
\leq& 2\kappa_1^2(\sigma^2)^{t-1}(\norm{\theta^{0}-\theta^\star}^2+\norm{w^\star-w^{0}}^2)\rightarrow 0.
\end{align*}
\end{proof}

\newpage

\section{Numerical simulation details in illustrative examples}

\subsection{Numerical simulation with a single parameter space}
\label{sec:appendix-example-ones}

In this section, we provide the details of the numerical simulation study, as an illustrative example of the algorithm convergence in Section~\ref{sec:fvalue}. We consider a Markov Chain $\mathcal{M}=(\mathcal{S},P,\gamma,r)$ with two states $\mathcal{S}=\{s_1,s_2\}$ and the transition matrix $P=\begin{bmatrix} 0.6&0.4\\0.2& 0.8 \end{bmatrix}$, e.g. $P(s_1|s_1)=0.6$. $\gamma=\frac{1}{2}$. Let $r=\begin{bmatrix} 1\\1 \end{bmatrix}$. The true value 
$
v^*=(I-\gamma P)^{-1}r=\begin{bmatrix} 2\\2 \end{bmatrix}.
$ 

Next, we parametrize the function class $v_\theta$ and $v_w$ differently. We set the linear function class $v_\theta(s)=\phi_\theta(s) ^\top \theta$ where $\phi_\theta(\cdot),\theta\in\RR^3$ with $\phi_\theta(s_1)=\begin{bmatrix} 1&2&1 \end{bmatrix}^\top$ and $\phi_\theta(s_2)=\begin{bmatrix} 1&1&2 \end{bmatrix}^\top$. The stationary distribution for $P$ is $\rho=(1/3,2/3)$. 

In Figure~\ref{fig:mini_example}, we set the initial point $\theta^0=[1.2,2,0.5],w=[0.1,2,0.5]$, $T=800,K_L=400,K_R=1$. The parameter $\theta$ converges to $\theta^T=[0.663, 0.445, 0.445]$, the parameter $w$ converges to $w^T=[-0.236, 0.745, 0.745]$ and both the value functions converge to $v_{w^T}=v_{\theta^T}=v^*=[2,2]$.

\subsection{Numerical simulation with different parameter spaces}
\label{sec:appendix-example-twos}

\begin{figure*}[ht]
  \centering
  \begin{subfigure}[b]{0.28\linewidth}
    \includegraphics[width=1\linewidth]{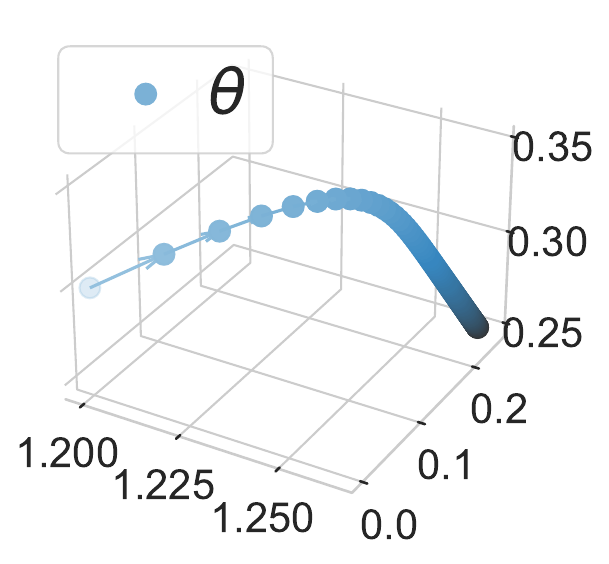}
    \caption{Path of $\theta$ in $\theta$-parameter space}
    \label{fig:example_theta_twos}
  \end{subfigure}
  \hfill 
  \begin{subfigure}[b]{0.35\linewidth}
    \includegraphics[width=1\linewidth]{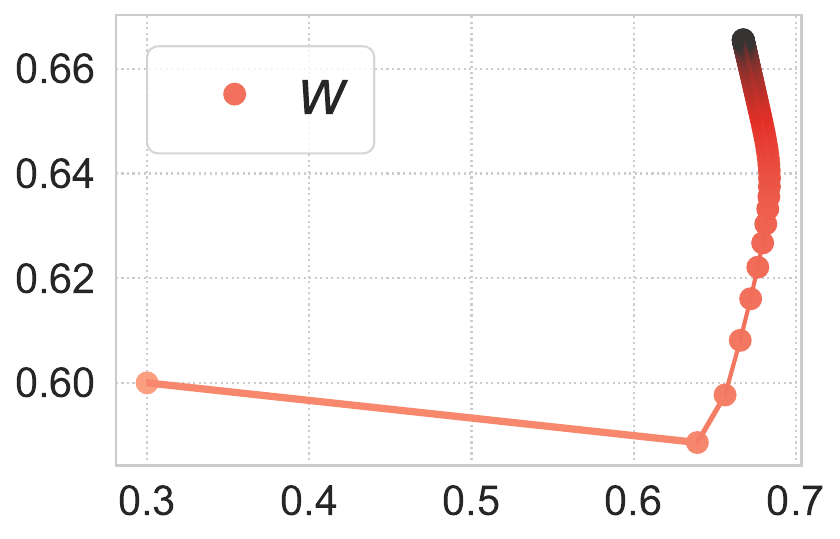}
    \caption{Path of $w$ in $w$-parameter space}
    \label{fig:example_w_twos}
  \end{subfigure}
  \hfill 
  \begin{subfigure}[b]{0.35\linewidth}
    \includegraphics[width=\linewidth]{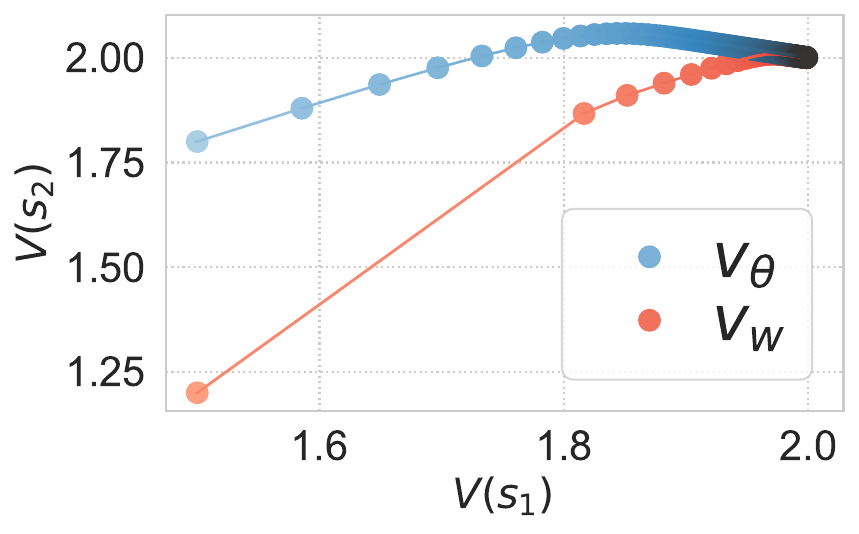}
    \caption{Path of $v_\theta$ and $v_w$ in value function space}
    \label{fig:example_value_twos}
  \end{subfigure}
  \caption{A simulation of the convergence for Algorithm~\ref{alg:main} with different parameterization. Parameter $\theta\in\mathbb{Rd}^3$, and parameter $w\in \mathbb{R}^2$.}
  \label{fig:example_twos} 
\end{figure*}
Although in main paper we only discuss the case of both $\theta$ and $w$ are in the same parameter space $\Theta$, our algorithm and analysis can naturally apply to the case when there are two different ways of parameterization for $v_\theta$ and $v_w$. Here we provide an numerical simulation of such setting.

The MRP as well as $v_\theta$ is the same as described in the single parameter space example, except that we use a different parameterization for $v_w$. We set the linear function class $v_w(s)=\phi_w(s)^\top w$ where $\phi_w(\cdot),w\in\RR^2$ with $\phi_w(s_1)=\begin{bmatrix} 1&2 \end{bmatrix}^\top$ and $\phi_w(s_2)=\begin{bmatrix} 2&1 \end{bmatrix}^\top$. In Figure~\ref{fig:example_twos}, we set the initial point $\theta^0=[1.2,0,0.3],w=[0.3,0.6]$, $T=800,K_L=400,K_R=1$. The parameter $\theta$ converges to $\theta^T=[1.264,0.245,0.245]$, the parameter $w$ converges to $w^T=[2/3,2/3]$ and both the value functions converge to $v_{w^T}=v_{\theta^T}=v^*=[2,2]$.



\newpage

\section{Full Results on Atari}

\begin{figure}[H]
\centering\captionsetup[subfigure]{justification=centering,skip=0pt}
\begin{subfigure}[t]{0.24\textwidth} 
\centering 
\includegraphics[width=\textwidth]{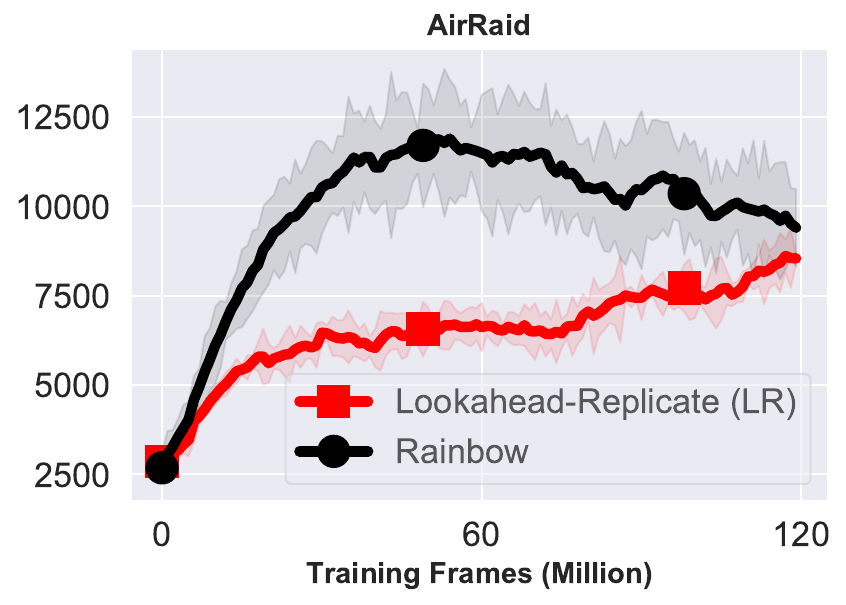} 
\end{subfigure}%
~ 
\begin{subfigure}[t]{ 0.24\textwidth} 
\centering 
\includegraphics[width=\textwidth]{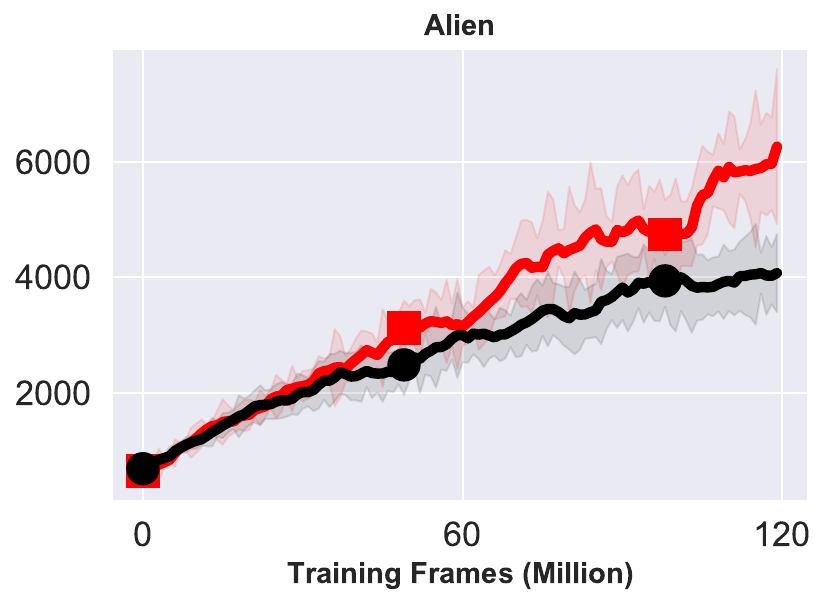} 
\end{subfigure}%
~ 
\begin{subfigure}[t]{ 0.24\textwidth} 
\centering 
\includegraphics[width=\textwidth]{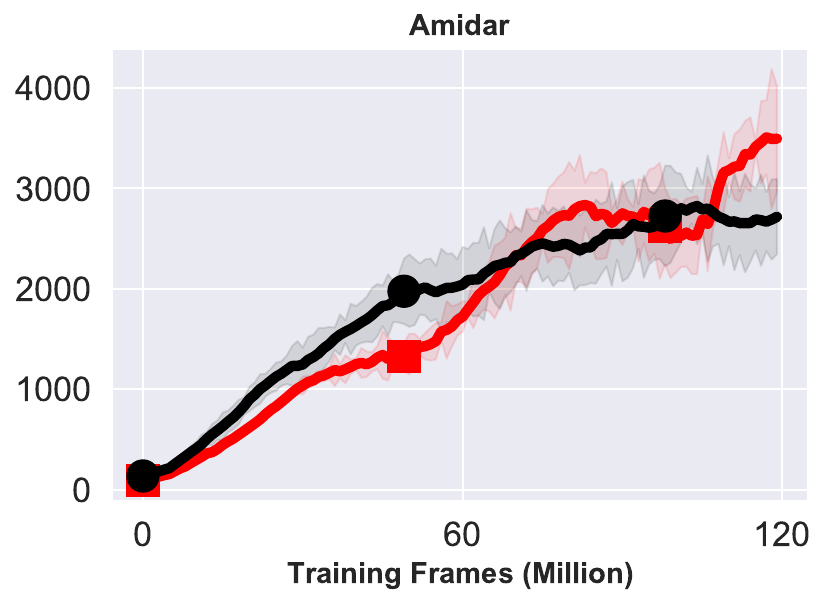}  
\end{subfigure}%
~ 
\begin{subfigure}[t]{ 0.24\textwidth} 
\centering 
\includegraphics[width=\textwidth]{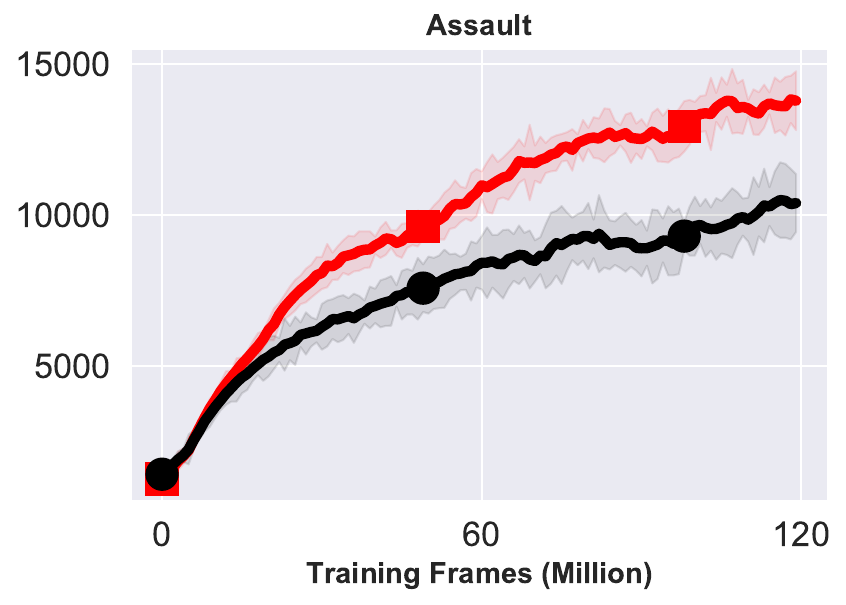} 
\end{subfigure}
\vspace{-.25cm}

\begin{subfigure}[t]{0.24\textwidth} 
\centering 
\includegraphics[width=\textwidth]{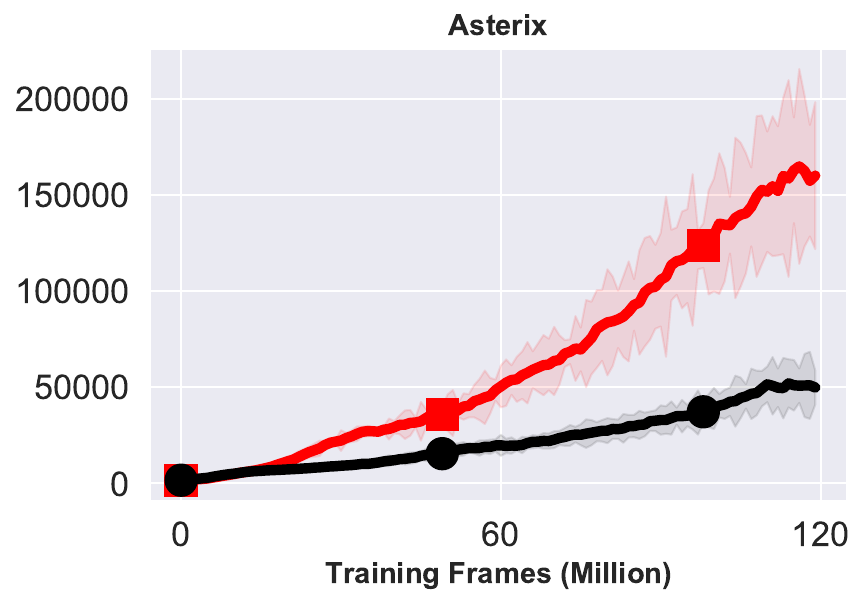}
\end{subfigure}%
~ 
\begin{subfigure}[t]{ 0.24\textwidth} 
\centering 
\includegraphics[width=\textwidth]{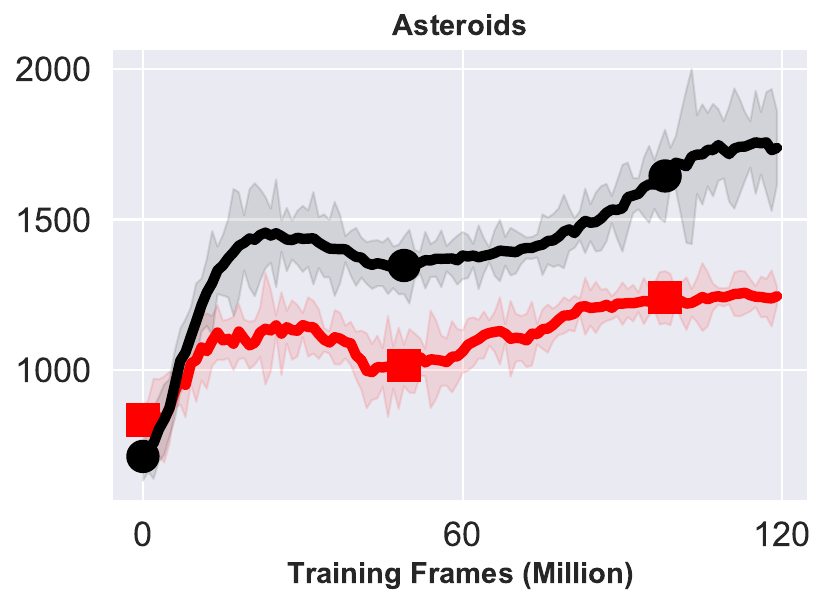} 
\end{subfigure}%
~ 
\begin{subfigure}[t]{ 0.24\textwidth} 
\centering 
\includegraphics[width=\textwidth]{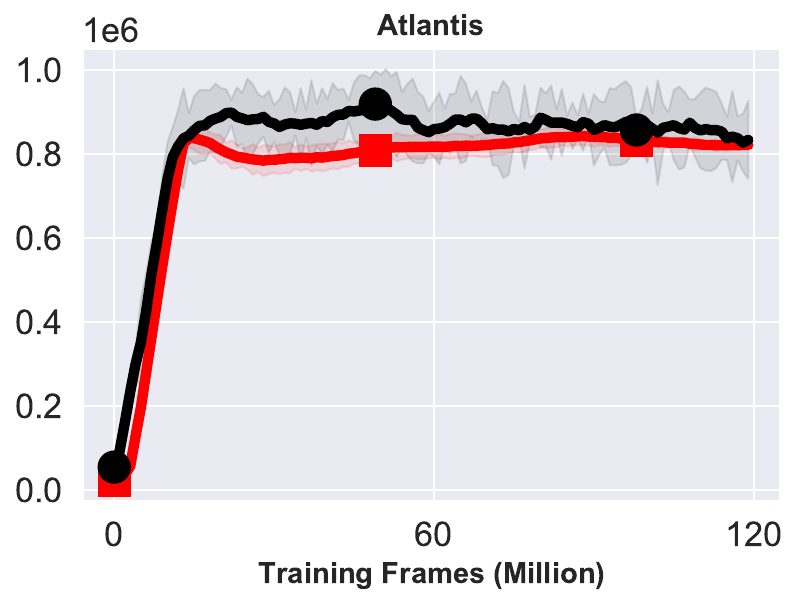} 
\end{subfigure}%
~ 
\begin{subfigure}[t]{ 0.24\textwidth} 
\centering 
\includegraphics[width=\textwidth]{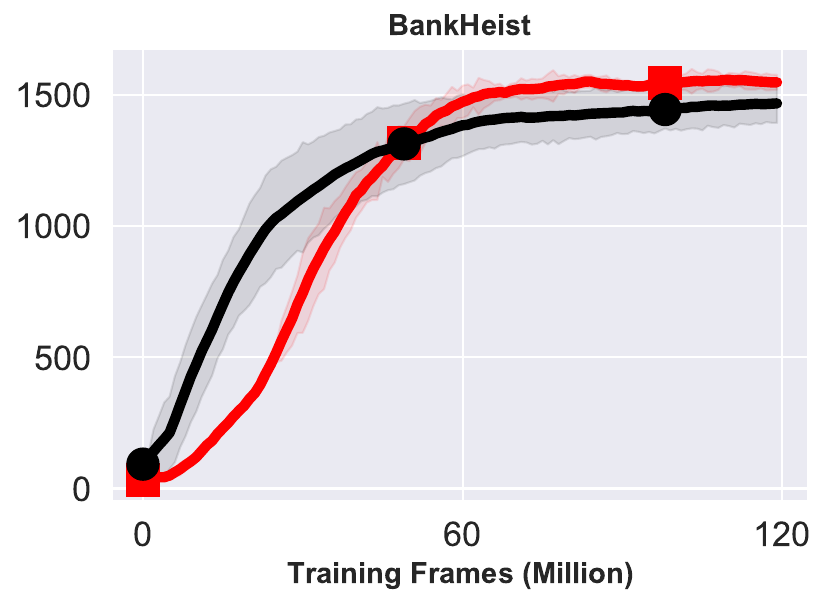} 
\end{subfigure}
\vspace{-.25cm}

\begin{subfigure}[t]{0.24\textwidth} 
\centering 
\includegraphics[width=\textwidth]{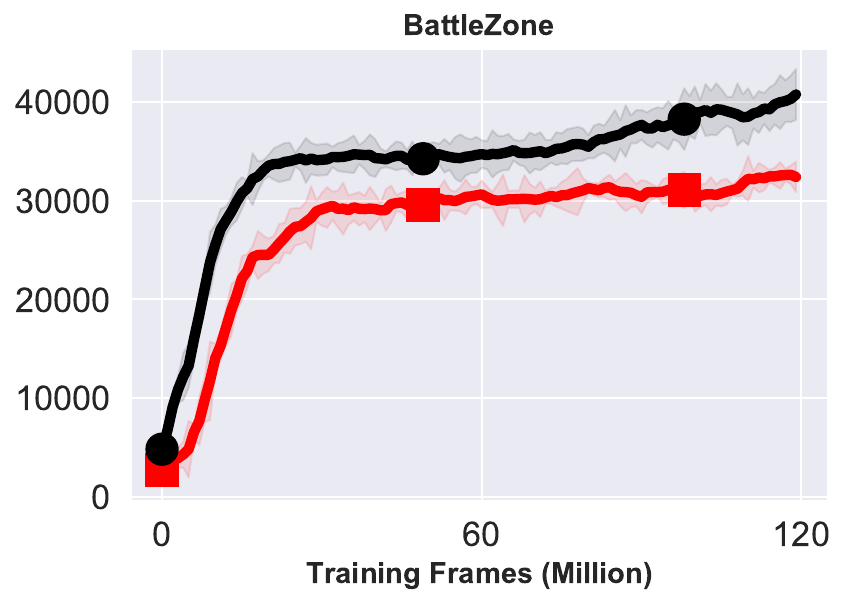}  
\end{subfigure}%
~ 
\begin{subfigure}[t]{ 0.24\textwidth} 
\centering 
\includegraphics[width=\textwidth]{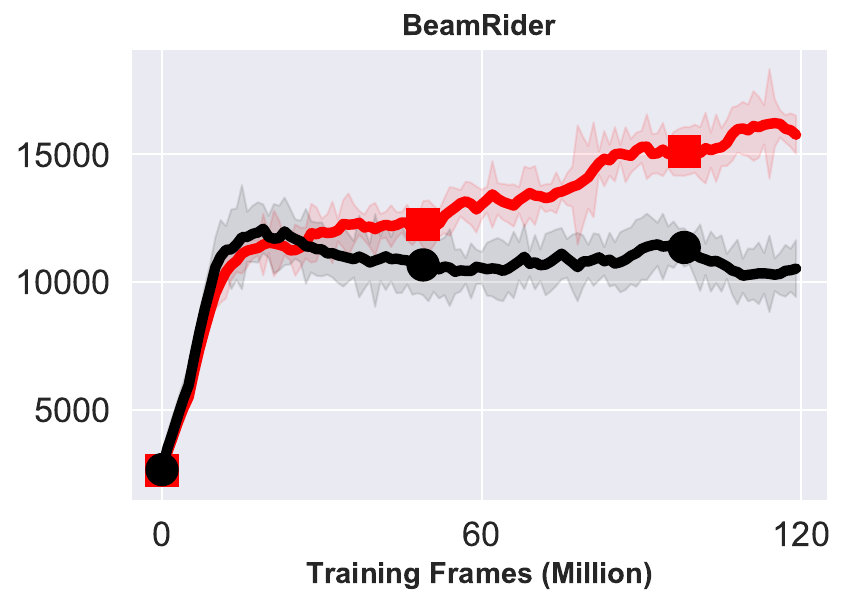}  
\end{subfigure}%
~ 
\begin{subfigure}[t]{ 0.24\textwidth} 
\centering 
\includegraphics[width=\textwidth]{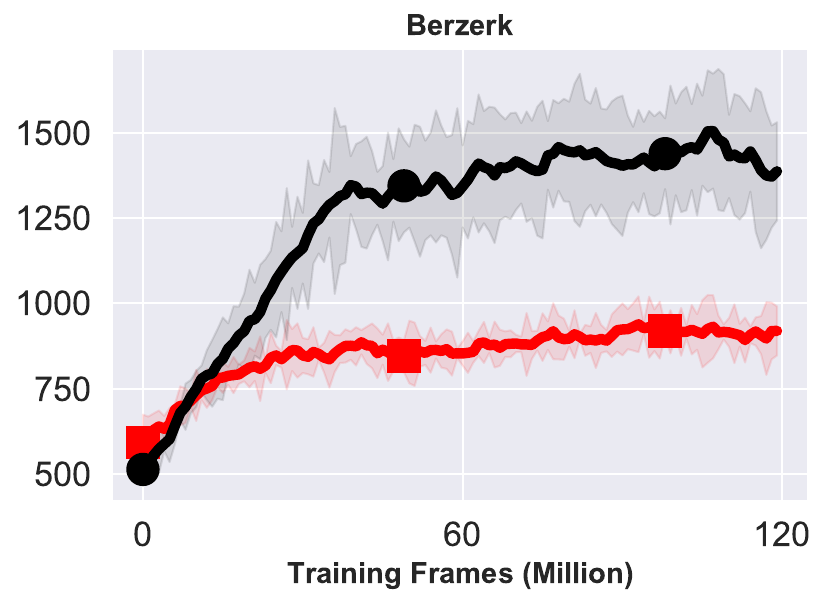} 
\end{subfigure}%
~ 
\begin{subfigure}[t]{ 0.24\textwidth} 
\centering 
\includegraphics[width=\textwidth]{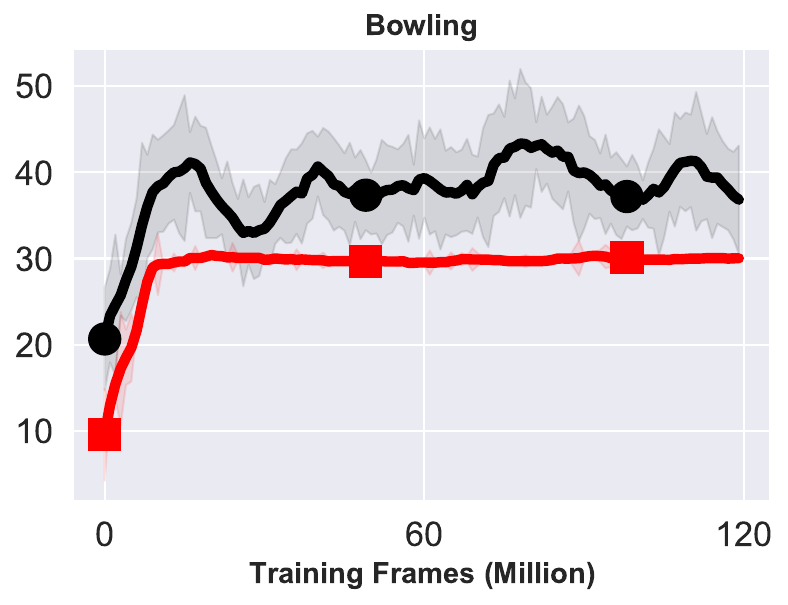} 
\end{subfigure}
\vspace{-.25cm}

\begin{subfigure}[t]{0.24\textwidth} 
\centering 
\includegraphics[width=\textwidth]{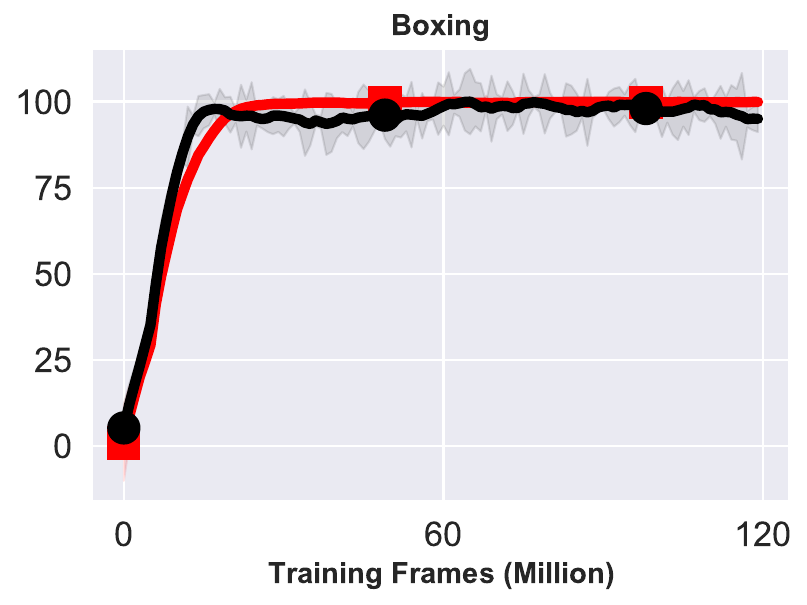}  
\end{subfigure}%
~ 
\begin{subfigure}[t]{ 0.24\textwidth} 
\centering 
\includegraphics[width=\textwidth]{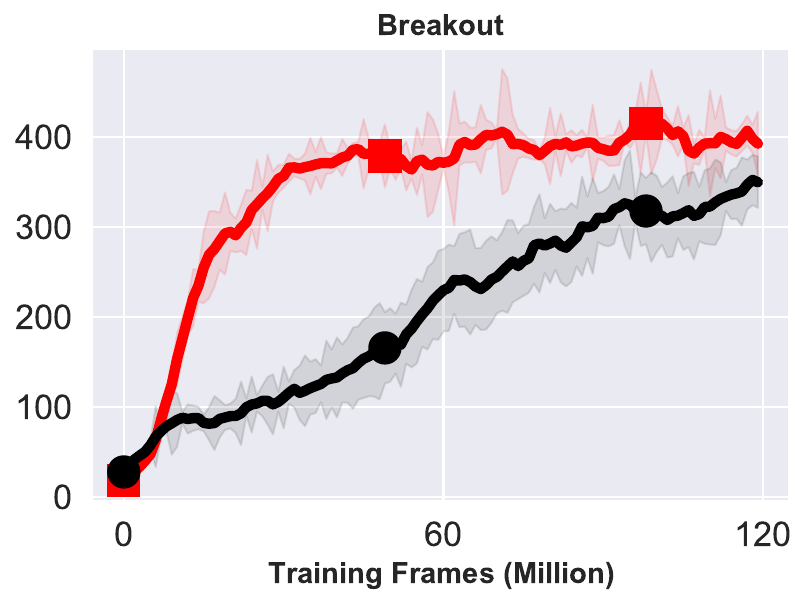}  
\end{subfigure}%
~ 
\begin{subfigure}[t]{ 0.24\textwidth} 
\centering 
\includegraphics[width=\textwidth]{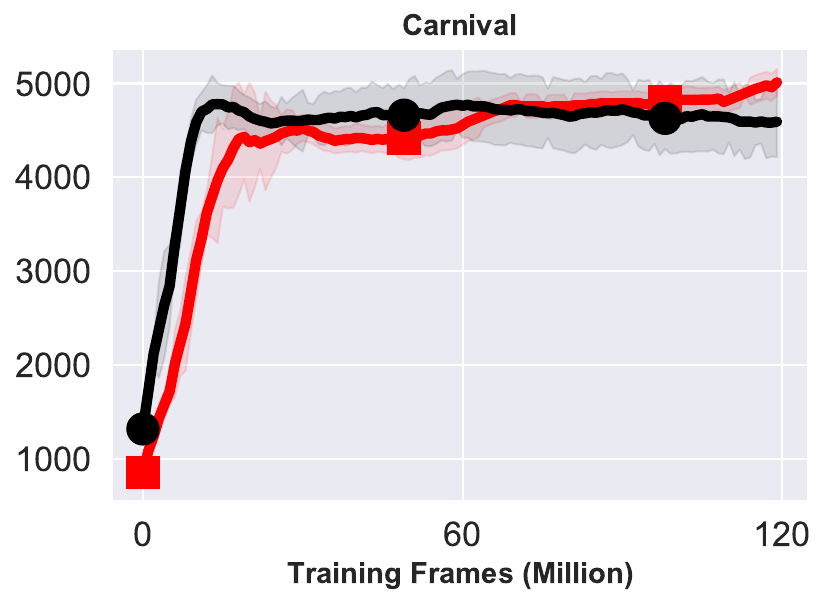} 
\end{subfigure}%
~ 
\begin{subfigure}[t]{ 0.24\textwidth} 
\centering 
\includegraphics[width=\textwidth]{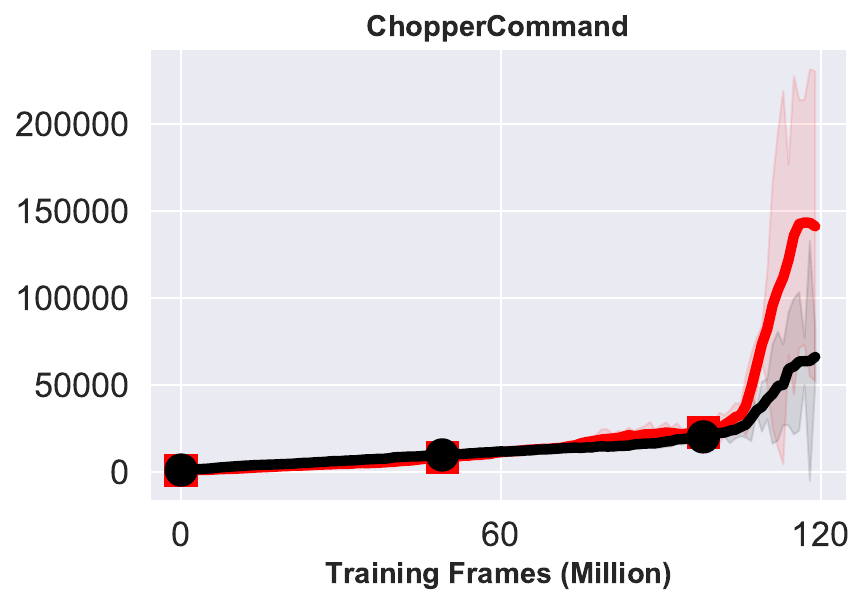} 
\end{subfigure}
\vspace{-.25cm}

\begin{subfigure}[t]{0.24\textwidth} 
\centering 
\includegraphics[width=\textwidth]{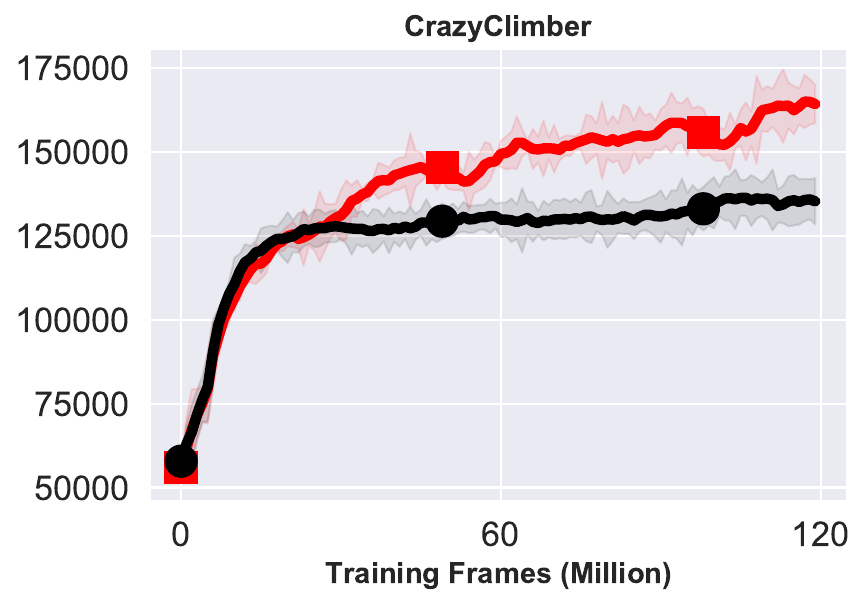}  
\end{subfigure}%
~ 
\begin{subfigure}[t]{ 0.24\textwidth} 
\centering 
\includegraphics[width=\textwidth]{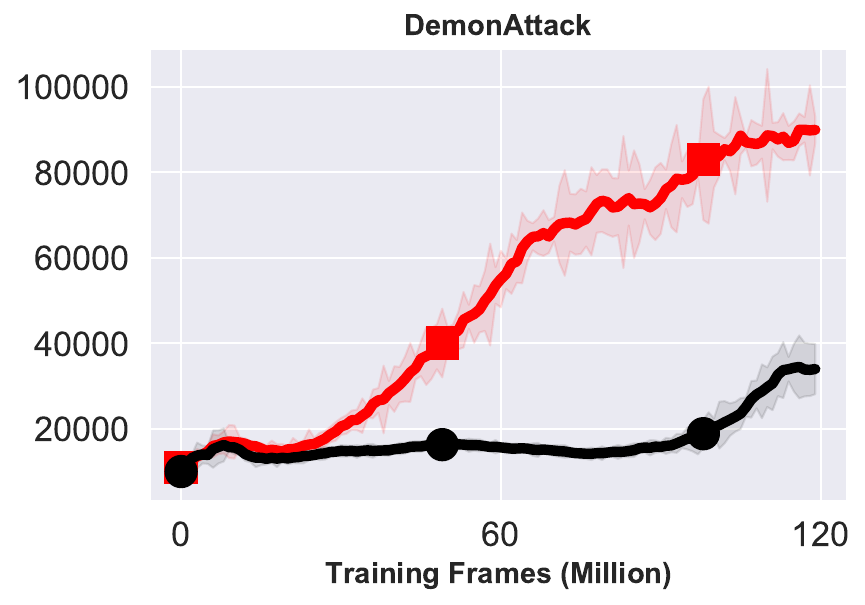}  
\end{subfigure}%
~ 
\begin{subfigure}[t]{ 0.24\textwidth} 
\centering 
\includegraphics[width=\textwidth]{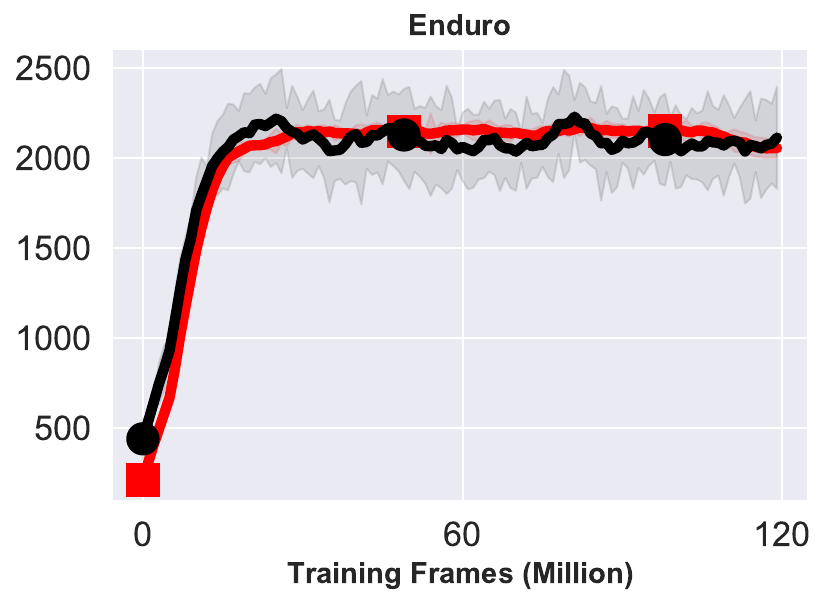} 
\end{subfigure}%
~ 
\begin{subfigure}[t]{ 0.24\textwidth} 
\centering 
\includegraphics[width=\textwidth]{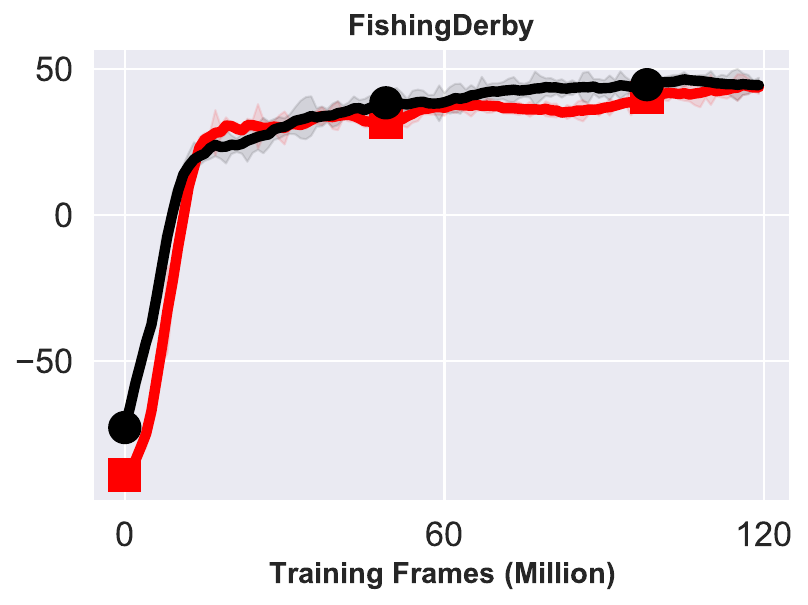} 
\end{subfigure}
\vspace{-.25cm}

\begin{subfigure}[t]{0.24\textwidth} 
\centering 
\includegraphics[width=\textwidth]{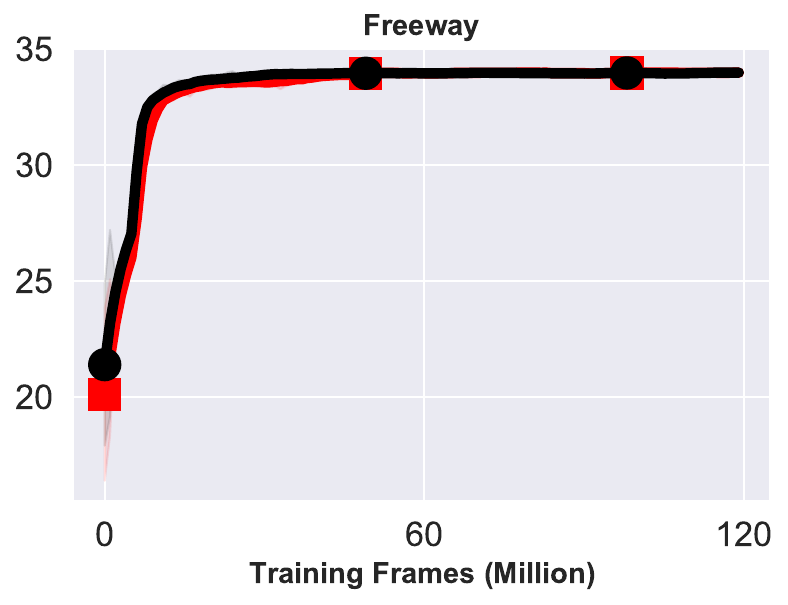}  
\end{subfigure}%
~ 
\begin{subfigure}[t]{ 0.24\textwidth} 
\centering 
\includegraphics[width=\textwidth]{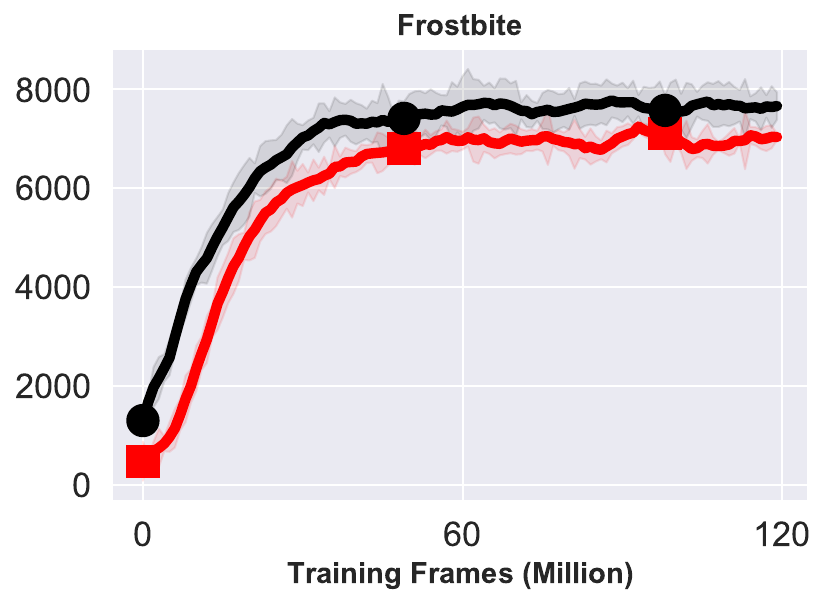}  
\end{subfigure}%
~ 
\begin{subfigure}[t]{ 0.24\textwidth} 
\centering 
\includegraphics[width=\textwidth]{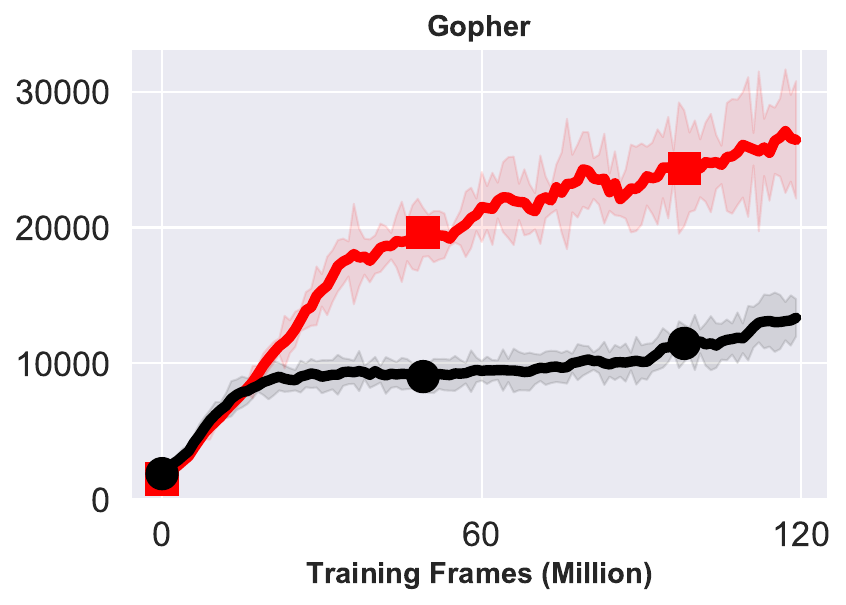} 
\end{subfigure}%
~ 
\begin{subfigure}[t]{ 0.24\textwidth} 
\centering 
\includegraphics[width=\textwidth]{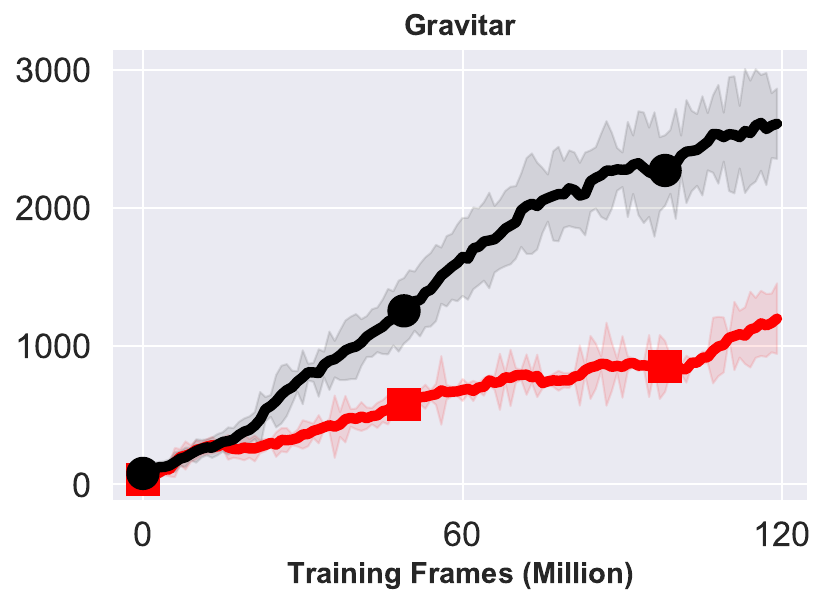} 
\end{subfigure}
\vspace{-.25cm}

\begin{subfigure}[t]{0.24\textwidth} 
\centering 
\includegraphics[width=\textwidth]{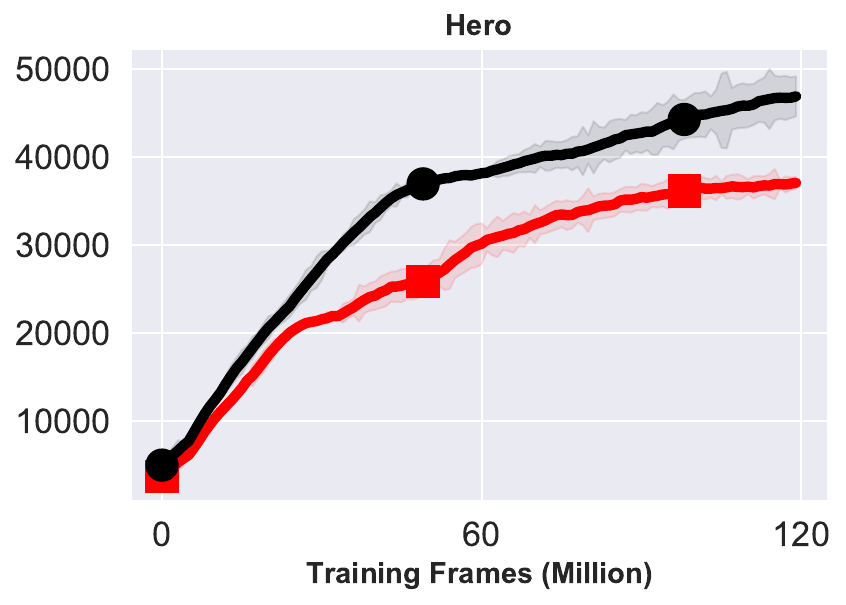}  
\end{subfigure}%
~ 
\begin{subfigure}[t]{ 0.24\textwidth} 
\centering 
\includegraphics[width=\textwidth]{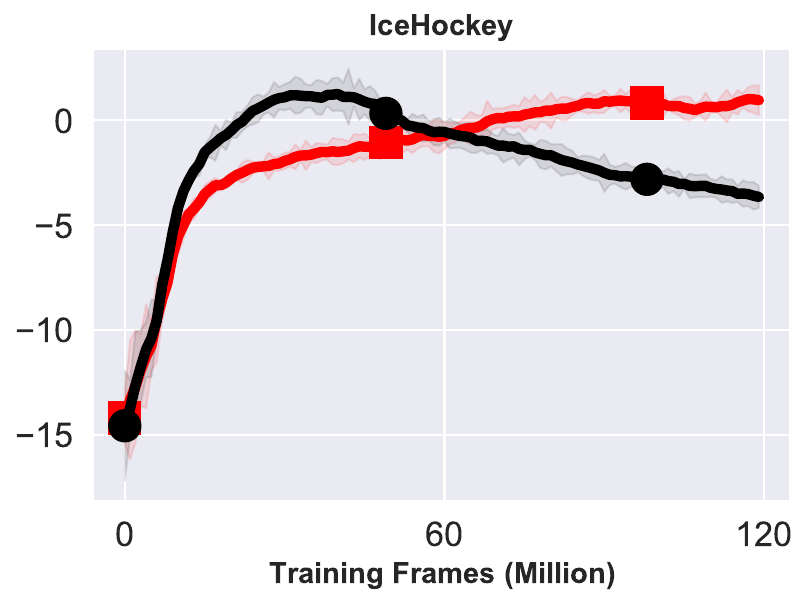}  
\end{subfigure}%
~ 
\begin{subfigure}[t]{ 0.24\textwidth} 
\centering 
\includegraphics[width=\textwidth]{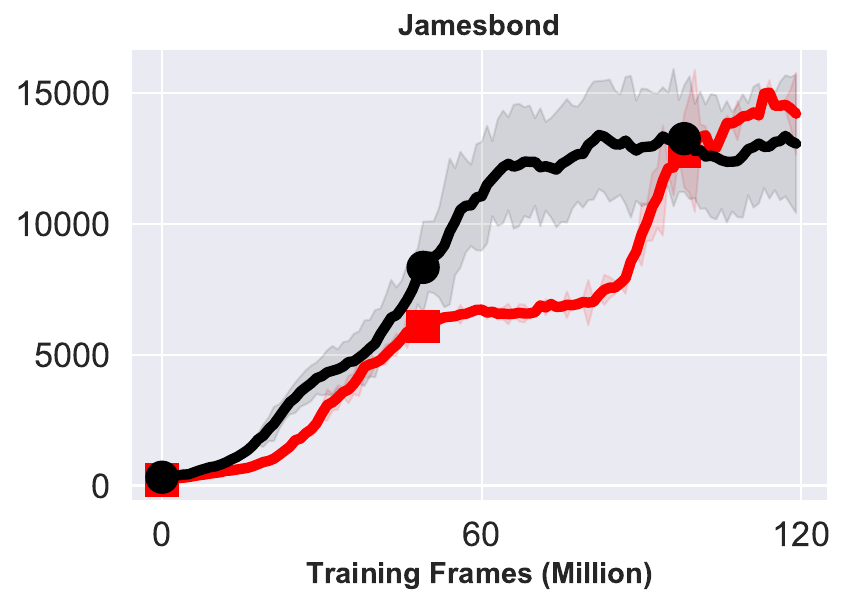} 
\end{subfigure}%
~ 
\begin{subfigure}[t]{ 0.24\textwidth} 
\centering 
\includegraphics[width=\textwidth]{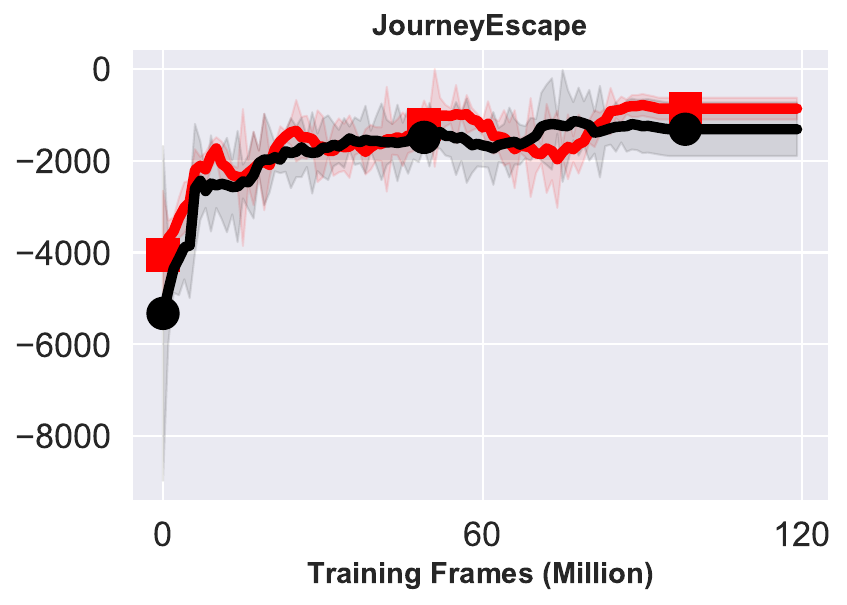} 
\end{subfigure}
\vspace{-.25cm}

\caption{ Learning curves for 55 games (Part I).}
\end{figure}

\begin{figure}[H]
\centering\captionsetup[subfigure]{justification=centering,skip=0pt}
\begin{subfigure}[t]{0.24\textwidth} 
\centering 
\includegraphics[width=\textwidth]{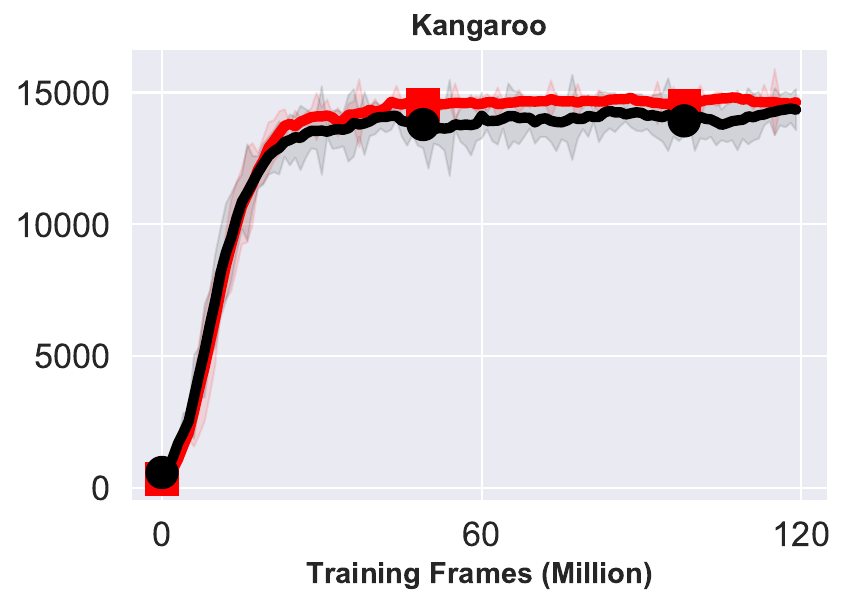} 
\end{subfigure}%
~ 
\begin{subfigure}[t]{ 0.24\textwidth} 
\centering 
\includegraphics[width=\textwidth]{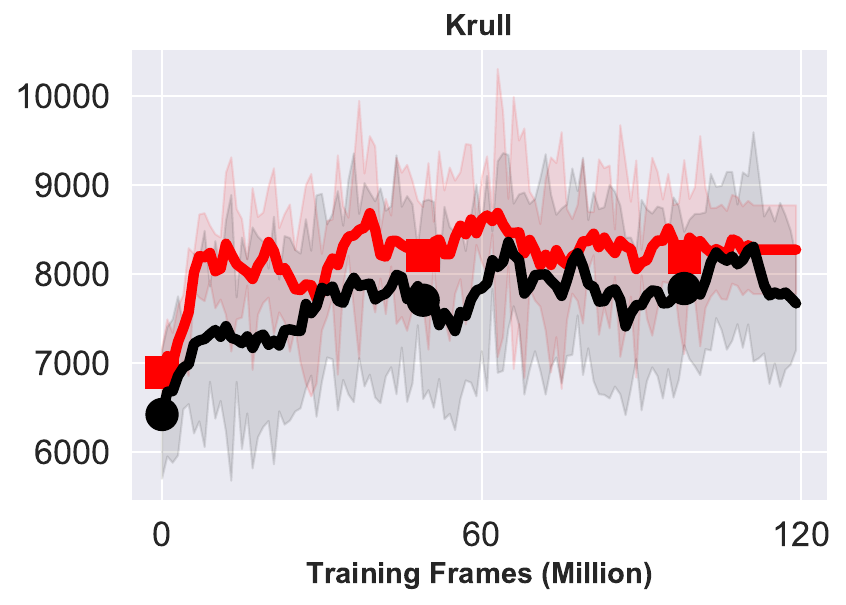} 
\end{subfigure}%
~ 
\begin{subfigure}[t]{ 0.24\textwidth} 
\centering 
\includegraphics[width=\textwidth]{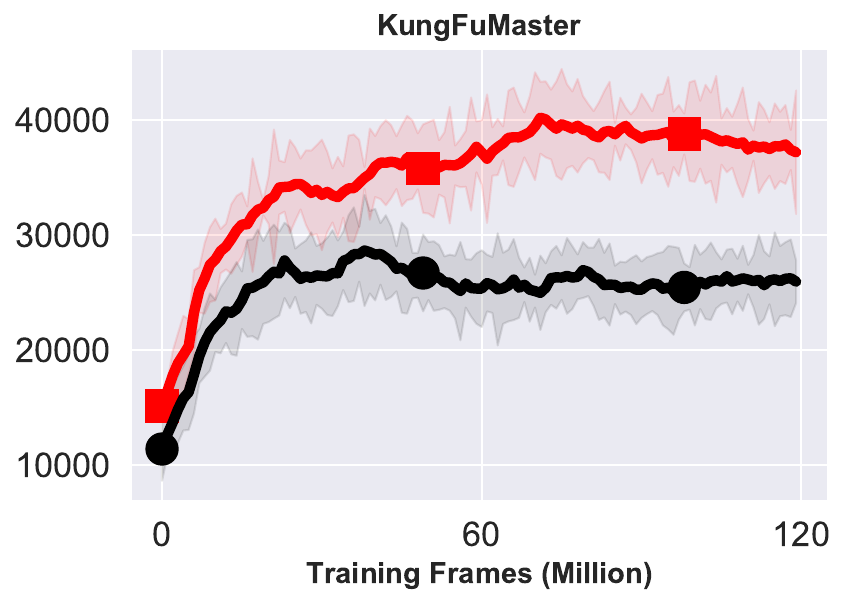}  
\end{subfigure}%
~ 
\begin{subfigure}[t]{ 0.24\textwidth} 
\centering 
\includegraphics[width=\textwidth]{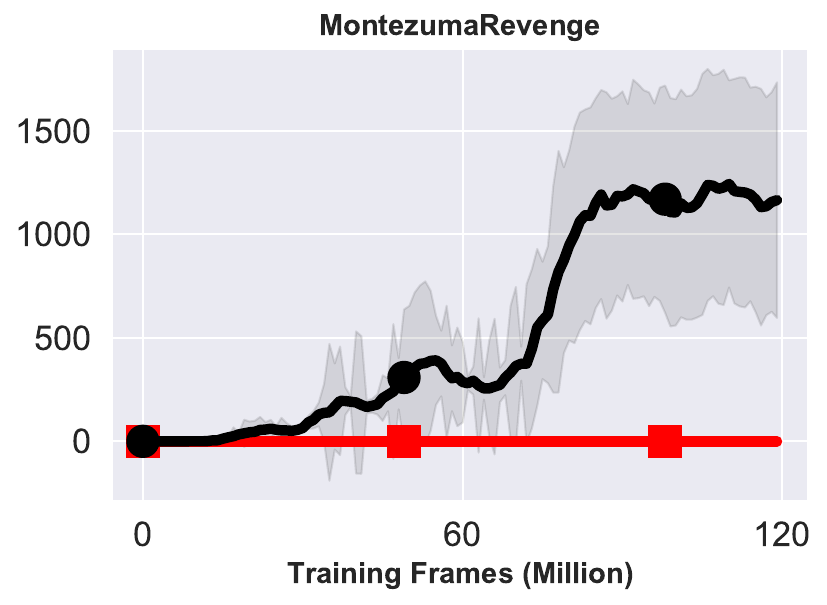} 
\end{subfigure}
\vspace{-.25cm}

\begin{subfigure}[t]{0.24\textwidth} 
\centering 
\includegraphics[width=\textwidth]{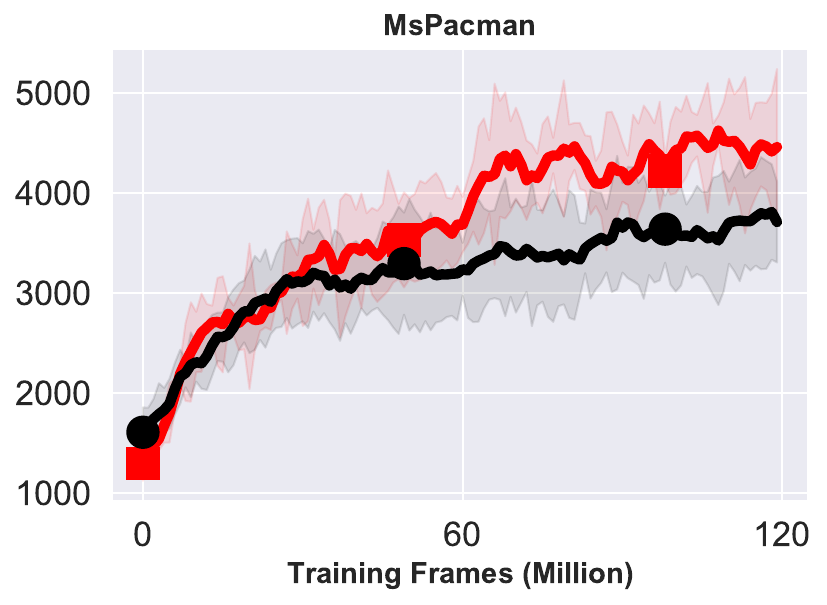} 
\end{subfigure}%
~ 
\begin{subfigure}[t]{ 0.24\textwidth} 
\centering 
\includegraphics[width=\textwidth]{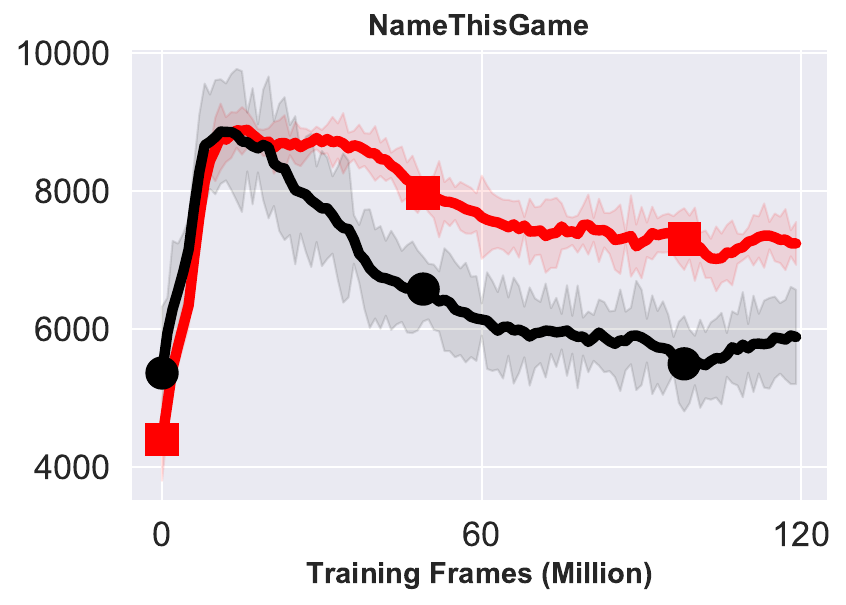} 
\end{subfigure}%
~ 
\begin{subfigure}[t]{ 0.24\textwidth} 
\centering 
\includegraphics[width=\textwidth]{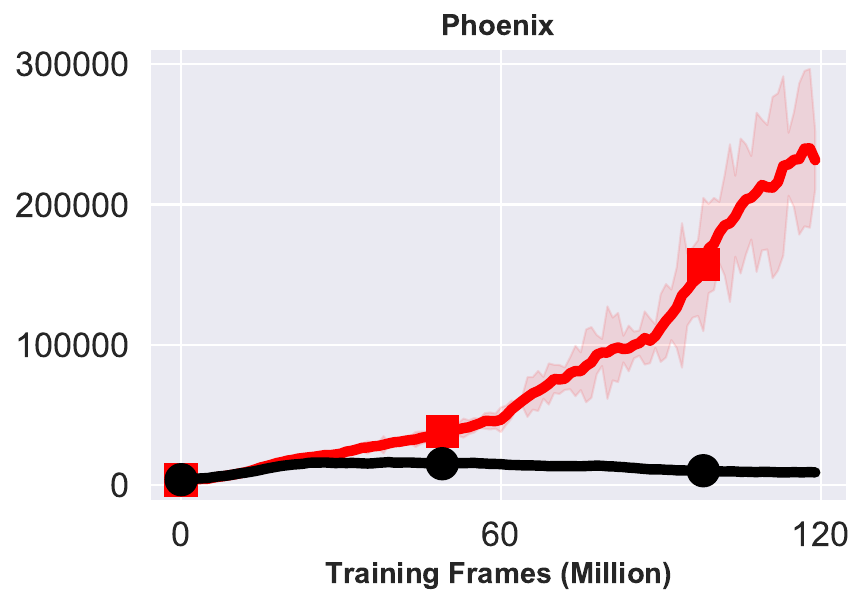}  
\end{subfigure}%
~ 
\begin{subfigure}[t]{ 0.24\textwidth} 
\centering 
\includegraphics[width=\textwidth]{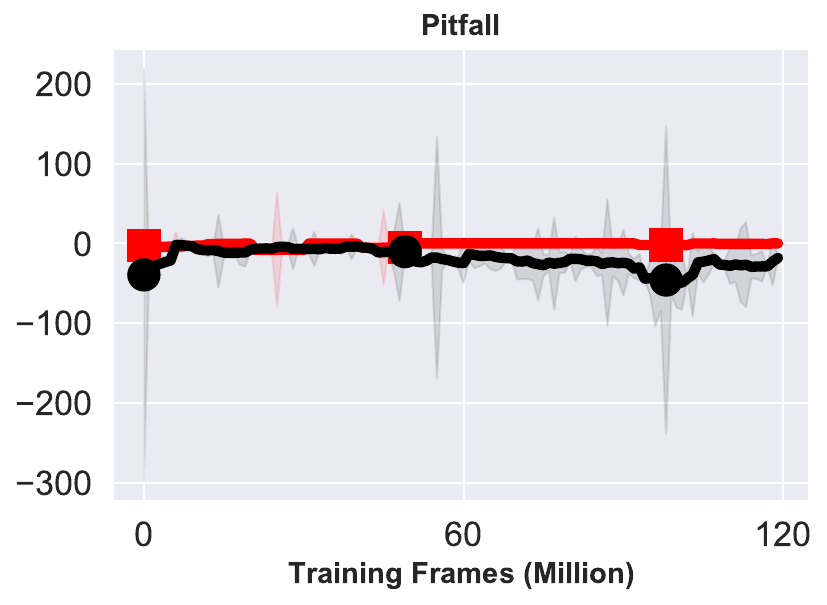} 
\end{subfigure}
\vspace{-.25cm}

\begin{subfigure}[t]{0.24\textwidth} 
\centering 
\includegraphics[width=\textwidth]{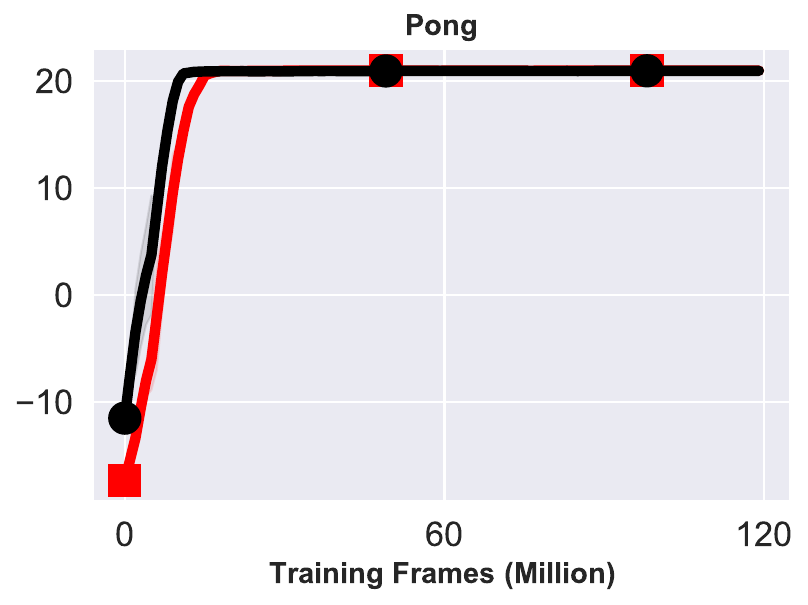} 
\end{subfigure}%
~ 
\begin{subfigure}[t]{ 0.24\textwidth} 
\centering 
\includegraphics[width=\textwidth]{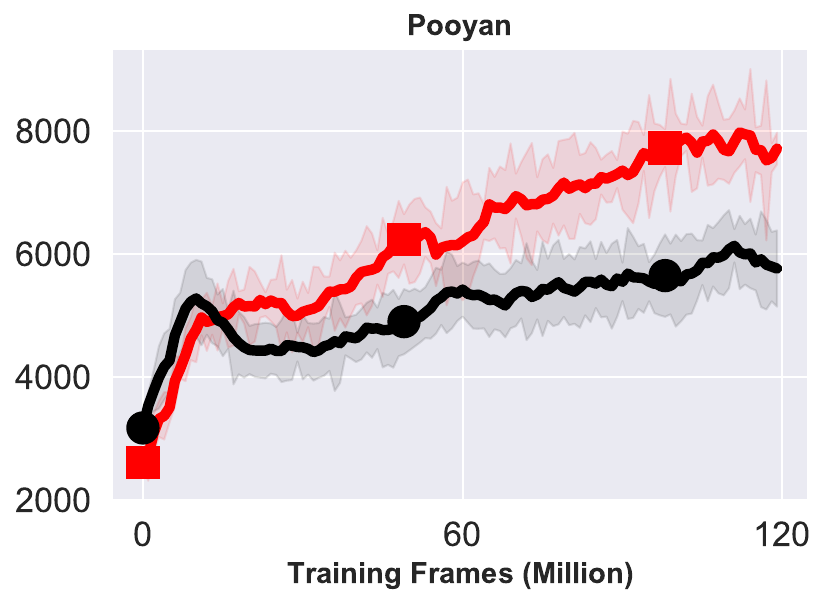} 
\end{subfigure}%
~ 
\begin{subfigure}[t]{ 0.24\textwidth} 
\centering 
\includegraphics[width=\textwidth]{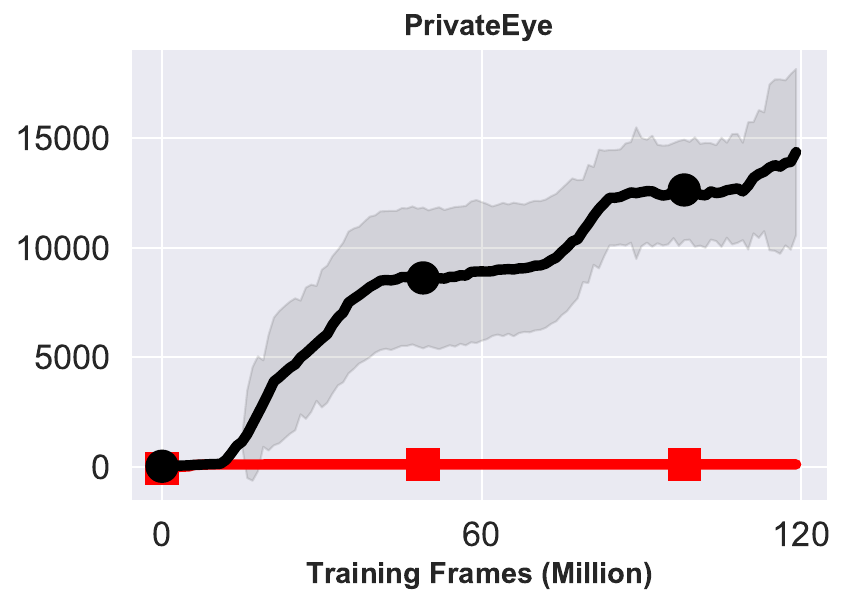}  
\end{subfigure}%
~ 
\begin{subfigure}[t]{ 0.24\textwidth} 
\centering 
\includegraphics[width=\textwidth]{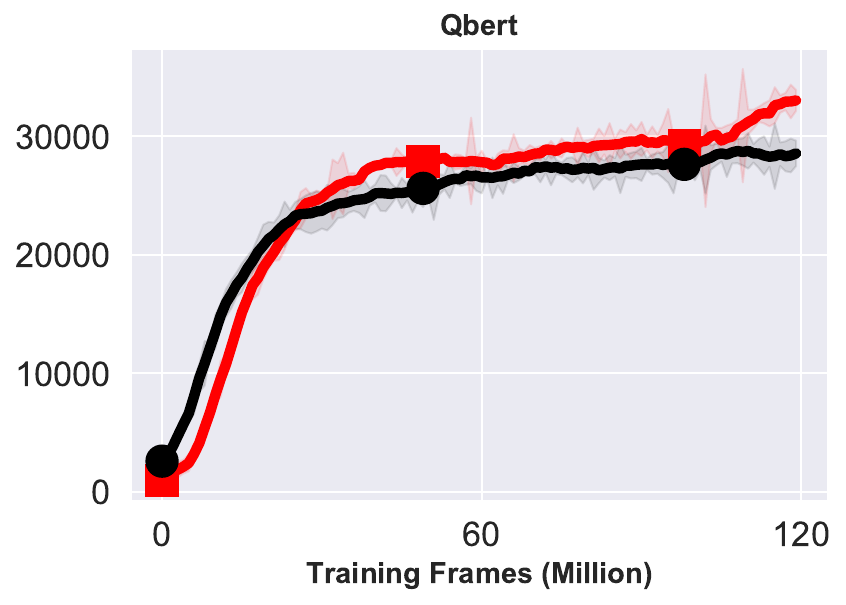} 
\end{subfigure}
\vspace{-.25cm}

\begin{subfigure}[t]{0.24\textwidth} 
\centering 
\includegraphics[width=\textwidth]{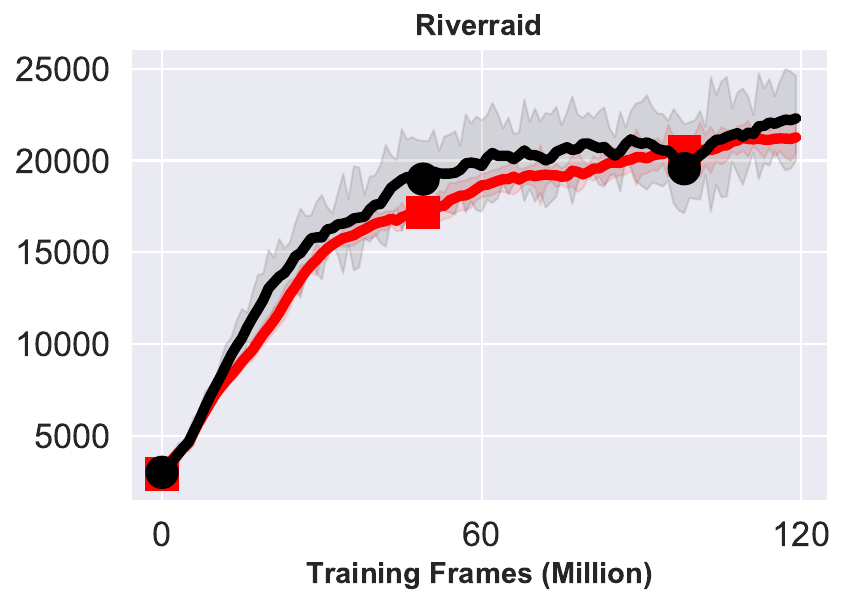} 
\end{subfigure}%
~ 
\begin{subfigure}[t]{ 0.24\textwidth} 
\centering 
\includegraphics[width=\textwidth]{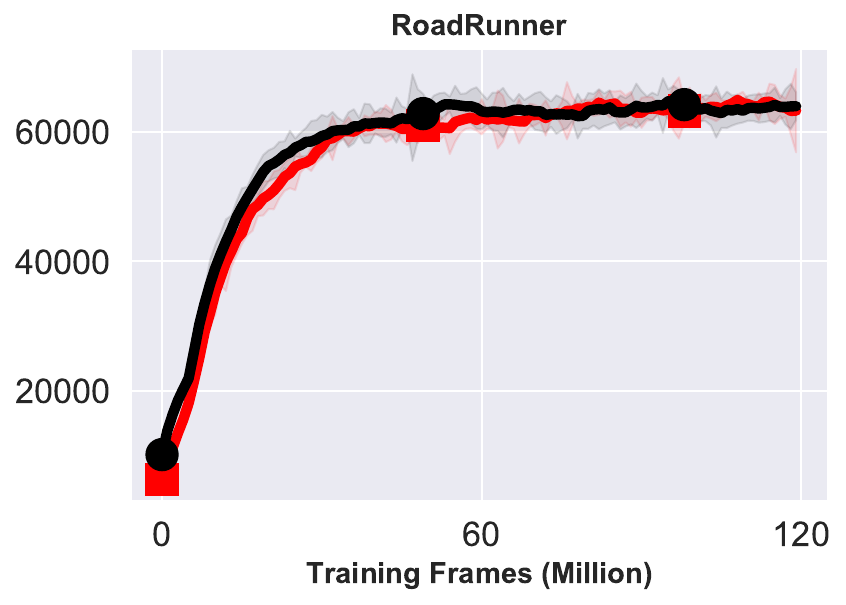} 
\end{subfigure}%
~ 
\begin{subfigure}[t]{ 0.24\textwidth} 
\centering 
\includegraphics[width=\textwidth]{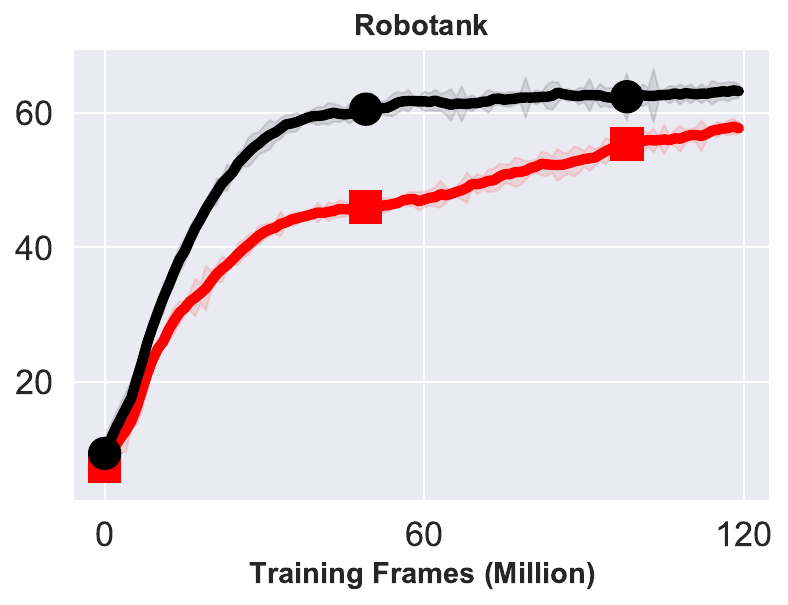}  
\end{subfigure}%
~ 
\begin{subfigure}[t]{ 0.24\textwidth} 
\centering 
\includegraphics[width=\textwidth]{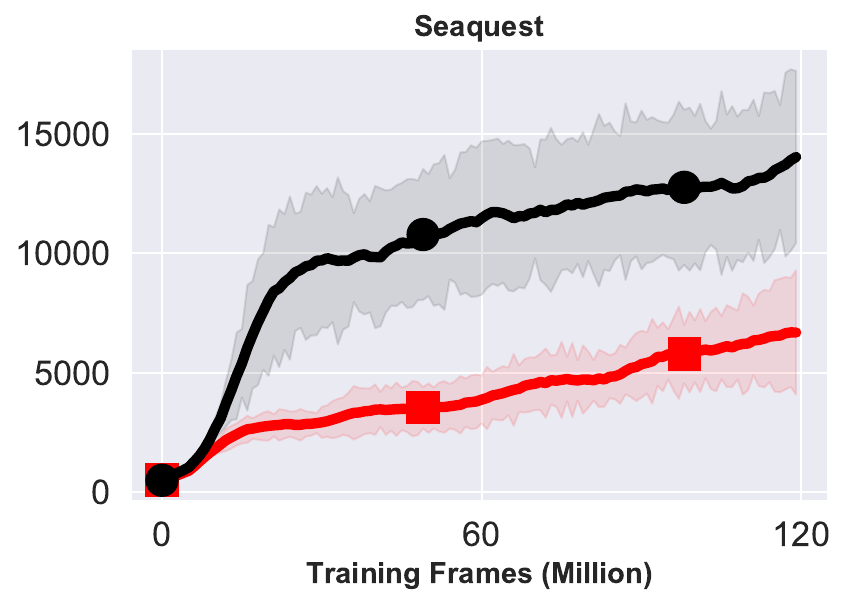} 
\end{subfigure}
\vspace{-.25cm}

\begin{subfigure}[t]{0.24\textwidth} 
\centering 
\includegraphics[width=\textwidth]{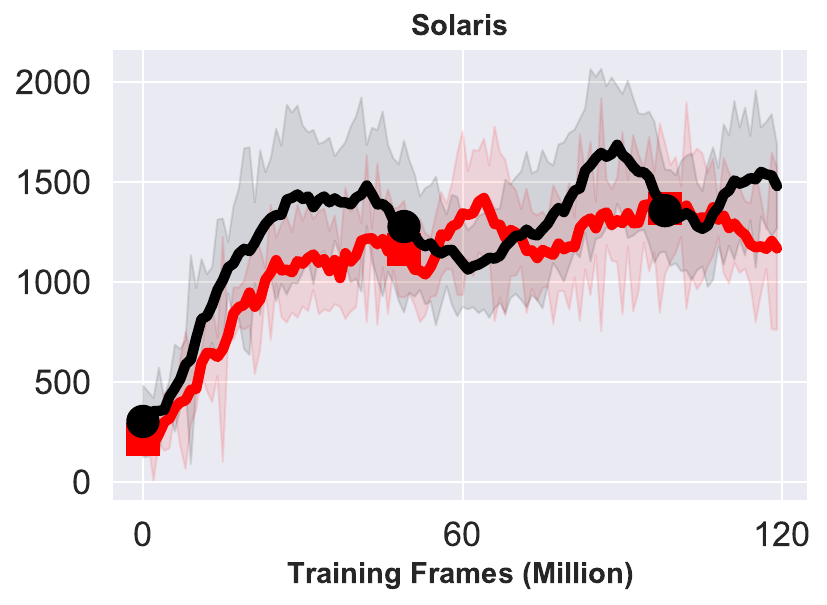} 
\end{subfigure}%
~ 
\begin{subfigure}[t]{ 0.24\textwidth} 
\centering 
\includegraphics[width=\textwidth]{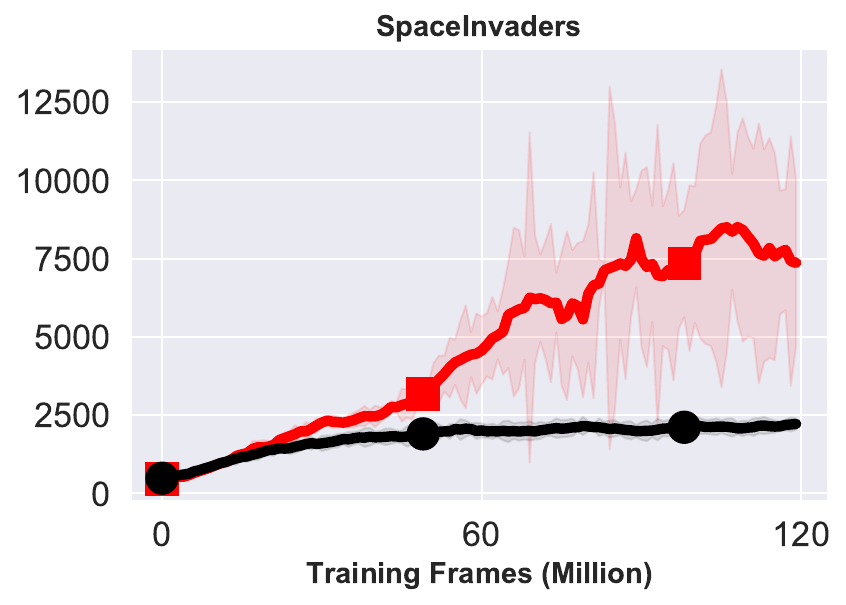} 
\end{subfigure}%
~ 
\begin{subfigure}[t]{ 0.24\textwidth} 
\centering 
\includegraphics[width=\textwidth]{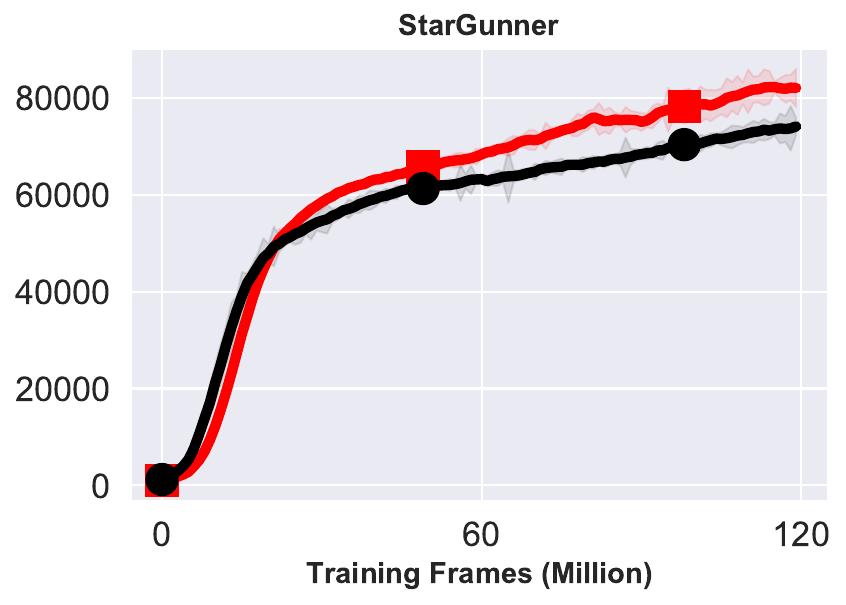}  
\end{subfigure}%
~ 
\begin{subfigure}[t]{ 0.24\textwidth} 
\centering 
\includegraphics[width=\textwidth]{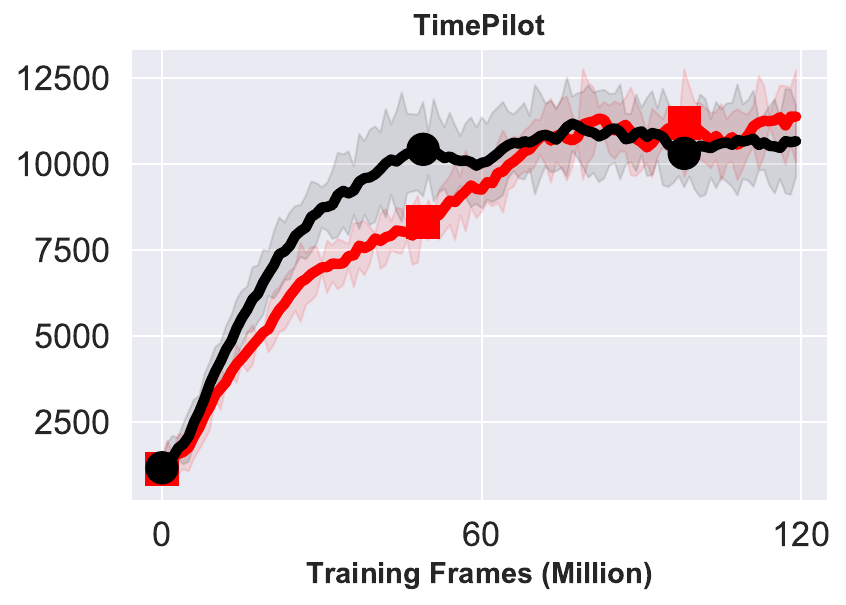} 
\end{subfigure}
\vspace{-.25cm}

\begin{subfigure}[t]{0.24\textwidth} 
\centering 
\includegraphics[width=\textwidth]{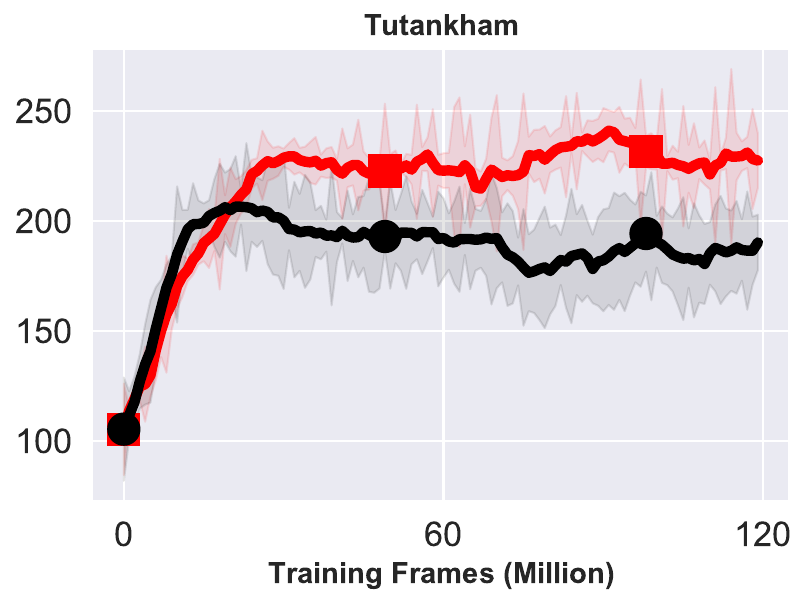} 
\end{subfigure}%
~ 
\begin{subfigure}[t]{ 0.24\textwidth} 
\centering 
\includegraphics[width=\textwidth]{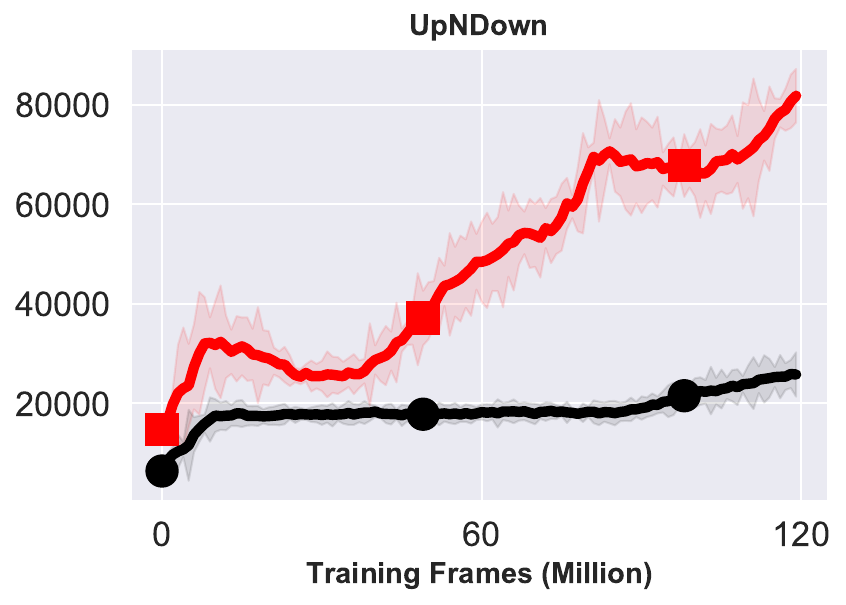} 
\end{subfigure}%
~ 
\begin{subfigure}[t]{ 0.24\textwidth} 
\centering 
\includegraphics[width=\textwidth]{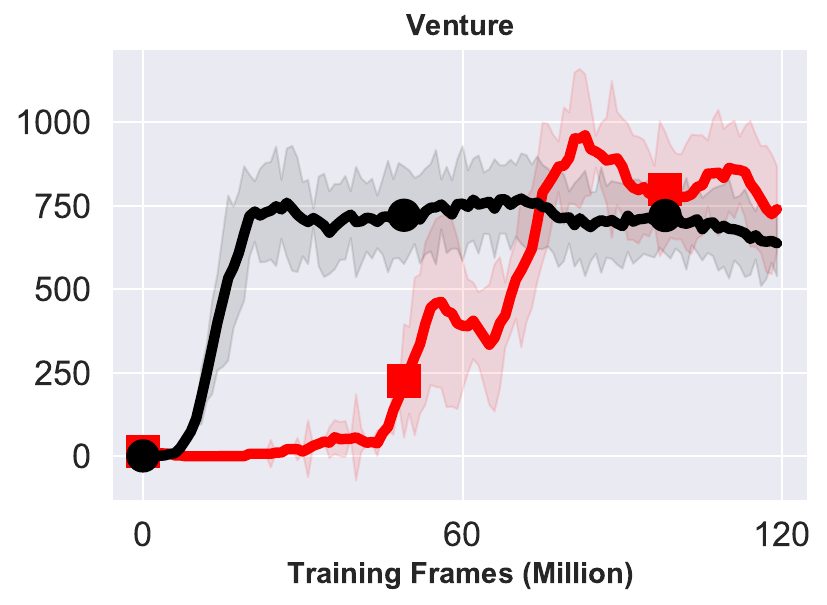}  
\end{subfigure}%
~ 
\begin{subfigure}[t]{ 0.24\textwidth} 
\centering 
\includegraphics[width=\textwidth]{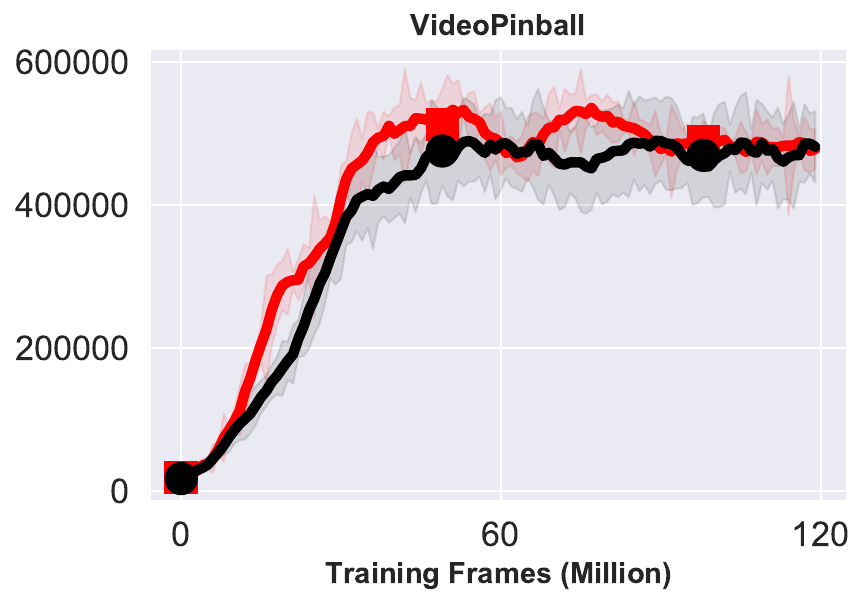} 
\end{subfigure}
\vspace{-.25cm}

\begin{subfigure}[t]{0.24\textwidth} 
\centering 
\includegraphics[width=\textwidth]{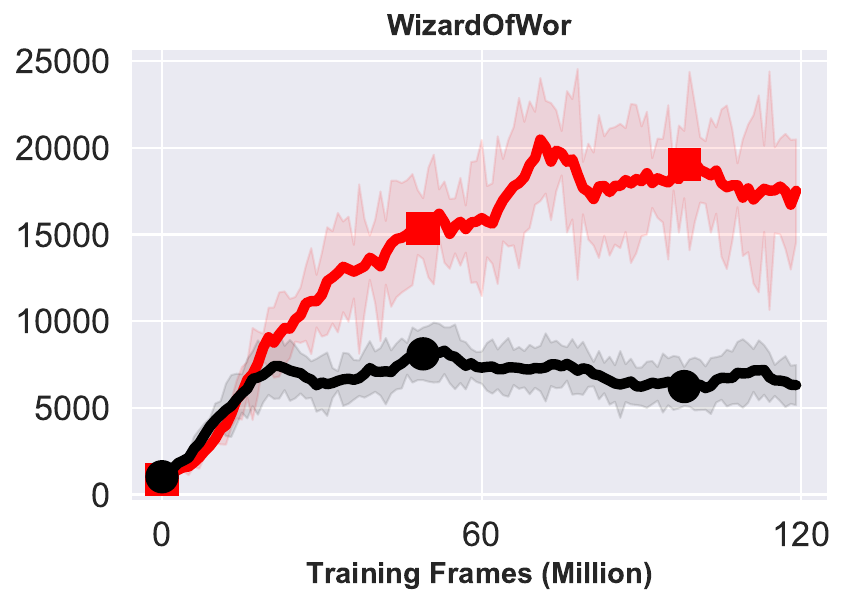} 
\end{subfigure}%
~ 
\begin{subfigure}[t]{ 0.24\textwidth} 
\centering 
\includegraphics[width=\textwidth]{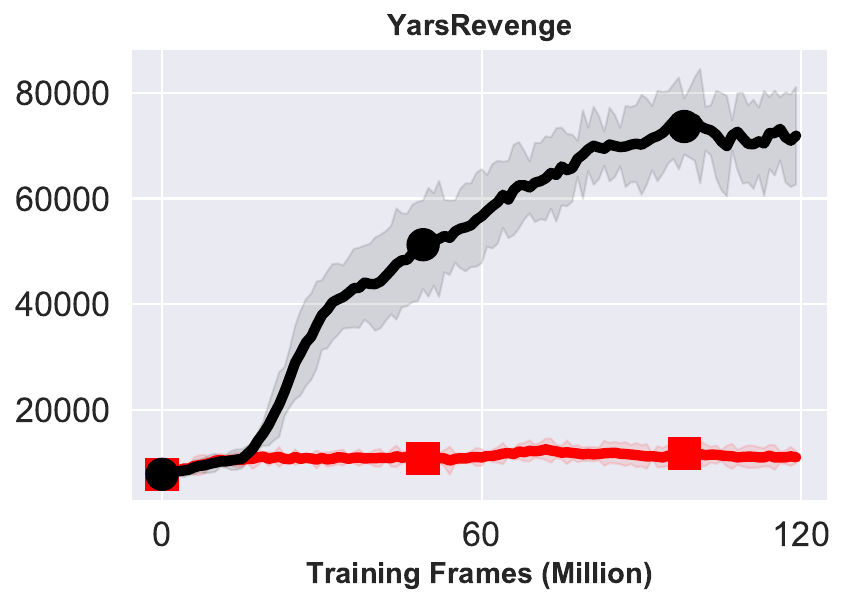} 
\end{subfigure}%
~ 
\begin{subfigure}[t]{ 0.24\textwidth} 
\centering 
\includegraphics[width=\textwidth]{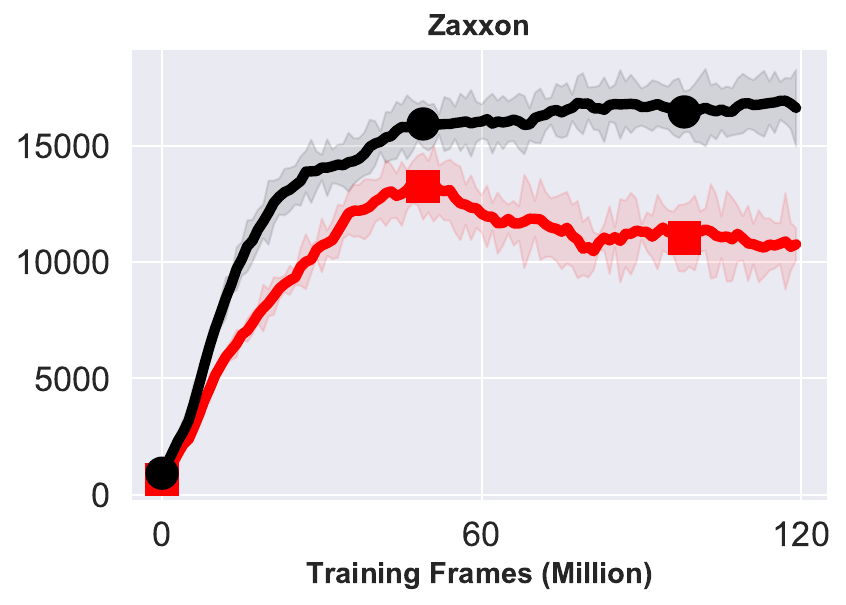}  
\end{subfigure}%
\caption{ Learning curves for 55 games (Part II).}
\end{figure}


\end{document}